%% file: main.tex
\documentclass{article}

\usepackage{microtype}
\usepackage{graphicx}
\usepackage{subcaption}
\usepackage{booktabs} 

\usepackage{hyperref}
\hypersetup{
    colorlinks,
    linkcolor={red!50!black},
    citecolor={blue!50!black},
    urlcolor={blue!80!black}
}

\usepackage{url}
\usepackage{outlines}
\usepackage{amssymb}
\usepackage{pifont}
\usepackage{soul}
\usepackage{color}
\usepackage{wrapfig}
\usepackage{subcaption}
\usepackage{amsfonts}       
\usepackage[nice]{nicefrac}       
\usepackage{xcolor}         
\usepackage{xspace}         
\usepackage{csquotes}
\usepackage{enumitem}
\usepackage{amsmath}
\usepackage{multirow}
\usepackage[normalem]{ulem}
\usepackage{float}
\usepackage[makeroom]{cancel}

\DeclareMathOperator*{\argmin}{argmin}

\usepackage{tablefootnote}

\usepackage[accepted]{icml2023}

\usepackage{amsthm,verbatim,amscd, mathtools}
\theoremstyle{plain}
\newtheorem{theorem}{Theorem}
\newtheorem{corollary}{Corollary}
\newtheorem{lemma}{Lemma}

\newtheorem{remark}{Remark}
\newtheorem{proposition}{Proposition}

\newtheorem{conjecture}{Conjecture}
 
\newtheorem{definition}{Definition}

\newcommand{\norm}[1]{\left\lVert#1\right\rVert}

\definecolor{myred}{RGB}{215,48,39}
\definecolor{mygreen}{RGB}{26,152,80}
\definecolor{myblack}{RGB}{0,0,0}
\definecolor{myyellow}{RGB}{226,192,100}
\definecolor{mygray}{gray}{0.5}
\newcommand{\mycmark}{\textcolor{mygreen}{\ding{51}}}
\newcommand{\myxmark}{\textcolor{myred}{\ding{55}}}

\usepackage[symbol]{footmisc}

\icmltitlerunning{Mechanistic Mode Connectivity}

\begin{document}

\twocolumn[
\icmltitle{Mechanistic Mode Connectivity}

\icmlsetsymbol{equal}{*}

\begin{icmlauthorlist}
\icmlauthor{Ekdeep Singh Lubana}{1,2,4}
\icmlauthor{Eric J.\ Bigelow}{2}
\icmlauthor{Robert P.\ Dick}{1}
\icmlauthor{David Krueger}{equal,3}
\icmlauthor{Hidenori Tanaka}{equal,2,4}
\end{icmlauthorlist}

\icmlaffiliation{1}{EECS Department, University of Michigan, Ann Arbor, MI, USA}
\icmlaffiliation{2}{Center for Brain Science, Harvard University, Cambridge, MA, USA}
\icmlaffiliation{3}{University of Cambridge, UK}
\icmlaffiliation{4}{Physics \& Informatics Laboratories, NTT Research, Inc., Sunnyvale, CA, USA}

\icmlcorrespondingauthor{Ekdeep Singh Lubana}{eslubana@umich.edu}

\icmlkeywords{Machine Learning, ICML}

\vskip 0.3in
]



\printAffiliationsAndNotice{\icmlEqualAdvising} 

\begin{abstract}
We study neural network loss landscapes through the lens of mode connectivity, the observation that minimizers of neural networks retrieved via training on a dataset are connected via simple paths of low loss. Specifically, we ask the following question: \textit{are minimizers that rely on different mechanisms for making their predictions connected via simple paths of low loss?} We provide a definition of \textit{mechanistic similarity} as shared invariances to input transformations and demonstrate that lack of linear connectivity between two models implies they use dissimilar mechanisms for making their predictions. Relevant to practice, this result helps us demonstrate that na\"{i}ve fine-tuning on a downstream dataset can fail to alter a model's mechanisms, e.g., fine-tuning can fail to eliminate a model's reliance on spurious attributes. Our analysis also motivates a method for targeted alteration of a model's mechanisms, named \emph{connectivity-based fine-tuning} (CBFT), which we analyze using several synthetic datasets for the task of reducing a model's reliance on spurious attributes. Code is available at: \url{https://github.com/EkdeepSLubana/MMC}.
\end{abstract}


\vspace{-20pt}
\section{Introduction}
\label{sec:intro}

\begin{figure}
  \vspace{-2pt}
  \begin{center}
    \includegraphics[width=0.9\columnwidth]{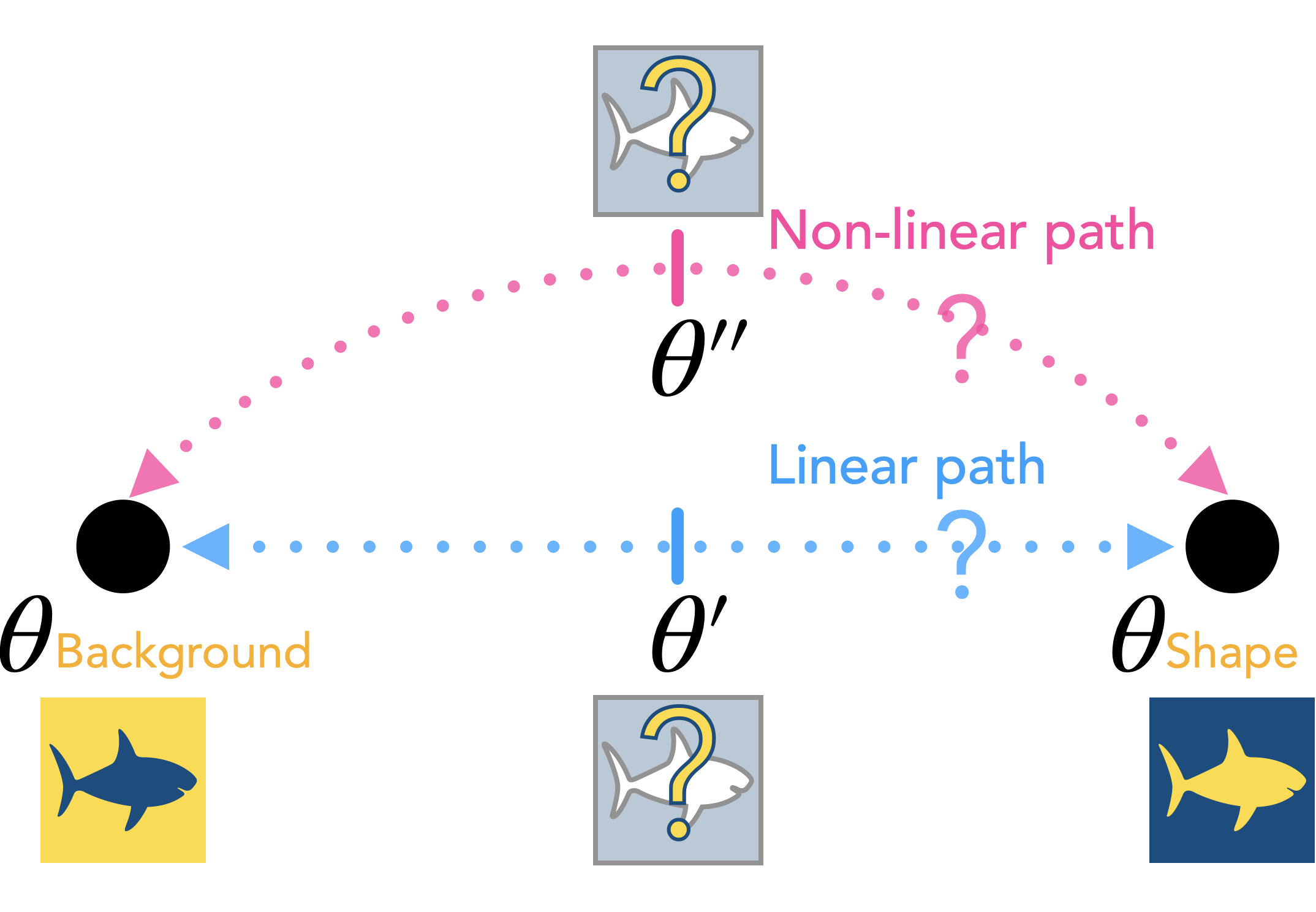}
  \end{center}
  \vspace{-1em}
  \caption{\label{fig:intro} \textbf{Mechanistic Lens on Mode connectivity. }
  Consider two sets of parameters that minimize loss using background $\theta_\mathrm{Background}$ and object shape $\theta_\mathrm{Shape}$ as the input attributes for prediction, respectively.
  Are such \textit{mechanistically dissimilar} minimizers connected via paths of low loss in the landscape?
  Does the dissimilarity of these mechanisms affect the simplicity of their connectivity paths? 
  Can we exploit this connectivity to switch between minimizers that use our desired mechanisms?}
  \vspace{-10pt}
\end{figure}

Loss landscapes of modern deep neural networks (DNNs) have been shown to contain infinitely many global minimizers that are equally reachable via standard gradient-based optimization techniques~\citep{kawaguchi2016deep, du2018gradient, du2019gradient, arora2018optimization, nguyen2017loss, nguyen2018optimization}. Recent work finds intriguing geometrical constraints relating these minimizers~\citep{simsek2021geometry, freeman2016topology, nguyen2018loss, nguyen2019connected, kuditipudi2019explaining, nguyen2021solutions}, showing them to be \textbf{connected} via a single, continuous manifold that emerges as a result of overparameterization. The existence of such connected sets of solutions has been heavily corroborated in literature on \textbf{mode connectivity}~\citep{garipov2018loss, draxler2018essentially, frankle2020linear, entezari2021role, ainsworth2022}, which, quite surprisingly, shows that the paths connecting global minimizers obtained via standard training pipelines are \textit{relatively simple} (e.g., linear or quadratic). In parallel, several papers recently demonstrated that different models trained on a task can perform radically differently at test time~\citep{d2020underspecification, hermann2020shapes, hendrycks2021many}. This behavior can be partially ascribed to models learning to utilize rather dissimilar attributes of an input for making their predictions~\citep{hermann2020origins, islam2021shape, scimeca2021shortcut, taori2020measuring}. For example, in most vision datasets, backgrounds are correlated with object categories---a sampling bias~\citep{beery2018recognition, xiao2020noise}. Consequently, a model can infer the correct label of an object by learning \textbf{mechanisms} to identify either its background or its shape; however, only models that rely on shape are likely to generalize robustly~\citep{geirhos2018imagenet, geirhos2020shortcut, ritter2017cognitive}. Thus, despite models of both types being equally performant on a given dataset, the exact mechanisms they use for making their predictions disallows for us to consider them equally useful.

\textbf{This work:} We argue prior literature analyzing connectivity properties in DNN loss landscapes has ignored the influence of the exact mechanisms a model implements for performing a task (see Fig.~\ref{fig:intro}). In fact, due to inherent tendencies in the training pipelines of modern DNNs towards learning simple functions~\citep{kalimeris2019sgd, valle2018deep, rahaman2019spectral, shah2020pitfalls, mangalam2019deep}, minimizers identified via training on the same dataset often exhibit similar biases~\citep{shah2020pitfalls, nanda2022measuring}. Such \textbf{similarity} in the models' prediction mechanisms may influence the identifiability of simple connectivity patterns in the loss landscape, such as the ones observed in prior work. Importantly, it has remained unclear if \textit{\textbf{mechanistically dissimilar} models, e.g., ones that rely on background and ones that rely on shape, exhibit connectivity at all.} Beyond a better scientific understanding of DNN loss landscapes, knowledge of such geometric properties relating mechanistically dissimilar minimizers can possibly lead to practical insights for designing post-hoc, sample-efficient \textit{fine-tuning} strategies that allow switching to minimizers that follow our desired predictions mechanisms. Motivated by questions above, we make the following contributions.
\begin{itemize}[itemsep=2pt,topsep=-4pt,leftmargin=9pt, parsep=0pt,partopsep=0pt]
    \item \textbf{Defining a notion of \emph{mechanistic similarity} (\S\ref{sec:funcsim}).} We characterize mechanistic similarity of two models via systematic interventions on the data-generating process, claiming similarity if the models are \textit{invariant} to the same set of interventions. Our definition is motivated to account for the specific attributes of an input (e.g., shape vs.\ background) a model relies on for making predictions. When analyzed in the context of \textit{spurious} attributes, our definition leads to a characterization of DNN loss landscapes that is relevant to challenges of robustness~\citep{d2020underspecification, teney2022predicting, jacobsen2018excessive}.

    \item \textbf{Characterizing connectivity properties of mechanistically (dis)similar models (\S\ref{sec: mech-connect}).} Our analysis shows that \textit{if two models lack linear connectivity in the landscape (up to architectural symmetries), they must be mechanistically dissimilar}; that is, existence of loss barriers on the linear path between two models implies they have learned different invariances (see Fig.~\ref{fig:smc},~\ref{fig:lmc}). Our results especially hold implications for na\"{i}ve fine-tuning of a pretrained network, which often yields models linearly connected with the original pretraining solution~\citep{neyshabur2020being}. Specifically, if a model has learned to rely on spurious attributes during pretraining, our results imply mere fine-tuning on some ``clean'' dataset may not improve its robustness. We augment these first steps towards a mechanistic characterization of loss landscapes with extensive empirical verification over a broad variety of settings, including different datasets, architectures, connectivity paths, and training strategies.

    \item \textbf{Exploiting lack of linear connectivity to efficiently alter a model's mechanisms (\S\ref{sec:mechanistic fine tuning}).} Based on our analysis, we propose a method, \textit{Connectivity-Based Fine-Tuning (CBFT)}, that exploits lack of linear connectivity between mechanistically dissimilar models to induce models that differ in specific prediction mechanisms (\S\ref{sec:mechanistic fine tuning}). Extensive experiments on synthetic datasets show CBFT is more effective than recent methods~\citep{flkirichenko2022, kirichenko2022last, kumar2022fine} at reducing a model's tendency to rely on spurious attributes for making its predictions.
\end{itemize} 
\vspace{-5pt}

\section{Preliminaries: Mode Connectivity}
Intuitively, mode connectivity along a path implies moving along that path does not witness \textit{barriers} in error or loss. We formalize this below, in line with prior work~\citep{frankle2020linear, garipov2018loss, draxler2018essentially, entezari2021role, benton2021loss, pittorino2022deep}. Consider a neural network $f: \mathbb{R}^{n} \times \mathbb{R}^{d} \to [K]$ that takes $n$-dimensional inputs $x \in \mathcal{X} \subset \mathbb{R}^{n}$, has parameters $\theta \in \mathbb{R}^{d}$, and produces an output $f(x; \theta) \in [K]$, where $[K]$ denotes the set $\{1, 2, \dots, K\}$. We say $\theta$ ``induces the model'' $f(.; \theta)$. A model's loss on a dataset $\mathcal{D} \in \mathcal{X} \times [K]$ for set of parameters $\theta$ is denoted using $\mathcal{L}(f(\mathcal{D}; \theta))$; $\theta$ is called a minimizer of the loss on that dataset if $\mathcal{L}(f(\mathcal{D}; \theta)) < \epsilon$, where $\epsilon$ is some small scalar. Note that we primarily focus on minimizers obtained using SGD. We denote a continuous path between two sets of parameters $\theta_1, \theta_2$ as $\gamma_{\theta_1 \to \theta_2}(t)$, where $\gamma_{\theta_1 \to \theta_2}(0) = \theta_1$ and $\gamma_{\theta_1 \to \theta_2}(1) = \theta_2$.

\begin{definition} 
\label{def:mode connectivity}
\textbf{(Mode Connectivity.)} Minimizers $\theta_1, \theta_2$ corresponding to a dataset $\mathcal{D}$ are called \textit{mode connected} along the path $\gamma_{\theta_1 \to \theta_2}(t)$ if moving along the path never yields barriers. Formally, $\forall\, t \in [0, 1]$, $\mathcal{L}(f(\mathcal{D}, \gamma_{\theta_1 \to \theta_2}(t))) \leq t \cdot \mathcal{L}(f(\mathcal{D}; \theta_0)) + (1-t) \cdot \mathcal{L}(f(\mathcal{D}; \theta_1))$.
\vspace{-6pt}
\end{definition}
As mentioned in \S\ref{sec:intro}, prior work shows mode connectivity is exhibited in modern DNNs' loss landscapes along rather simple paths. We focus on the following two families:
\vspace{-5pt}
\begin{equation*}
\label{eq:paths}
\begin{split}
&\text{(i) \textbf{Linear: }} \gamma_{\theta_1 \to \theta_2}(t) = (1-t) \theta_1 + t \theta_2 \quad\text{and} \\
&\text{(ii) \textbf{Quadratic: }} \gamma_{\theta_1 \to \theta_2}(t) = (1-t)^2 \theta_1 + 2 t (1-t) \theta_{12} + t^2 \theta_2.
\end{split}
\vspace{-5pt}
\end{equation*}
In the above, $\theta_{12}$ denotes a set of parameters that is explicitly optimized to identify a quadratic path connecting $\theta_1$ and $\theta_2$; notably, then, quadratic paths are a function of the data used for identifying them (see App.~\ref{app:quad_paths} for further discussion). 
 
\citet{entezari2021role} recently hypothesized that accounting for permutation symmetry\footnote{Note that DNNs exhibit several architectural symmetries and a more general statement would account for all such symmetries, as done by \citep{pittorino2022deep}. However, symmetries beyond permutations are unlikely to play a critical role in analysis of mode connectivity of SGD based minimizers (see App.~\ref{app:othersyms} for details).} of DNN architectures~\citep{hecht1990algebraic} in fact leads to observance of linear connectivity between any two linearly \textit{disconnected} minimizers obtained using SGD; \citet{entezari2021role, singh2020model, ainsworth2022} extensively probe and corroborate this claim empirically. To demonstrate the robustness of our results, we also assess whether accounting for permutation symmetry leads to linear connectivity between mechanistically dissimilar models. Specifically, we follow the ``activation matching'' algorithm used by \citet{ainsworth2022} and call these paths \textit{Linear (permuted)}. \vspace{-3pt}

\section{Defining Mechanistic Similarity}
\label{sec:funcsim}
To analyze whether models that rely on different mechanisms for making their predictions exhibit mode connectivity, we must first define a notion of mechanistic similarity between two models. For this purpose, we argue two models are mechanistically similar if they utilize the same attributes of an input to make their predictions (e.g., shape or background). This can be assessed by transforming an input to alter some attribute of interest and thereafter checking if the two models under consideration make the same predictions on these transformed inputs. By using transformations that embody task-relevant vulnerabilities, this definition can be made practically well-motivated. For example, by choosing background randomization as an input transformation, we can assess whether two models rely on the (often) spurious attribute of background to make their predictions (Fig.~\ref{fig:intro}). 

To formalize the intuition above, we describe a generative model of data that can represent input transformations in a general manner. Specifically, we follow prior literature on disentanglement~\citep{locatello2019challenging, locatello2020weakly, gresele2020incomplete, gresele2021independent, von2021self} and Independent Component Analysis (ICA)~\citep{hyvarinen2016unsupervised, hyvarinen2017nonlinear, khemakhem2020variational, khemakhem2021causal}, and assume that there is a latent space $\mathcal{Z} \subset \mathbb{R}^{m}$, with $z$ sampled from a factorizable distribution, $P(z) = \prod_i P(z_i)$, such that each $z$ uniquely maps to samples in the dataset via a generative process $\mathcal{G}:\mathcal{Z} \to \mathcal{X} \times [K]$, i.e., $(x, y) := \mathcal{G}(z)$. If $\mathcal{G}_X$, $\mathcal{G}_Y$ define the components of $\mathcal{G}$ producing $x$ and $y$, the uniqueness of $z$ amounts to assuming invertibility of $\mathcal{G}_X(.)$, denoted as $\mathcal{G}_X^{-1}:\mathcal{X} \to \mathcal{Z}$. Using the notations above, we can model input transformations as counterfactuals generated via systematic interventions on the data-generating process, similar to~\citet{besserve2018counterfactuals, besserve2018group}.

\begin{definition} 
\label{defn:intervene}
\textbf{(Unit Interventions and Counterfactuals.)} A unit intervention $\mathcal{A}_i^{\alpha_i}: \mathcal{Z}_i \times \mathcal{Z}_i \to \mathcal{Z}_i$ on the data-generating process $\mathcal{G}$ is the alteration of the $i^{\text{th}}$ dimension of a latent vector $z$ by setting it to a predefined scalar $\alpha_i \in \mathcal{Z}_i$. Meanwhile, a counterfactual process $\mathcal{E}: \mathcal{X} \times \mathcal{Z}_m \times \dots \times \mathcal{Z}_1 \to \mathcal{X}$ transforms a sample $x$ by changing its corresponding latent vector $z = \mathcal{G}_X^{-1}(x)$ via a set of unit interventions $\widehat{\mathcal{A}} := \{\mathcal{A}_{i}^{{\alpha_i}}\}_{i=1}^{m}$ and mapping it back to the input space, i.e., $\mathcal{E}(x; \widehat{\mathcal{A}}) = \mathcal{G}_X \circ \mathcal{A}_{m}^{\alpha_m} \circ \dots \circ \mathcal{A}_{1}^{\alpha_1} \circ \mathcal{G}_X^{-1}(x)$.\footnote{We slightly abuse notation and assume that unit interventions corresponding to all latent dimensions need \textit{not} be mentioned in $\widehat{\mathcal{A}}$: if a dimension is unmentioned, then its value is unmodified.}
\vspace{-5pt}
\end{definition}
Broadly, unit interventions describe systematic manipulations of the latent space of a generative process, while counterfactuals describe mapping of these manipulations to the observable data space. Note that due to independence of latent dimensions, our definition of unit interventions easily composes and can model other notions of interventions~\citep{scholkopf2021towards, peters2017elements}. Combined with counterfactuals, unit interventions are thus sufficient to model any general input transformations in a formal manner and can be used to characterize the input attributes a network is sensitive to, as shown next. 

\begin{figure}
  \begin{center}
    \includegraphics[width=\linewidth]{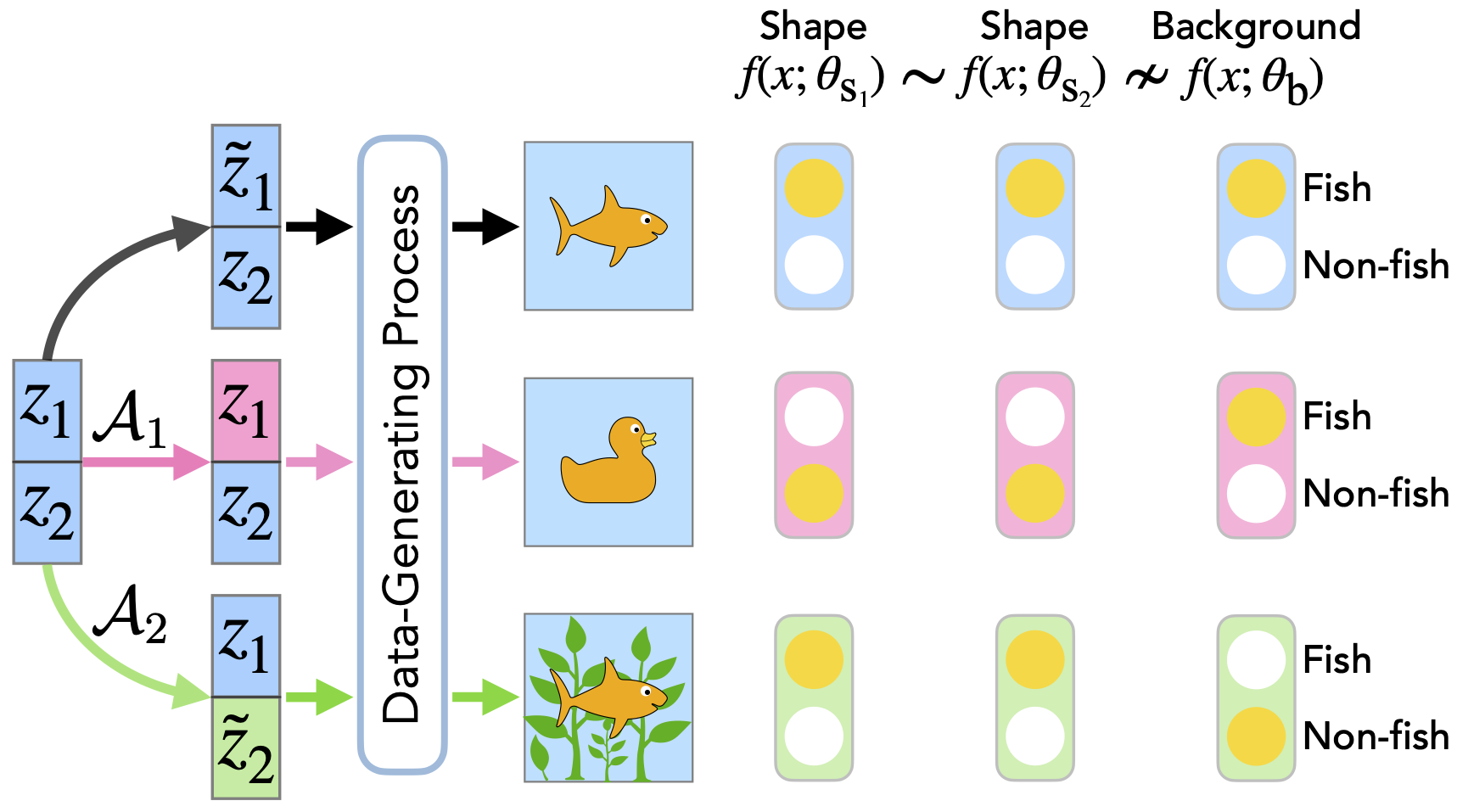}
  \end{center}
  \vspace{-5pt}
  \caption{\label{fig:mechsim}\textbf{Mechanistic Similarity:} We define mechanistic similarity of two models based on how they respond to unit interventions on the data-generating process, i.e., interventions on specific dimensions of the latent vector $z$; e.g., $\mathcal{A}_1$ (shape) and $\mathcal{A}_2$ (background) in the figure. Here, yellow circles represent the prediction of a given model (column) on a counterfactual image (row). Models whose predictions are invariant to the same set of interventions (denoted $\theta_1 \sim \theta_2$) are termed mechanistically similar.}
  \vspace{-15pt}
\end{figure}

\begin{definition} \textbf{(Invariance.)} We say $f(.; \theta)$ is invariant to unit intervention $\mathcal{A}_{i}$ if counterfactuals generated by $\mathcal{A}_{i}$ do not increase its loss, i.e., $\mathcal{L}(f(\mathcal{D}; \theta)) = \mathbb{E}_{\alpha \in \mathcal{Z}_{i}}\mathcal{L}(f(\mathcal{E}(\mathcal{D}; \mathcal{A}_{i}^{\alpha}); \theta))$.
\end{definition}

\begin{proposition}
\label{lem1}
\textbf{(Exhaustiveness of Unit Interventions.)}
If $f(.; \theta)$ is invariant to unit interventions $\mathcal{A}_{i}$ and $\mathcal{A}_{j}$, it must be invariant to their composition. Further, lack of invariance to $\mathcal{A}_{i}$ or $\mathcal{A}_{j}$ precludes invariance to their composition. 
\vspace{-5pt}
\end{proposition}

\begin{figure*}
\centering
\begin{subfigure}{0.22\textwidth}
  \centering
  \centerline{\includegraphics[width=0.85\columnwidth]{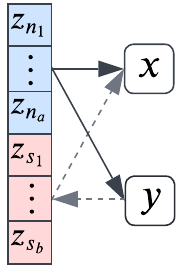}}
\end{subfigure}%
\begin{subfigure}{0.7\textwidth}
  \centering
  \centerline{\includegraphics[width=0.85\columnwidth]{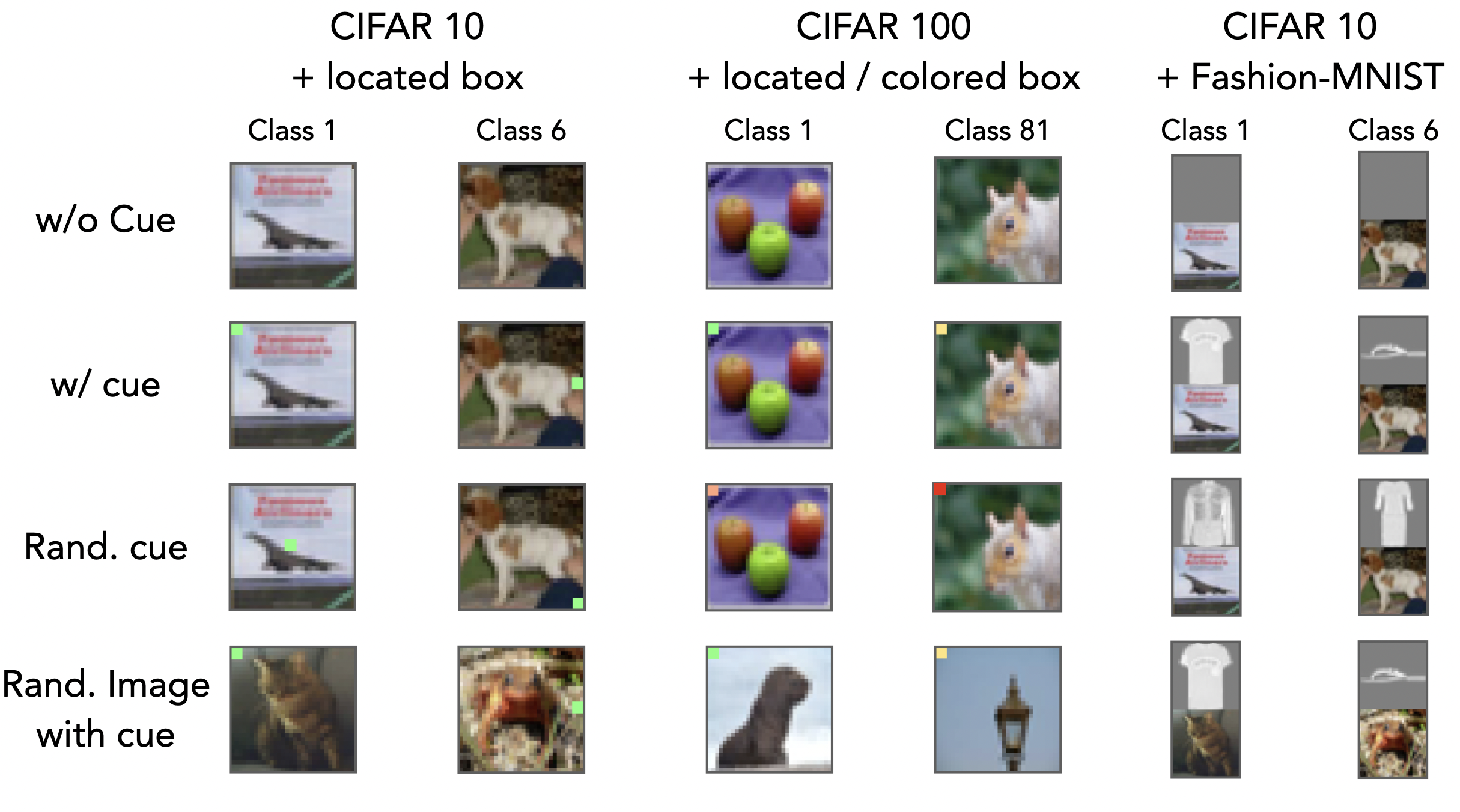}}
\end{subfigure}
  \caption{\label{fig:dgp}
  \textbf{Data-Generating Process (left).} We augment the natural latents $\{z_{n}\}$ of a data-generating process with a set of synthetic latents $\{z_{s}\}$. The attributes induced in the input by these synthetic latents are called \textit{cues}. Conditioning (grey, dotted line) the value of a synthetic latent on the target label ($y$), we can induce correlation between its corresponding cue and the desired model output. If the cue is made easily separable, a DNN will preferentially learn mechanisms to use the cue for making its predictions~\citep{shah2020pitfalls} (see also training curves in App.~\ref{app:setup}). 
  \textbf{Synthetic Datasets (right).} Following the protocol above, we embed synthetic cues in three existing datasets: (1) CIFAR-10 with $3 \times 3$ box cues whose locations depend on the target label; (2) CIFAR-100 with $3 \times 3$ box cues colored according to the first digit of the object label, and located according to the second digit; and (3) Dominoes \citep{shah2020pitfalls}, where CIFAR-10 images are concatenated with Fashion-MNIST images of the same class. We analyze counterfactual datasets that involve removing the cue (\textit{w/o Cue}), keeping it (\textit{w/ cue}), randomizing it (\textit{Rand. cue}), or randomizing the natural image (denoted \textit{Rand.\ image}). These counterfactuals help us ascertain the extent to which a model's prediction relies on natural vs.\ spurious attributes.
 }
  \vspace{-10pt}
\end{figure*}

The above statement shows that studying a model's response to individual unit interventions is sufficient to characterize which attributes of the data a model is using for making predictions: if a model is invariant to a set of unit interventions, it must be invariant to their composition too; similarly, lack of invariance to a unit intervention is sufficient to preclude invariance to all counterfactuals produced by the composition of that intervention and a set of invariant interventions. This result thus helps us circumvent the need for assessing a model's sensitivity to all possible combinations of interventions to fully characterize it. We are now ready to define mechanistic similarity.

\begin{definition}\label{def:mechsim}
\textbf{(Mechanistic Similarity.)} Consider a set of unit interventions $\widehat{\mathcal{A}} := \{\mathcal{A}_{i}\}$, where $i \in [m]$. For parameters $\theta$, denote the subset of interventions that $f(.; \theta)$ is invariant to as $\mathcal{I}(\theta) \subset \widehat{\mathcal{A}}$. Then, $f(.; \theta_1)$ and $f(.; \theta_2)$ are called mechanistically similar if $\mathcal{I}(\theta_1) = \mathcal{I}(\theta_2)$.
\vspace{-5pt}
\end{definition}

Fig.~\ref{fig:mechsim} illustrates mechanistic similarity in an intuitive manner. Formally, given a set of independent transformations (instantiated by use of unit interventions), we say two models are mechanistically similar if they exhibit invariance to the same set of interventions. Our definition shares motivation with the idea of \textit{prediction mismatch}, which involves assessing the number of distinct examples two models produce different predictions on, and has been used in prior work to analyze properties such as calibration and catastrophic interference~\citep{hooker2019compressed, mania2019model, toneva2018empirical, maini2022characterizing}. In contrast, mechanistic similarity is based on assessment of the number of distinct interventions on the data-generating process to which two models are simultaneously invariant. This makes mechanistic similarity more appropriate for problems involving distribution shifts and robustness, where modeling the data-generating process is of crucial importance~\citep{kaur2022modeling}. We next extend the definition of mode connectivity to account for mechanistic similarity of two models.

\begin{definition} \label{def:mechconnect}
\textbf{(Mechanistic Connectivity.)} 
Consider two minimizers $\theta_1$ and $\theta_2$ of loss $\mathcal{L}(f(\mathcal{D}; \theta))$ on a dataset $\mathcal{D}$. Let $\mathcal{E}(\mathcal{D}) := \{\mathcal{E}(\mathcal{D}; \mathcal{A}_{i}^{\alpha_i\sim\mathcal{Z}_i})\}_{i=1}^{m}$ denote a set of counterfactual datasets designed by applying unit interventions $\mathcal{A}_i$ to all points in dataset $\mathcal{D}$, where intervention assignments $\alpha_i$ are chosen uniformly from the respective range of values $\mathcal{Z}_i$. Then, $\theta_1$ and $\theta_2$ are called mechanistically connected along the path $\gamma_{\theta_1 \to \theta_2}(t)$ if, for all counterfactual datasets, they are minimizers that exhibit mode connectivity.
\vspace{-5pt}
\end{definition}
Essentially, if two minimizers exhibit mechanistic connectivity, then there exists a path such that moving along it does not yield increase in loss on the counterfactual dataset described by any pre-defined intervention; that is, all points on the path induce mechanistically similar models. Meanwhile, if two minimizers induce mechanistically dissimilar models, moving along any path between them will necessarily involve a change in the mechanisms used for making predictions. If this change yields increase in loss on an intermediate point on the path between two minimizers, then it is harmful for the distribution shift described by the corresponding intervention. Mechanistic connectivity is defined to succinctly capture this behavior and characterize the connectivities of mechanistically (dis)similar models.
\vspace{-6pt}

\begin{figure*}
\centering
\begin{subfigure}{\textwidth}
  \centering
  \centerline{\includegraphics[width=\textwidth]{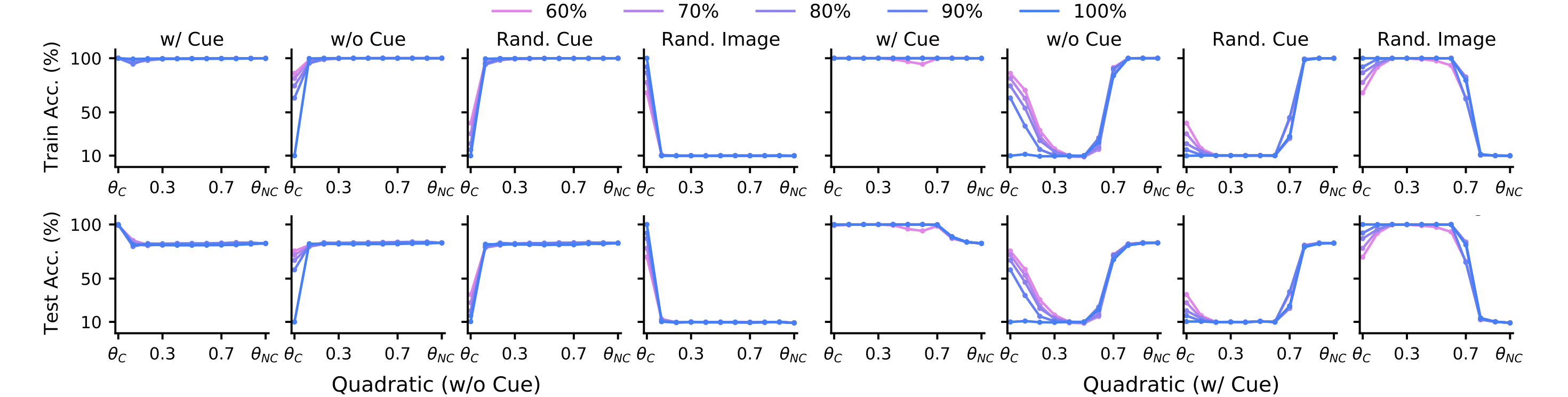}}
  \vspace{0pt}
  \label{fig:1smc_vgg}
\end{subfigure}
\begin{subfigure}{\textwidth}
  \centering
  \centerline{\includegraphics[trim={0 0 0 0 pt}, clip, width=\textwidth]{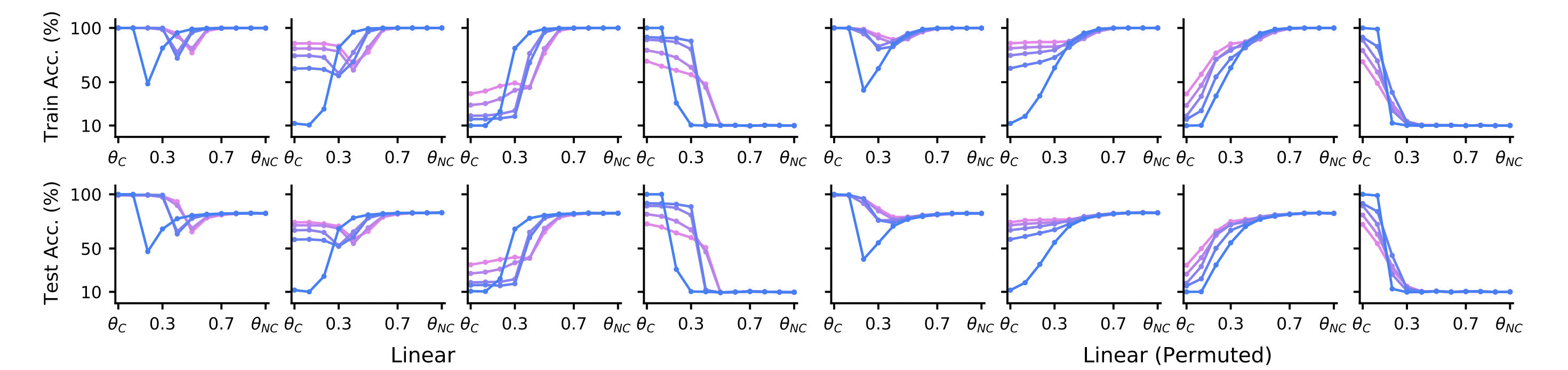}}
  \label{fig:1smc_res18}
\end{subfigure}
\vspace{-14pt}
\caption{\label{fig:smc}\textbf{Non-Linear Mode Connectivity of Mechanistically Dissimilar Models.} We train ResNet-18 models on our synthetic CIFAR-10 datasets with and without box-cues (denoted $\theta_{\text{C}}$ and $\theta_{\text{NC}}$, respectively). We evaluate quadratic and linear connectivity paths; quadratic paths identified using both data with and w/o cues are analyzed. Line colors denote proportion of the training data with synthetic cues. Plot titles denote evaluation data (see Fig.~\ref{fig:dgp}), including data where either the cue is present (w/ Cue), absent (w/o Cue), randomized (Rand.\ Cue), or the underlying image is randomized (Rand.\ Image). As shown, $\theta_{\text{NC}}$ yields the same performance upon randomization of the cue, while the performance of $\theta_{\text{C}}$ decreases substantially; i.e., the two minimizers induce mechanistically dissimilar models. We see: (i) quadratic paths can be easily identified to mode connect mechanistically dissimilar models; (ii) linear paths cannot be identified, even after permutations; and (iii) mechanistic connectivity is unfounded. See App.~\ref{app:smc_results} for similar results on other settings and loss curves.
}
\vspace{-14pt}
\end{figure*}

\vspace{-2pt}
\section{Setup for a Mechanistic Evaluation}
\label{sec:dgpsetup}
Before proceeding further, we discuss how we construct mechanistically dissimilar models and assess mechanistic connectivity between them. This allows us to interleave our formal results with experimental verification and demonstrate the validity of our claims in context. 

\paragraph{Designing mechanistically dissimilar models.} To design models that use different mechanisms for making predictions, we design easily manipulable synthetic datasets that contain multiple viable discriminative attributes. Specifically, our data-generating process is illustrated in Fig.~\ref{fig:dgp} and involves augmenting the natural generative process with synthetic latent variables that are conditioned on the target label. We refer to the attributes induced in the input by such latents as \textit{cues}. By intentionally designing cues that are easily separable, we can exploit the \textit{simplicity bias} of modern DNNs and force our models to preferentially utilize these cues over natural attributes for making their predictions~\citep{shah2020pitfalls}. Training curves for different models are shown in App.~\ref{app:setup} and clearly demonstrate that the process above yields mechanistically dissimilar models: models trained with high correlation between cue and target label rely only on the cue for making predictions, showing invariance to natural attributes; models trained without cues are invariant to them. Importantly, such low-complexity cues can be viewed as stand-ins for spurious or shortcut attributes that are commonplace in realistic settings~\citep{beery2018recognition, geirhos2020shortcut}, allowing us to determine whether minimizers that induce models reliant on spurious vs.\ non-spurious attributes are connected in the landscape. \vspace{-5pt}

\paragraph{Generating counterfactuals for analyzing mechanistic connectivity.} A primary need for our mechanistic analysis of mode connectivity is the ability to generate counterfactuals via unit interventions. To that end, we highlight that the data-generating process defined above is easy to unit-intervene on. Specifically, since the natural attributes and the synthetically embedded cue are controlled by independent latents, the following counterfactual datasets can be generated via valid unit intereventions: (i) \textit{w/ Cue}: identity intervention that does not alter the cue; (ii) \textit{w/o Cue}: removes the cue from the image; (iii) \textit{Rand.\ Cue}: randomizes the cue to break its correlation with the target label (e.g., uniformly changing location of the box in the CIFAR-10 with box cue dataset); and (iv) \textit{Rand.\ Image}: randomizes the natural attributes by altering the underlying image, while keeping the cue intact (e.g., replacing plane with cat). These counterfactuals are especially interesting since they allow us to assess how much a model relies on natural attributes found in the source image vs.\ our synthetically embedded, spurious cues for making its predictions (see Fig.~\ref{fig:dgp}). 
\vspace{-5pt}

\section{Mechanistic Analysis of Mode Connectivity}
\label{sec: mech-connect}
We now demonstrate how mechanistic similarity of two models affects their connectivity patterns in the landscape. We start with the following proposition, which is implied by the results of \citet{nguyen2019connected, simsek2021geometry}, and shows mechanistically dissimilar models can indeed be mode connected.\vspace{-5pt}

\begin{figure*}
\centering
\begin{subfigure}{0.5\textwidth}
  \centering
  \centerline{\includegraphics[width=\columnwidth]{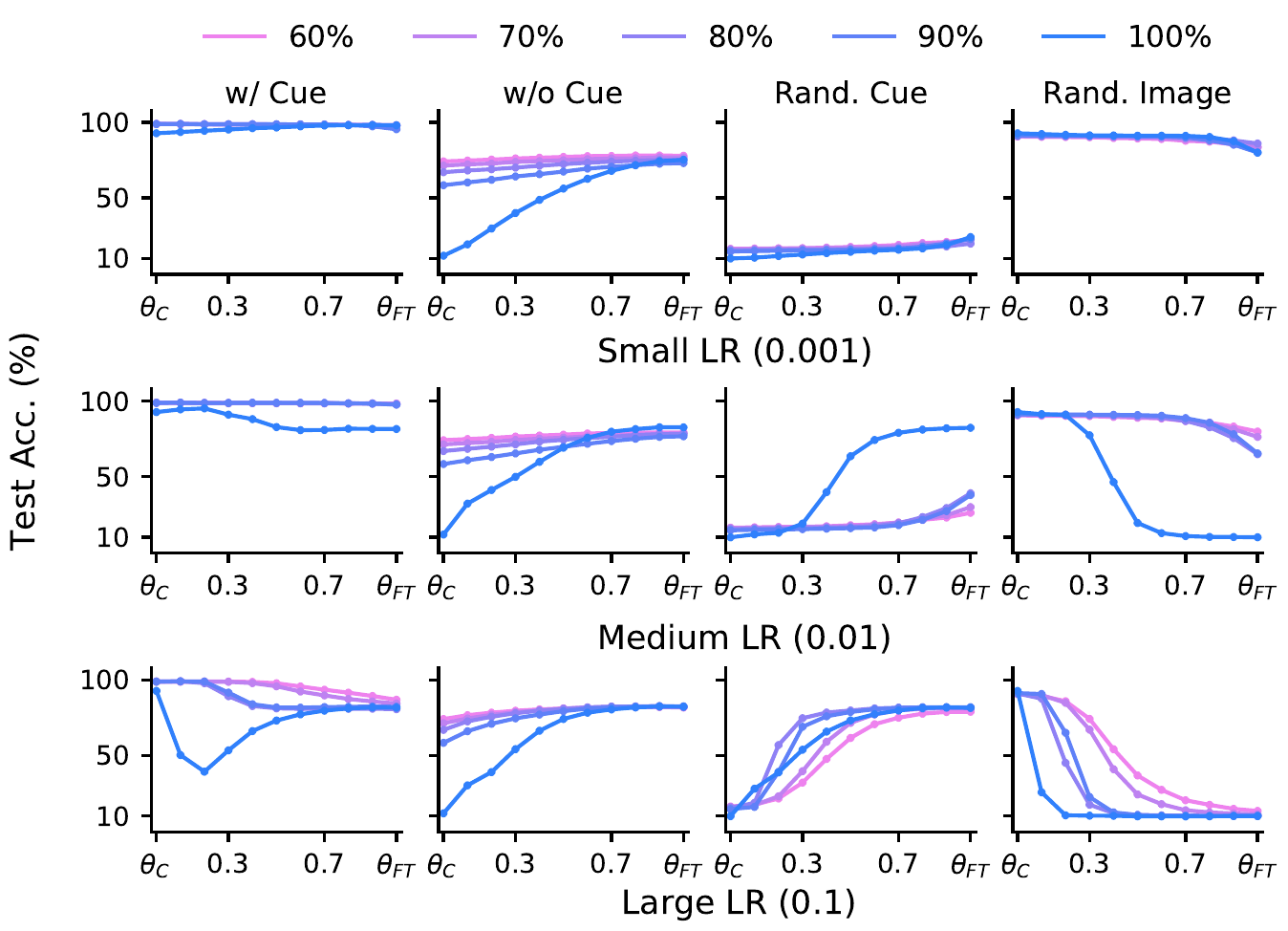}}
  \caption{VGG.}
\end{subfigure}%
\begin{subfigure}{0.5\textwidth}
  \centering
  \centerline{\includegraphics[width=\columnwidth]{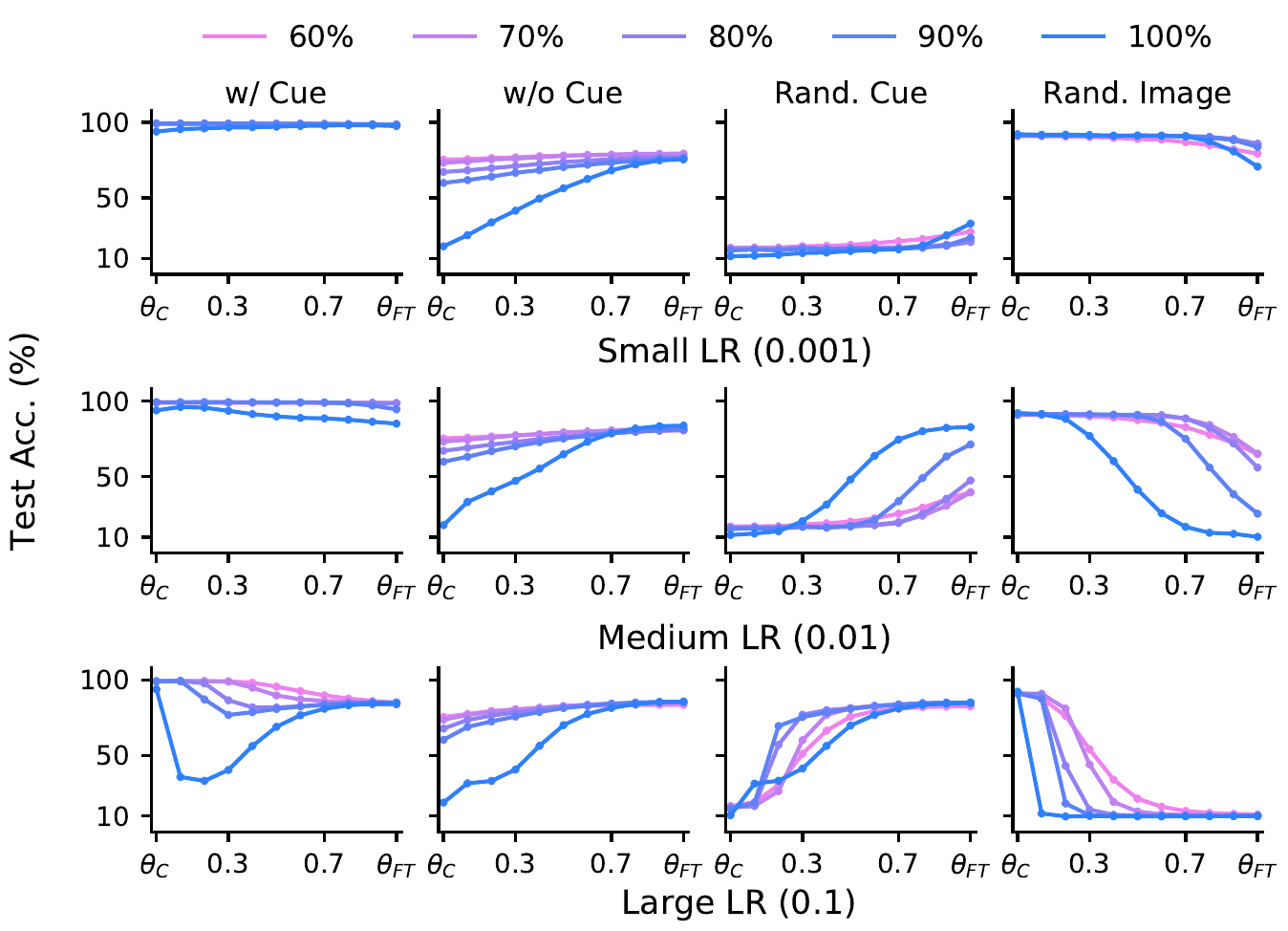}}
  \caption{ResNet18.}
\end{subfigure}
\vspace{-10pt}
\caption{\label{fig:lmc}\textbf{Analyzing Pre-trained vs.\ Fine-Tuned Models: Lack of Linear Connectivity implies Mechanistic Dissimilarity.} 
We train VGG-13 and ResNet-18 models on our synthetic CIFAR-10 dataset with box-cues and perform na\"{i}ve fine-tuning on data without cues for 100 epochs using different initial learning rates (LR) and a step-decay schedule. Corresponding models are denoted $\theta_{\text{C}}$ and $\theta_{\text{FT}}$; line colors denote proportion of dataset with synthetic cues; titles denote evaluation datasets, similar to Fig.~\ref{fig:smc}. We plot test accuracy as a function of location on the linear paths (after permutation). Using a large learning rate or enforcing perfect correlation between the cue and label induces loss barriers along the linear path, i.e., linear mode connectivity does not hold. Simultaneously, the models respond differently to counterfactuals, i.e, they are mechanistically dissimilar and not connected. For a small/medium learning rate, we notice $\theta_{\text{FT}}$ remains linear mode connectivity $\theta_{\text{C}}$ on data with cues. Simultaneously, we see the corresponding models responding similarly on counterfactuals and are mechanistically similar. See App.~\ref{app:lmc_results} for similar results on other datasets, models, and loss curves.}
\vspace{-10pt}
\end{figure*}

\begin{proposition}
\label{claim:all_mimima_connect} 
\textbf{(Mode Connectivity under Mechanistic Dissimilarity.)}
Assume $\theta_1, \theta_2$ are minimizers of the loss on a dataset $\mathcal{D}$ and induce mechanistically dissimilar models. Given sufficient overparameterization, there exists a continuous path along which the minimizers are mode connected.
\end{proposition}

That is, even if two minimizers of loss on a dataset $\mathcal{D}$ induce models that rely on completely distinct mechanisms, \textit{there necessarily exists a continuous path along which the two minimizers exhibit mode connectivity.} 
\vspace{-2pt}

Note, however, the claim above does not yet address the simplicity of these connectivity paths, which is empirically observed to be surprisingly high for minimizers retrieved from the same dataset. To investigate whether this property also holds for mechanistically dissimilar models, we train VGG-13 and ResNet-18 models on the synthetic datasets described in \S\ref{sec:dgpsetup}. We analyze accuracy on counterfactual datasets (see Fig.~\ref{fig:dgp}) along quadratic and linear paths (see Eq.~\ref{eq:paths}), including quadratic paths identified using data with/without cues, linear paths, and linear (permuted) paths. Results for ResNet-18 are shown in Fig.~\ref{fig:smc} and remaining are deferred to App.~\ref{app:smc_results}. 
Interestingly, we find minimizers that induce mechanistically dissimilar models can be mode connected via fairly simple paths as well: \textit{we see we can identify quadratic, but not linear, mode connectivity paths for two mechanistically dissimilar models}. In fact, we conjecture that lack of linear connectivity between two models is intricately related to their mechanistic similarity.

\begin{conjecture}
\label{claim:lmc} 
\textbf{(Lack of Linear Connectivity implies Mechanistic Dissimilarity.)}
If two minimizers $\theta_1$ and $\theta_2$ of the loss $\mathcal{L}(f(\mathcal{D}; \theta))$ on a dataset $\mathcal{D}$ cannot be linear mode connected (up to architectural symmetries), their induced models $f(.; \theta_1), f(.; \theta_2)$ must be mechanistically dissimilar.
\end{conjecture}
\vspace{-5pt}

\input{tables/table_ft.tex}

In App.~\ref{app:proofs}, we show the claim above holds true for a 1-hidden layer model on a simplified data-generating process inspired by our setup. Here, we show extensive empirical evidence of its validity in more complex settings. In particular, we follow the experimental protocol of \citet{neyshabur2020being}, who demonstrate that a pretrained model exhibits linear mode connectivity on the original pretraining dataset before and after fine-tuning on another target dataset. We thus train VGG-13 and ResNet-18 models on our synthetic datasets with (partially) predictive cues and then fine-tune them on data without cues. Results on CIFAR-10 with box cues are shown in Fig.~\ref{fig:lmc}; App.~\ref{app:lmc_results} has additional results. We see that \textit{when linear mode connectivity does not hold, the fine-tuned models behave differently on counterfactuals}, i.e., are mechanistically dissimilar to the pretrained model. For example, the models before and after fine-tuning using a large learning rate do not exhibit linear mode connectivity; correspondingly, the fine-tuned models exhibit clear invariance to cue attributes, while the pretrained models do not. Similarly, under perfect correlation between labels and cue attributes, fine-tuned models are not linear mode connected with their pretrained counterparts, and exhibit different behavior on counterfactuals (even for small initial learning rates).
We note this latter, specific instance of success in altering the pretrained model's mechanisms via fine-tuning is a result of the model being rendered entirely invariant to natural attributes during pretraining (see App.~\ref{app:setup}); consequently, the model lacks any transferable mechanisms for the target data distribution and hence the mechanisms necessarily have to change to fit the new dataset. 

A practical takeaway of our results above is that \textit{na\"{i}ve fine-tuning can fail to alter the mechanisms learned by a model during pretraining.} While large learning rates can help overcome this limitation, they are likely to heavily distort features learned during pretraining~\citep{kumar2022fine}, rendering pretraining obsolete and the sample complexity of fine-tuning similar to that of training from scratch~\citep{he2019rethinking}. 
This indicates that for fine-tuning to be useful, pretraining must be performed with care to ensure desirbale mechanisms relevant to downstream tasks are learned. 
If incorrect mechanisms, such as ones that rely on spurious attributes, are learned, mere fine-tuning on some ``clean'' dataset will be insufficient to alter the model's behavior, as hinted at by results in few recent works~\citep{lovering2021predicting, mireshghallah2022empirical, min2022rethinking, panigrahi2023task}. Intriguingly, this latter strategy of fine-tuning on a clean dataset forms the basis of several recent methods on improving DNNs' robustness~\citep{kirichenko2022last, flkirichenko2022, kumar2022fine, rosenfeld2022domain}: such methods fine-tune some/all layers of a pretrained model on a \textit{minimal} dataset that is known to not contain the spurious attribute we want to reduce the model's reliance on. 
In the next section, we perform a thorough counterfactual evaluation on our synthetic datasets to assess if such methods can actually alter a model's behavior.

\vspace{-5pt}

\section{Altering a Model's Mechanisms Efficiently}
\label{sec:mechanistic fine tuning}
\vspace{-2pt}

In this section, our goal is to show that our newfound understanding of DNN loss landscapes from a mechanistic perspective (see Fig.~\ref{fig:editing}) can be used to devise a sample-efficient strategy that allows targeted altering of a model's mechanisms. We primarily see the results below as further corroboration of our analysis in \S\ref{sec: mech-connect}.
\vspace{-5pt}

\begin{figure}
  \begin{center}
    \includegraphics[width=0.7 \linewidth]{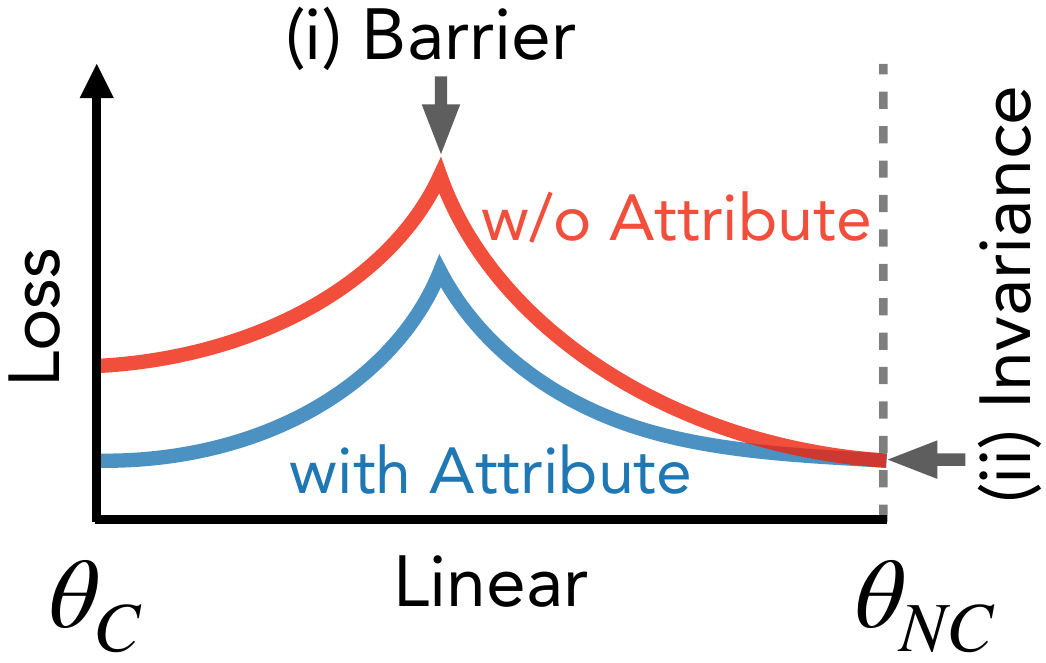}
  \end{center}
  \vspace{-12pt}
  \caption{\label{fig:editing}\textbf{Clues for altering a model's mechanisms.} Given a discriminative attribute $C$, the loss landscape along the linear path connecting a model invariant to the attribute ($\theta_{NC}$) versus a model that relies on the attribute ($\theta_{C}$) generally shows (i) a loss barrier along the path and (ii) invariance at the endpoint corresponding to $\theta_{NC}$, i.e., $\mathcal{L}(\theta_{NC}, D_{NC}) = \mathcal{L}(\theta_{NC}, D_{C})$.}
  \vspace{-15pt}
\end{figure}

\subsection{Connectivity-Based Fine-Tuning (CBFT)} 
\vspace{-5pt}
As defined in our work, mechanistic dissimilarity corresponds to lack of shared invariances. Our results in \S\ref{sec: mech-connect} demonstrate that lack of linear connectivity between two models implies they will be mechanistically dissimilar. \textit{A valid strategy for altering a model's mechanisms then involves moving the model to a region in the landscape that does not exhibit linear connectivity to the current minimizer}. Of course, we specifically want the unshared invariance to correspond to ignoring of the spurious attribute  (denoted $\text{C}$) that we desire to reduce the model's reliance on. For this purpose, we follow prior works and assume access to a \textit{minimal} dataset $\mathcal{D}_{\text{NC}}$ that does \textit{not} contain the attribute $\text{C}$. Note that this setting is not similar to the often used setup in domain adaptation, where the original training dataset (denoted $\mathcal{D}_{\text{C}}$ here) and the novel dataset, $\mathcal{D}_{\text{NC}}$, are assumed to be pairs of images in different environments.

\begin{table}
\centering
\begin{tabular}{c|c|c|c|c} 
& Linear & Nonlinear & Linear & Nonlinear  \\
& Mode & Mode & Mech. & Mech.  \\
\midrule
Mech.\ Similar & \mycmark & \mycmark & \mycmark & \mycmark \\
Mech.\ Dissimilar & \myxmark$^{*}$ & \mycmark & \myxmark & \myxmark \\
\end{tabular}
\caption{\textbf{Summarizing our Findings.} \mycmark, \myxmark$\,$ respectively indicate whether there always exist paths along which mechanistically (dis)similar models identified using gradient-based optimization can exhibit the type of connectivity specified in the column title. $^{*}$ denotes there are exceptional, but primarily theoretical, cases where the connectivity definition can hold (see App.~\ref{app:proofs}).
}
\label{tab:check_marks}
\end{table}

In the following, we use $\mathcal{D}^{i}$ to denote the subset of examples in dataset $\mathcal{D}$ belonging to the $i^{\text{th}}$ class in a $K$-class classification problem, $\gamma_{\theta \to \theta_{\text{C}}}(t)$ to denote the linear path between a set of parameters $\theta$ and the pretraining solution $\theta_{\text{C}}$, and $f_{r}(x; \theta)$ to denote the model's representation for an input $x$ at the penultimate layer. Let $\mathcal{N}_{\text{Tr}}$ denote the \citet{trunc} with mean/std of $0.5$ that is constrained to the range $[0, 1]$. Our method, Connectivity-Based Fine-Tuning (CBFT), involves minimizing the following loss:
\begin{equation}
\label{eq:cbft}
\begin{split}
    &\mathcal{L}_{\text{CBFT}}
    = \mathcal{L}_{\text{CE}}(f(\mathcal{D}_{\text{NC}}; \theta), y) + \mathcal{L}_{\text{B}} + \frac{1}{K} \mathcal{L}_{\text{I}} \text{, where}\\
    &\mathcal{L}_{\text{B}} = \mathbb{E}_{t \sim \mathcal{N}_{\text{Tr}}}|\lambda_\text{B} - \mathcal{L}_{\text{CE}}(f(\mathcal{D}_{\text{C}}; \gamma_{\theta \to \theta_{\text{C}}}(t)), y)|  \text{  and} \\ 
    &\mathcal{L}_{\text{I}} =  \sum_{k=1}^{K} \norm{\mathbb{E}_{x \in \mathcal{D}^{k}_{\text{C}}} (f_r(x;\theta)) - \mathbb{E}_{\Tilde{x} \in \mathcal{D}^{k}_{\text{NC}}}(f_r(\Tilde{x};\theta))}_{2}^{2}.
\end{split}
\end{equation}

Here $\mathcal{L}_{\text{CE}}$ denotes the cross-entropy loss and promotes learning of correct labels on the minimal dataset $\mathcal{D}_{\text{NC}}$, while $\mathcal{L}_{\text{B}}$, $\mathcal{L}_{\text{I}}$ instantiate the two principles discussed in Fig.~\ref{fig:editing}: $\mathcal{L}_{\text{B}}$ denotes a ``barrier loss'' that randomly samples a point on the linear path between $\theta$, $\theta_{C}$ and maximizes the loss at this point up to an upper bound $\lambda_\text{B}$ ($=$1 in all our experiments) and $\mathcal{L}_{\text{I}}$ denotes an invariance loss that promotes reducing the distance between class-average representations on $\mathcal{D}_{\text{NC}}$ and $\mathcal{D}_{\text{C}}$. 
Overall, $\mathcal{L}_{\text{B}}$ helps CBFT find a set of parameters $\theta$ that does not exhibit linear connectivity to $\theta_{C}$, while $\mathcal{L}_{\text{I}}$ helps CBFT pick a solution that is (approximately) invariant to attribute $C$. We emphasize that since the cross entropy loss can be made arbitrarily large, using the hyperparameter $\lambda_\text{B}$ is important. We also note that using class-average representations to learn (approximately) invariant representations has the advantage of not requiring  access to simultaneous pairs of samples in different environments, i.e., ones with and without the spurious attributes~\citep{li2018domain, sun2016deep}.

\textbf{Evaluating CBFT:} We empirically validate the effectiveness of CBFT by using our synthetic datasets from \S\ref{sec:dgpsetup} as a benchmark. We compare CBFT against  na\"{i}ve fine-tuning, Last-Layer Re-Training (LLR)~\citep{kirichenko2022last, flkirichenko2022}, and Linear Probe plus Fine-Tuning (LPFT)~\citep{kumar2022fine} (see App.~\ref{app:ftdetails} for implementation details). Results are reported in Tab.~\ref{tab:CBFT}. We see that while the baselines perform well on clean data, \textit{they do not yield desired behavior on counterfactual datasets}: e.g., they achieve high accuracy even if we randomize the image, indicating that they are more sensitive to the cue. In contrast, we see that beyond just performing well on clean data, CBFT models show the desired behaviors: sensitivity to randomization of the image and invariance to spurious attributes. These results suggest CBFT successfully alters a model's mechanisms and provides further corroboration to the claim that lack of linear connectivity implies mechanistic dissimilarity between two models (see Conj.~\ref{claim:lmc}). We also provide detailed ablations for CBFT in App.~\ref{app:ablations} and find both losses, $\mathcal{L}_{\text{B}}$ and $\mathcal{L}_{\text{I}}$, are important for getting the desired results: the barrier loss helps induce a mechanistically dissimilar model, while the invariance loss helps select the mechanisms we desire.
\vspace{-5pt}

\section{Related Work}
\paragraph{Mode connectivity.}
Existence of a single, continuous manifold connecting global minimizers was first identified theoretically by \citet{freeman2016topology, nguyen2019connected} and empirically discovered in concurrent works under the title of ``mode connectivity'' by \citet{garipov2018loss} and \citet{draxler2018essentially}. A geometrical characterization of this manifold was provided by \citet{simsek2021geometry}, who showed the manifold is \textit{primarily} composed of affine subspaces. Connectivity properties of neural networks have been used for designing and analyzing algorithms for several practically relevant applications, such as ensembling~\citep{benton2021loss, izmailov2018averaging, wortsman2021learning, wortsman2022model}, network pruning~\citep{frankle2020linear, entezari2021role}, optimization~\citep{kaddourflat}, adversarial robustness~\citep{zhao2020bridging}, and multi-task/continual learning~\citep{mirzadeh2020linear, lubana2021quadratic}. During the course of this work, we became aware of the contemporary work by \citet{juneja2022linear}. 
Therein, the authors use NLP datasets designed by \citet{mccoy2019right} to perform an empirical analysis similar to ours, finding that models that lack linear connectivity show different generalization behaviors, relying on different attributes of an input to make their predictions. Our work further formalizes this result: we show lack of linear connectivity implies mechanistic dissimilarity. The results by Juneja et al.\ thus provide further corroboration for our claims on a different modality.

\paragraph{Fine-tuning and Model Editing.} Fine-tuning is a well-established practice in deep learning. The most basic fine-tuning method is to treat the pretrained model as an initialization, and continue training with new data. A variant is to train only a subset of parameters, such as the final classification layer~\citep{kirichenko2022last, flkirichenko2022}, possibly fine-tuning the entire model after that~\citep{kumar2022fine, rosenfeld2022domain}. A related application to fine-tuning, model editing has become quite popular recently and approaches for the same generally aim to make a targeted change to a model's factual knowledge~\citep{mitchell2021fast, santurkar2021editing, sinitsin2020editable}. For instance, \citet{sinitsin2020editable} give the example of correcting a model's prediction error on a particular example without changing its predictions on other examples. Prior work on model editing aims to make changes that are ``local'' in input space, e.g., only affecting the model's ``understanding'' of who the current prime minister of the UK is. CBFT shares this motivation of ``targeted'' alteration of a model; however, instead of altering the model's factual knowledge, the overarching goal of CBFT is to make changes to the specific rules or mechanisms the model implements to make its predictions (see \citet{dasgupta2022distinguishing} for a discussion on distinction between rule vs.\ exemplar / factual inference strategies). Specifically, CBFT aims to make a model invariant to features that it was not already invariant to (or vice versa), without changing any of its other learned invariances. 

\section{Conclusion and Future Work}

Depending on the mechanisms they learn for making their predictions, neural networks trained on a specific data distribution can nonetheless differ vastly in their behavior when evaluated on other distributions. This realization prompted us to perform a mechanistic characterization of connectivity properties in the loss landscape of neural networks. Our proposed notion of \textit{mechanistic similarity} instantiates the idea as shared invariances, and helps extend the prior notion of mode connectivity to account for mechanistic similarity. Our analysis reveals several surprising findings (see Tab.~\ref{tab:check_marks}): (i) mechanistically dissimilar minimizers can be mode connected via relatively simple, but non-linear, paths; (ii) linear mode connectivity of two minimizers is intricately related to the mechanistic similarity of their induced models; (iii) na\"{i}ve fine-tuning can fail to eliminate spurious attributes learned during pretraining; and (iv) finding linearly disconnected regions in the landscape enables sample-efficient alteration of a model's mechanisms.

Future work can involve use of counterfactual generators based on modern generative models~\citep{thiagarajan2021designing} to extend our synthetic data experiments and
corroborate our claims in naturalistic settings. We also believe our analysis can be useful to reason about benefits and limitations of recent averaging-based ensembling methods~\citep{wortsman2022robust, wortsman2022model, rame2022diverse, arpit2021ensemble}. Specifically, note that our claims do not preclude possible linear connectivity of mechanistically dissimilar models: in fact, any two solutions of the linear system of equations $y = W x$ can be interpolated regardless of their prediction mechanisms (hence the $^{*}$ in Tab.~\ref{tab:check_marks}). However, as we show in App.~\ref{app:proofs} in a simplified setup, these different mechanisms should be of similar ``complexity'' to enable linear connectivity (e.g., mechanisms corresponding to linearly separable attributes). In the context of our fine-tuning results, this implies na\"{i}ve fine-tuning can work well on a target distribution only if the desired mechanism is of similar complexity to the mechanism for identifying the spurious attribute (which would possibly imply it finds a spurious attribute again); otherwise, a loss barrier must be surmounted for successful learning on the target distribution. This suggests that pretraining should aim to promote learning of a variety of expressive prediction mechanisms, which can be challenging in practice~\citep{d2020underspecification}.

\section*{Acknowledgements}
We thank Anna Golubeva and Vasudev Shyam for several useful discussions during the course of this project. We also thank Yamini Bansal, Nikunj Saunshi, Puja Trivedi, Liu Ziyin, Cindy Wu, Robert Kirk, Bruno Kacper, Tegan Maharaj, and Daniel A. Roberts for useful feedback on the paper. ESL was partially supported via NSF under award CNS-2008151.

\section*{Authors' Contributions}
\label{sec:contribs}
ESL led the section on mechanistic similarity, refining it with HT and DK. ESL, DK, and HT discussed relating mechanistic similarity with mode connectivity, kickstarting the project. ESL and DK motivated studying fine-tuning from the lens of mechanisms and mode-connectivity. ESL, EJB, DK, and HT conceived the experimental setup; ESL and EJB co-led implementation/evaluation. ESL wrote the primary draft; ESL, EJB, and HT conceived and designed the figures, with inputs from DK; and ESL, HT, DK, and RPD refined the paper together. ESL led the theoretical analysis; ESL and HT refined it together, with inputs from DK and RPD. 


\bibliography{mmc}
\bibliographystyle{icml2023}

\newpage
\appendix
\onecolumn

\section*{Appendix}



\section{Detailed Related Work}
\textbf{Mode connectivity.}
Existence of a single, continuous manifold connecting global minimizers was first identified theoretically by \citet{freeman2016topology, nguyen2019connected} and empirically discovered in concurrent works under the title of ``mode connectivity'' by \citet{garipov2018loss} and \citet{draxler2018essentially}. A geometrical characterization of this manifold was provided by \citet{simsek2021geometry}, who showed the manifold is \textit{primarily} composed of affine subspaces. Connectivity properties of neural networks have been used for designing and analyzing algorithms for several practically relevant applications, such as ensembling~\citep{benton2021loss, izmailov2018averaging, wortsman2021learning, wortsman2022model}, network pruning~\citep{frankle2020linear, entezari2021role}, optimization~\citep{kaddourflat}, adversarial robustness~\citep{zhao2020bridging}, and multi-task/continual learning~\citep{mirzadeh2020linear, lubana2021quadratic}. During the course of this work, we became aware of the contemporary empirical paper by \citet{juneja2022linear}, who investigate whether minimizers connected via linear paths follow similar ``decision rules''. Their analysis focuses on NLP tasks and does not involve modeling the data-generating process or counterfactual evaluation; their results can be regarded as use of an alternative strategy to further verify our claims on a different modality. 

\textbf{Fine-tuning.} Fine-tuning is a well-established practice in deep learning. The most basic fine-tuning method is to treat the pretrained model as an initialization, and continue training with new data. A variant is to train only a subset of parameters, such as the final classification layer~\citep{kirichenko2022last, flkirichenko2022}, possibly fine-tuning the entire model after that~\citep{kumar2022fine, rosenfeld2022domain}. 

\textbf{Model editing.} A related application to fine-tuning, model editing has become quite popular recently and approaches for the same generally aim to make a targeted change to a model's factual knowledge~\citep{mitchell2021fast, santurkar2021editing, sinitsin2020editable}. For instance, \citet{sinitsin2020editable} give the example of correcting a model's prediction error on a particular example without changing its predictions on other examples. Prior work on model editing aims to make changes that are ``local'' in input space, e.g., only affecting the model's ``understanding'' of who the current prime minister of the UK is. CBFT shares this motivation of ``targeted'' alteration of a model; however, instead of altering the model's factual knowledge, the overarching goal of CBFT is to make changes to the specific rules or mechanisms the model implements to make its predictions (see \citet{dasgupta2022distinguishing} for a discussion on distinction between rule vs.\ exemplar / factual inference strategies). Specifically, CBFT aims to make a model invariant to features that it was not already invariant to (or vice versa), without changing any of its other learned invariances. This difference in goals make model editing approaches inappropriate for our setup.

\textbf{Use of synthetic datasets.} Our data-generation pipeline was influenced by several past works that use synthetic datasets for better understanding topics such as transfer learning~\citep{dittadi2020transfer}, domain generalization~\citep{wiles2021fine, van2019disentangled, arjovsky2019invariant}, disentanglement~\citep{higgins2016beta, klindt2020towards}, self/semi supervised learning~\citep{von2021self, trivedi2022analyzing, trivedi2022augmentations, locatello2020weakly}, and inductive biases of neural networks~\citep{hermann2020origins, hermann2020shapes, ritter2017cognitive}.

\section{Training Details and Datasets}
\label{app:setup}

When training from scratch (e.g., in Fig.~\ref{fig:smc}), we train models using SGD for 100 epochs with a batch-size of 256, momentum of 0.9, and weight decay of $10^{-4}$. Learning rate starts at 0.1 and is dropped by a factor of $10$ at the 40$^{\text{th}}$ and 80$^{\text{th}}$ epochs. No data augmentations are used. When fine-tuning to assess linear connectivity in Fig.~\ref{fig:lmc}, we train models for a further 100 epochs on data without cues using different initial learning rates, but the same step-decay schedule (decay factor of 0.1 at decay epochs 40 and 80). For details on training and evaluation of models in Tab.~\ref{tab:CBFT}, please refer to App.~\ref{app:ftdetails}.

\subsection{Dataset Visualizations and Training Curves}
When using synthetic datasets, if a proportion $c$ of samples is to be assigned the cue feature, we use the first $c$\% samples of all classes to assign them the respective cues. We do not store the samples beforehand; instead, we use manually designed PyTorch data-loaders that allow for easy manipulation of samples in an online manner, enabling straightforward counterfactual evaluations. While the dataset construction was discussed in Fig.~\ref{fig:dgp} and \S\ref{sec:dgpsetup}, we provide several visualizations of randomly sampled datapoints from different classes and their counterfactuals in Fig.~\ref{fig:c10viz} (CIFAR-10 with box cue), Fig.~\ref{fig:c100viz} (CIFAR-100 with box/color cue), and Fig.~\ref{fig:dominoesviz} (Dominoes: CIFAR-10 with concatenated FashionMNIST image cue). Learning curves with train/test accuracies for VGG / ResNet-18 models trained on different proportions of samples with cue features for these datasets are reported in Figs.~\ref{fig:curve_c10_vgg} and \ref{fig:curve_c10_res18} (CIFAR-10 with box cue), Fig.~\ref{fig:curve_c100_vgg},~\ref{fig:curve_c100_res18} (CIFAR-100 with box/color cue), and Fig.~\ref{fig:curve_dominoes_vgg},~\ref{fig:curve_dominoes_res18} (Dominoes). We note that our data-generation pipeline was heavily influenced by several past works that use synthetic datasets for better understanding topics such as transfer learning~\citep{dittadi2020transfer}, domain generalization~\citep{wiles2021fine, van2019disentangled, arjovsky2019invariant}, disentanglement~\citep{higgins2016beta, klindt2020towards}, self/semi supervised learning~\citep{von2021self, trivedi2022analyzing, trivedi2022augmentations, locatello2020weakly}, and inductive biases of neural networks~\citep{hermann2020origins, hermann2020shapes, ritter2017cognitive}. 

\begin{figure}[H]
\centering
\centerline{\includegraphics[width=0.85\columnwidth]{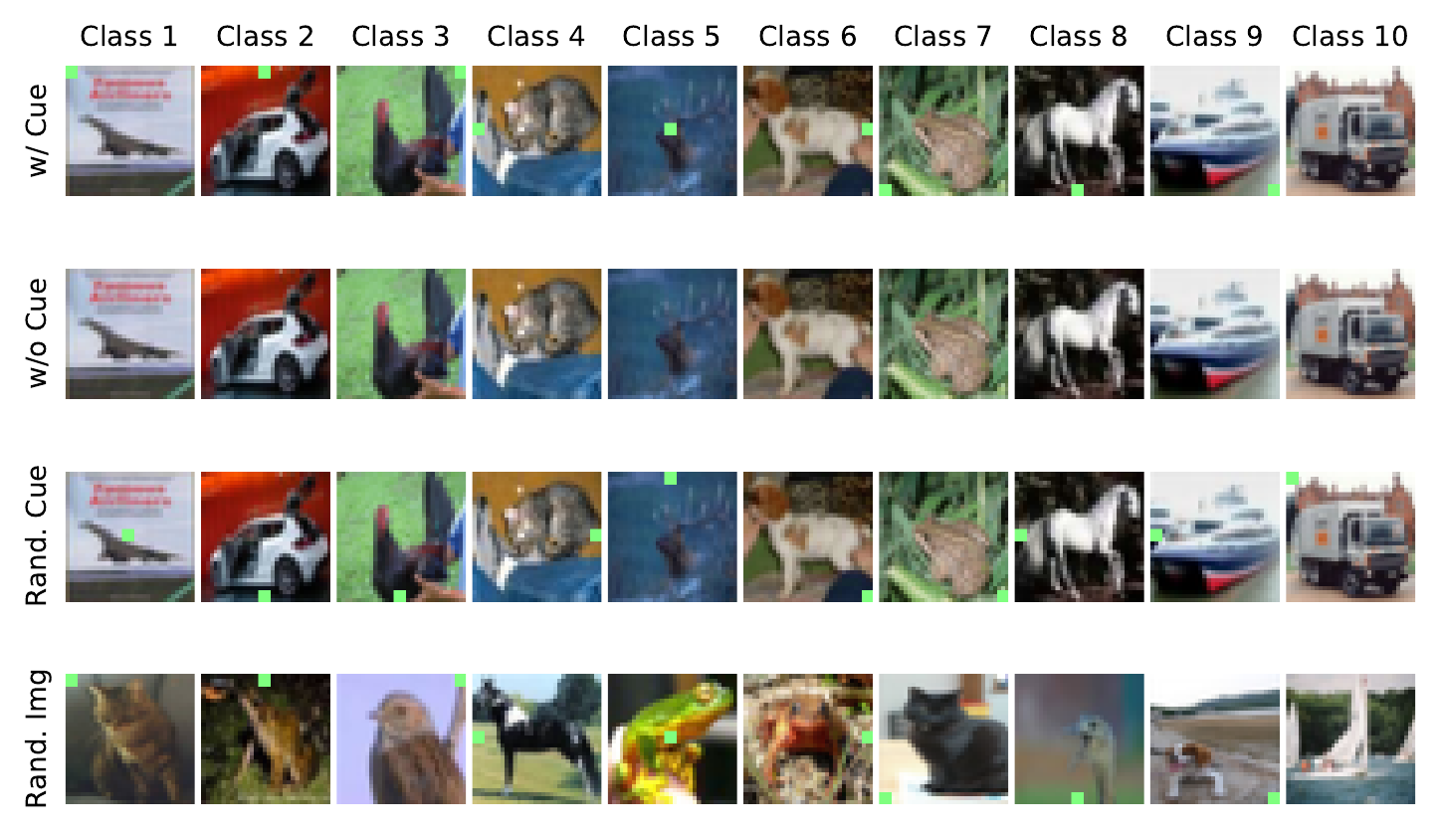}}
\caption{\label{fig:c10viz} CIFAR-10 with Box cue.}
\end{figure}

\begin{figure}[H]
\centering
\centerline{\includegraphics[width=0.85\columnwidth]{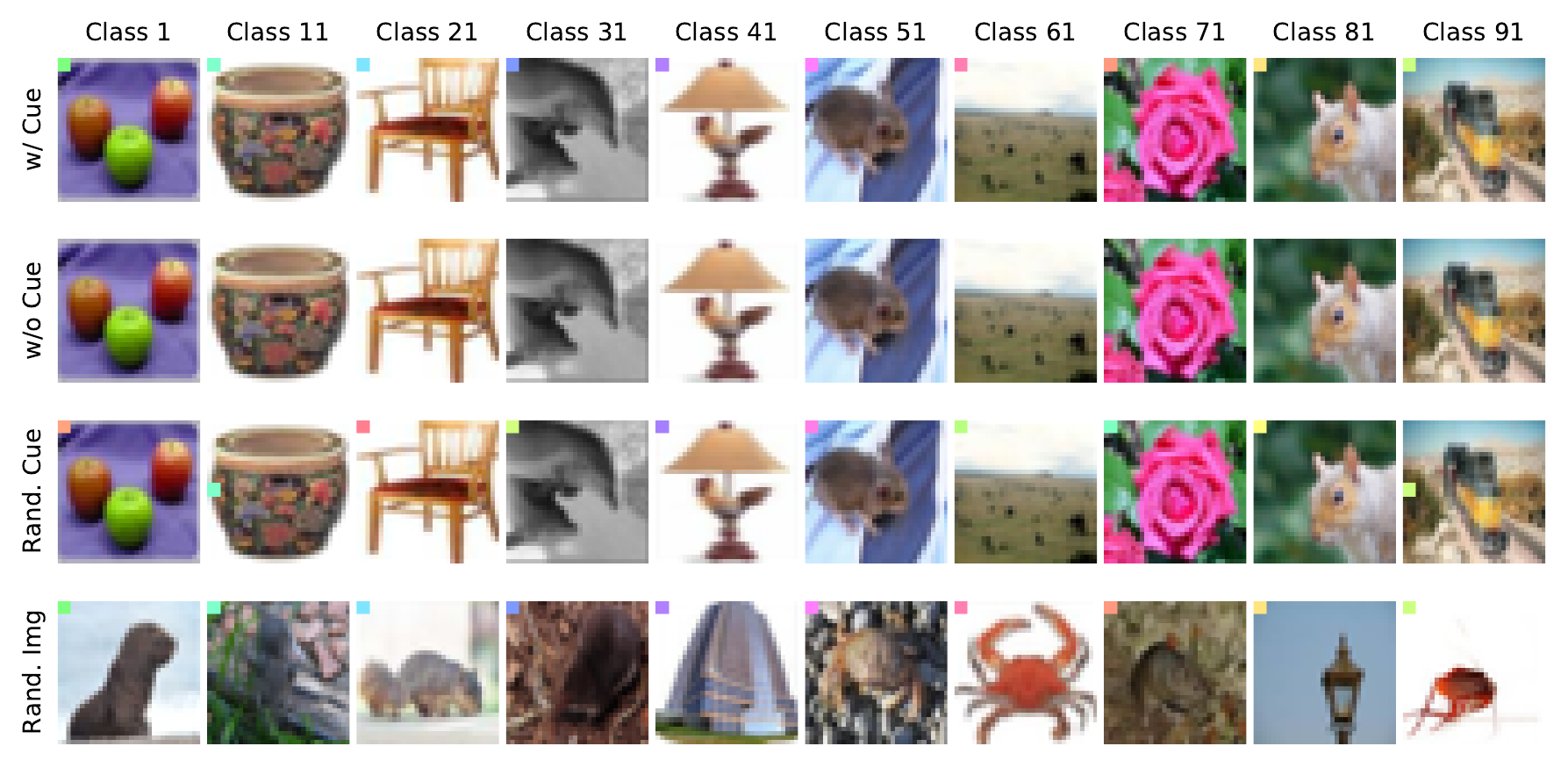}}
\caption{\label{fig:c100viz} CIFAR-100 with Box/Color cue.}
\end{figure}

\begin{figure}[H]
\centering
\centerline{\includegraphics[width=0.8\columnwidth]{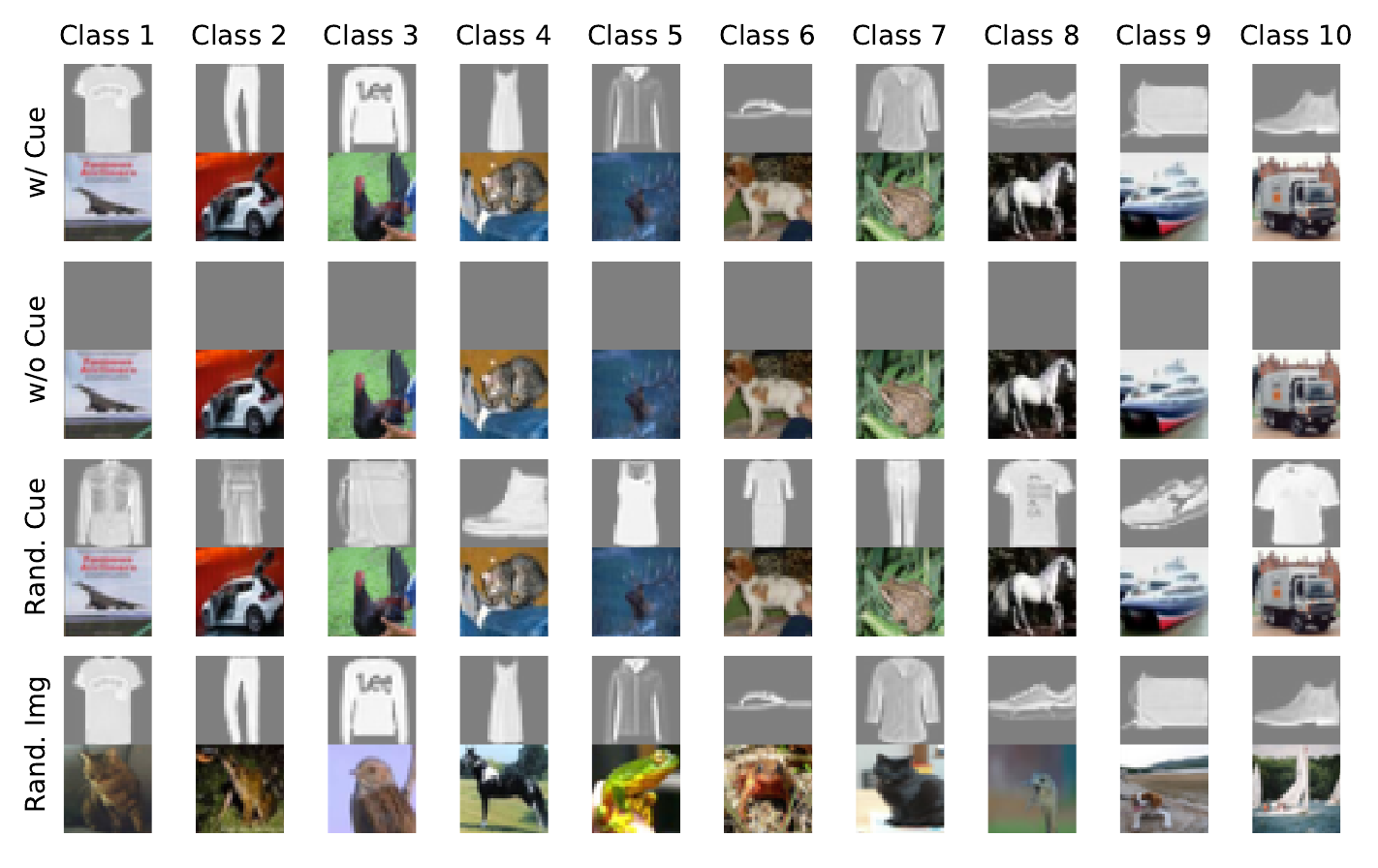}}
\caption{\label{fig:dominoesviz} Dominoes: CIFAR-10 with their corresponding ID image from Fashion-MNIST as the cue.}
\vspace{30pt}
\end{figure}

\begin{figure}[H]
\centering
\begin{subfigure}[b]{0.95\textwidth}
\centerline{\includegraphics[width=\textwidth]{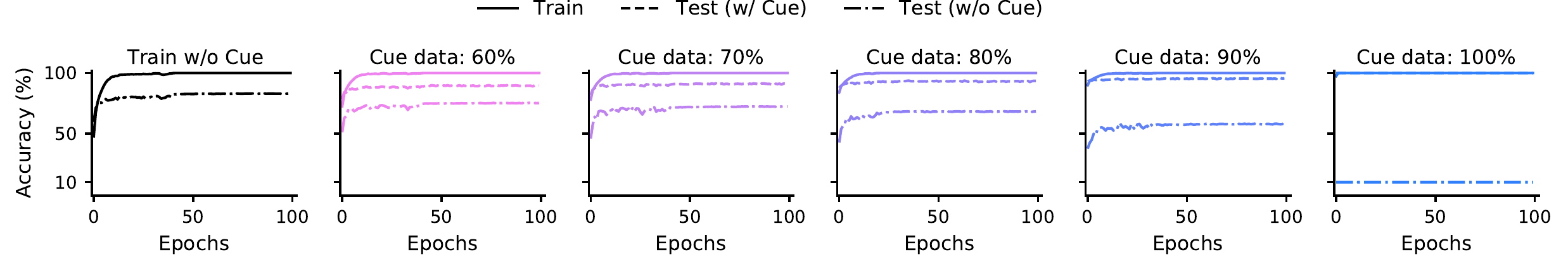}}
\caption{\label{fig:curve_c10_vgg} CIFAR-10 with box cues, wherein the box's location is a function of the target label.}
\vspace{5pt}
\end{subfigure}
\begin{subfigure}[b]{0.95\textwidth}
\centerline{\includegraphics[width=\textwidth]{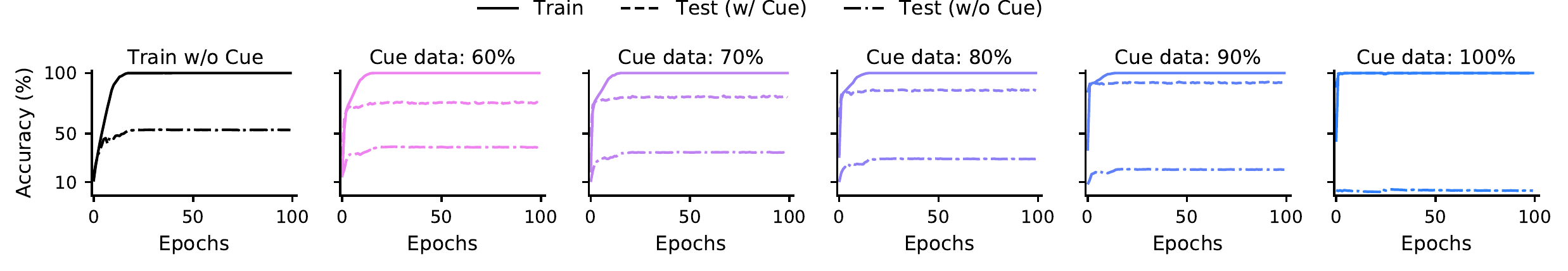}}
\caption{\label{fig:curve_c100_vgg} CIFAR-100 with box cues, wherein the box's location and color are a function of the target label.}
\vspace{5pt}
\end{subfigure}
\begin{subfigure}[b]{0.95\textwidth}
\centerline{\includegraphics[width=\textwidth]{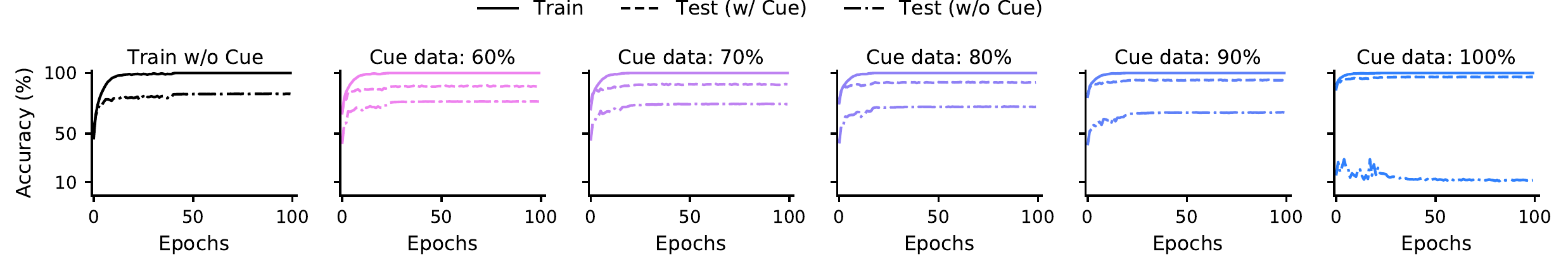}}
\caption{\label{fig:curve_dominoes_vgg} Dominoes, wherein Fashion-MNIST images are appended to CIFAR-10 images and act as the spurious cues.}
\vspace{5pt}
\end{subfigure}
\caption{Learning curves for VGG-13 models.}
\end{figure}

\begin{figure}[H]
\centering
\begin{subfigure}[b]{0.95\textwidth}
\centerline{\includegraphics[width=\textwidth]{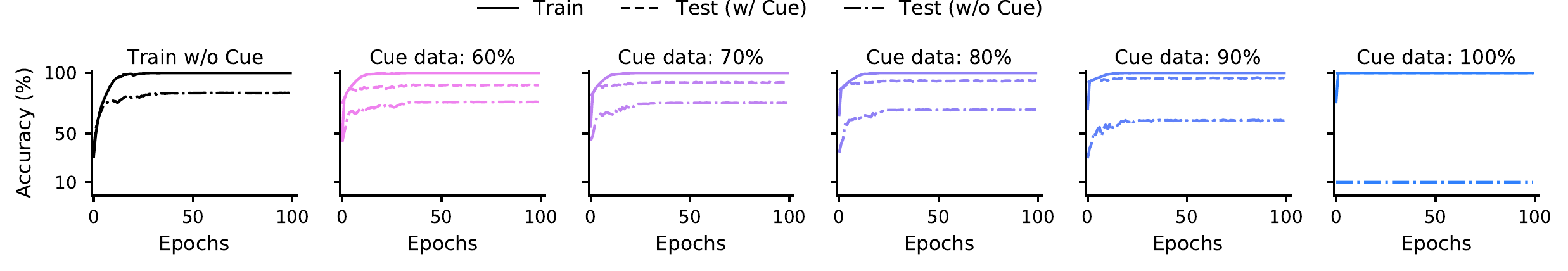}}
\caption{\label{fig:curve_c10_res18} CIFAR-10 with box cues, wherein the box's location is a function of the target label.}
\vspace{5pt}
\end{subfigure}
\begin{subfigure}[b]{0.95\textwidth}
\centerline{\includegraphics[width=\textwidth]{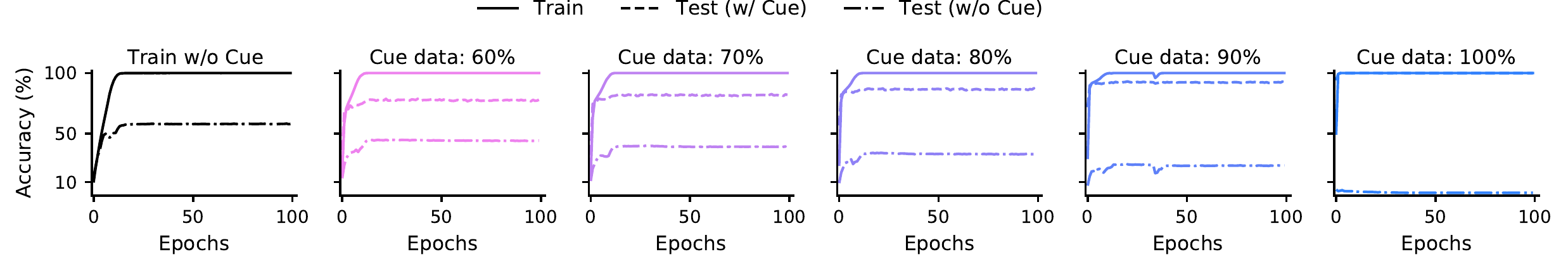}}
\caption{\label{fig:curve_c100_res18} CIFAR-100 with box cues, wherein the box's location and color are a function of the target label.}
\vspace{5pt}
\end{subfigure}
\begin{subfigure}[b]{0.95\textwidth}
\centerline{\includegraphics[width=\textwidth]{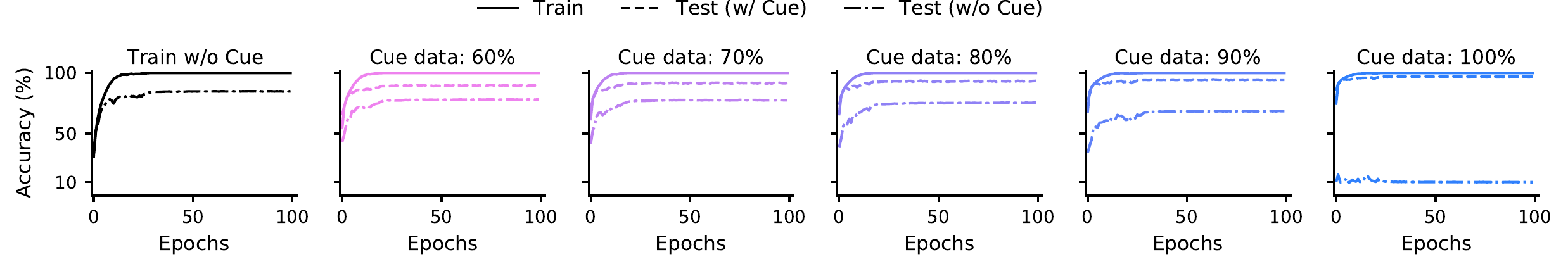}}
\caption{\label{fig:curve_dominoes_res18} Dominoes, wherein Fashion-MNIST images are appended to CIFAR-10 images and act as the spurious cues.}
\vspace{5pt}
\end{subfigure}
\caption{Learning curves for ResNet-18 models.}
\end{figure}

\subsection{Training details for Tab.~\ref{tab:CBFT}}
\label{app:ftdetails}
We train models using SGD on the synthetic data with cue features (47500 samples), reserving remaining 2,500 training samples as ``clean'' data. We emphasize that since the underlying images (i.e., ones without cues) are independent in the two sets, this setup is different from domain generalization methods that use simultaneous pairs of images in different environments to learn invariant representations. 

Depending on the method, the fine-tuning setup involves different hyperparameters. For consistency, we follow \citet{kirichenko2022last} and \citet{kumar2022fine} in using a cosine schedule for fine-tuning on clean data.

\textbf{Na\"{i}ve Fine-Tuning}. We use different initial learning rates, including medium ($0.01$) and small ($0.001$). For a large learning rate, we note that while fine-tuning on a minimal set induces good invariance properties, the performance on the original, without cue data (called NC in tables) is often rather poor (hence we omit those results to save space). This behavior is expected since use of a large learning rate renders the fine-tuning pipeline essentially equivalent to training from scratch, degrading its sample efficiency~\citep{he2019rethinking, kumar2022fine}. 

\textbf{LLRT} \citep{kirichenko2022last}. We freeze the model parameters at their current state, remove the final linear layer, and replace it with a randomly initialized one. The layer is fine-tuned on clean data for 100 epochs with a cosine decay schedule that starts at a LR of 30.

\textbf{LPFT} \citep{kumar2022fine}. First, we follow the protocol above for LLRT to produce a new linear layer. Thereafter, the entire model is fine-tuned on clean data for 20 epochs with initial learning rates of $0.01$, $0.001$, and $0.0001$. The best retrieved results on validation data are reported.

\textbf{CBFT}. We run CBFT for 20 epochs, using an initial learning rate of $0.01$ with a cosine decay schedule (similar to the baselines). The method turns out to be fairly robust to the exact values of $\lambda_1$; we fix it to 1 for all experiments without any explicit tuning therefore. We use a truncated Gaussian distribution center at 0.5 because this helps induce a loss barrier at the center of the linear path between the parameters we are trying to identify and the original, pretraining ones. Truncation is necessary so that only interpolations between the parameters is used. 

We also note that since training will yield gradients for the model that has parameters $\gamma_{\theta \to \theta_C}(t)$, we need to explicitly compute the gradients for $\theta$ by using the following relationship for some general objective function $\mathcal{L}$: 
\begin{equation*}
    \nabla_{\theta} \mathcal{L}(\gamma_{\theta \to \theta_{C}}(t)) = \left(\nabla_{\theta} \gamma_{\theta \to \theta_{C}}(t)\right)^{T} \, \nabla_{\gamma_{\theta \to \theta_{C}}(t)} \mathcal{L}(\gamma_{\theta \to \theta_{C}}(t)) = (1-t) \nabla_{\gamma_{\theta \to \theta_{C}}(t)} \mathcal{L}(\gamma_{\theta \to \theta_{C}}(t)).
\end{equation*}
Thus, one need only compute gradient of an objective with respect to $\gamma_{\theta \to \theta_{C}}(t)$ and multiply that by a factor of $1-t$ to retrieve the gradient of the objective with respect to $\theta$. This step has to be carried out explicitly and hence we have to break the optimization process of CBFT into two steps (see Eq.~\ref{eq:cbft}), executing alternating minimization for the barrier and invariance losses.

\section{Quadratic Connectivity Paths and Matching Permutations} 
\label{app:find_paths}
\subsection{Quadratic Paths}
\label{app:quad_paths}
The qudaratic path is defined as follows.
\begin{equation}
\gamma_{\theta_1 \to \theta_2}(t) = t^2 \theta_1 + 2 t (1-t) \theta_{12} + (1-t)^2 \theta_2.
\end{equation}
The set of parameters $\theta_{12}$ can be thought of as the vertex of a parabola that helps anchor the curve. 
To identify this set of parameters, we follow \citet{garipov2018loss} and train points uniformly sampled from the quadratic path to achieve zero loss on a given dataset $\mathcal{D}$, i.e.,
\begin{equation}
\theta_{12} = \argmin_{\theta} \mathbb{E}_{x \in \mathcal{D}, t \in [0,1]}(\mathcal{L}(f(x; \gamma_{\theta_1 \to \theta_2)(t)})).
\end{equation}
Consequently, note that a quadratic path necessarily depends on the dataset used for its identification and it is not mandatory that it generalize across datasets/distributions. This is precisely what we see in our results in Fig.~\ref{fig:smc}, where we are able to identify quadratic mode connectivty between two sets of parameters on a given dataset, but those paths do not generalize to counterfactual datasets.

\subsection{Finding Permutations for Linear Connectivity} 
\label{app:permutations}
Given two minimizers $\theta_1, \theta_2$, identifying the linear path between them involves merely interpolating the parameters. \citet{entezari2021role, ainsworth2022, singh2020model} hypothesize that minimizers discovered using SGD can always be linearly mode connected up to permutations of neurons that align the two models in their activations or weights. That is, there generally exists a permutation $\pi$ that connects $\pi(\theta_1)$ with $\theta_2$ in the sense of Def.~\ref{def:mode connectivity}. To empirically analyze this claim in our work, we identify $\pi$ by maximizing the similarity of activations produced by model with parameters $\theta_1$ and $\theta_2$:
\begin{equation}
\pi^{*} = \underset{\pi}{\argmin} ||f(x; \pi(\theta_1)) - f(x; \theta_2)||. \label{eq:minpi}
\end{equation}

Given that solving the problem above is NP-Hard~\citep{entezari2021role, ainsworth2022, singh2020model}, we follow the ``activation matching'' algorithm proposed by \citet{ainsworth2022} and  solve the above problem greedily by computing representations at each layer of the two models, finding a permutation that matches the representations maximally, and then repeating the process for the next layer. To this end, we use inputs with a batch-size of 512 and run the matching process over the entire original datasets (i.e., ones without cues). We note that we did conduct minimal experiments on finding permutations using data with cues, instead of without, but never found any noticeable differences in the results. Hence, we decided to use the original data without cues throughout our experiments for finding linear paths. Intuitively, we suspect the exact choice of dataset does not matter for our experimental setup because we analyze pairs of models which include one model that is invariant to the cue and one that is not. Since the invariant models produce the same representations on data with / without cues, the target for permutation matching remains the same. 

\subsection{Why plot accuracy curves instead of loss ones for mechanistic evaluation of mode connectivity}
Due to its differentiability, we focus on loss as our measure of interest for all formal analysis. However, since loss can increase without bound, visualizing loss curves become difficult for our setup that involves evaluation on counterfactual datasets, wherein the discriminative attributes are entirely removed (see Fig.~\ref{fig:dgp}). We thus follow \citet{frankle2020linear} and use accuracy curves for conveying experimental results in the main paper, since accuracy remains bounded within the range 0--100\% and can hence be visualized on a singular plot. We stress however we do provide loss curves as well in this appendix; see App.~\ref{app:smc_results},~\ref{app:lmc_results}.

\section{A Note on Difference Between Permutation and other Architectural Symmetries in the context of mode connectivity}
\label{app:othersyms}
    Note the notion of invariances discussed in this paper is rooted in the data-generating process, i.e., \textit{we discuss symmetry transformations of the data that are learned by the model during the optimization process}. However, similar to permutation symmetry, neural network architectures are known to exhibit several other architectural symmetries (i.e., symmetries that are not learned, but enforced by design of the architecture)~\citep{kunin2021neural}. Such architectural symmetries induce several minimizers that will necessarily be mechanistically similar. For example, resale symmetry, which involves scaling the weights of a given layer by a positive constant and another layer's by the inverse of that constant. This operation yields a different set of parameters that produce the same predictions, hence leading to mechanistically similar minimizers. Such architectural symmetries have an intriguing interplay with gradient-based optimizers (e.g., SGD)~\citep{kunin2021neural,wan2020spherical, roburin2022spherical} analogous to Noether's theorem~\citep{tanaka2021noether}, leading to implicit regularization behavior that yields minimizers with specific properties (e.g., rescale symmetry leads to minimizers with balanced layer-wise norms in the presence of weight decay~\citep{du2018algorithmic, kunin2021neural}). Correspondingly, even though infinite minimizers can be created by, e.g., rescaling layers of a model, only a minuscule fraction of these minimizers are actually reachable via gradient-based optimization. As we note in the preliminaries, we focus on minimizers retrieved using SGD. Thus, such equivalent classes of minimizers induced by other architectural symmetries are not a focus of this paper, as they are not identifiable via standard training pipelines and have to be synthetically induced by use of the corresponding architectural symmetry's operator. This is in contrast with permutation symmetry of neural networks, which does induce equivalent minimizers that are all reachable via the same training pipeline. For example, consider a model trained using some gradient-based optimizer. Permuting the neurons of such a model at initialization and running the same training pipeline will yield a different solution that relates to the original one via the exact same permutation of neurons. Since we randomly initialize models, both the original and the permuted initializations are equally probable, and hence both minimizers are equally likely to be identified using the same training pipeline.

\input{tables/table_ablations.tex}

\input{tables/tab_scratch.tex}

\section{Ablation Experiments on CBFT}
\label{app:ablations}

To analyze the role played by the two loss functions involved in the alternating minimization steps of Connectivity-Based Fine-Tuning (CBFT) (see \S\ref{sec:mechanistic fine tuning}, Eq.~\ref{eq:cbft}), we present an ablation study as follows. We analyze two variants of CBFT: (i) \textbf{$-\mathcal{L}_{\text{barrier}}$}, for which the barrier inducing loss $\mathbb{E}_{t \sim \mathcal{N}_{\text{Tr}}}|\lambda_1 - \mathcal{L}_{\text{CE}}(f(\mathcal{D}_{\text{C}}; \gamma_{\theta \to \theta_{\text{C}}}(t)), y)|$ has been removed from the training process, and (ii) \textbf{$-\mathcal{L}_{\text{Inv.}}$}, for which the invariance loss $\left(\sum_{k=1}^{K} \norm{\mathbb{E}_{x \in \mathcal{D}^{k}_{\text{C}}} (f_r(x;\theta)) - \mathbb{E}_{\Tilde{x} \in \mathcal{D}^{k}_{\text{NC}}}(f_r(\Tilde{x};\theta))}^{2}\right)$ has been removed. Results are shown in Tab.~\ref{tab:ablations}. We find that without the barrier loss, the trained model is unable to break its reliance on spurious cues, even though it generally achieves the best performance on data without cues (NC in table). Meanwhile, without the invariance loss, the trained model indeed loses sensitivity to spurious cues and shows poor performance when the underlying image is randomized, as we desire. However, in few instances the model can become anti-correlated with the spurious cue (e.g., see results on Dominoes). This is expected since the barrier loss's goal is to move the model to a region in the landscape that follows different mechanisms (with respect to the pre-trained model) by inducing a loss barrier; without the invariance loss, the model can learn to induce this barrier by merely becoming anti-correlated with the spurious cue. The invariance loss helps prevent this pitfall, selecting a mechanistically dissimilar region in the landscape that is uncorrelated, instead of being anti-correlated with the spurious cue. Overall, these results provide further corroboration to our claims in \S\ref{sec:mechanistic fine tuning}: \textit{preventing linear connectivity helps induce mechanistic dissimilarity and an invariance penalty helps select the exact mechanisms we want the models to differ in.} Overall, this ablation study help us infer that while the two losses involved in CBFT have their individual benefits, it is only when they are combined that they give the best results.

\subsection{Comparison with Training from Scratch}

We compare CBFT against training from scratch on the minimal clean dataset that we assume access to during the training process for all baselines and CBFT in Tab.~\ref{tab:CBFT}. Specifically, we train ResNet-18 models for 100 epochs using an initial learning rate of 0.1 and a cosine decay schedule. Results are reported in Tab.~\ref{tab:scratch} and we see training from scratch significantly underperforms all baselines and CBFT. This is expected since our setup assumes access to only a \textit{minimal} clean dataset for inducing invariance to spurious attributes. Since training from scratch is not a sample efficient strategy, it cannot perform well in this setting. We also highlight that using as initialization a model pretrained on an unclean dataset, i.e., one that contains spurious attributes, will make this overall process equal to na\"{i}ve fine-tuning on the clean dataset; we already provide results for na\"{i}ve fine-tuning in Tab.~\ref{tab:CBFT}.

\section{Deferred Proofs}
\label{app:proofs}

\subsection{Exhaustiveness of Unit Interventions}
\textbf{Proposition \ref{lem1}.}
\emph{\textbf{(Exhaustiveness of Unit Interventions.)} If $f(.; \theta)$ is invariant to unit interventions $\mathcal{A}_{i}$ and $\mathcal{A}_{j}$, it must be invariant to their composition; conversely, lack of invariance to either $\mathcal{A}_{i}$ or $\mathcal{A}_{j}$ precludes invariance to their composition. }
\vspace{-5pt}
\begin{proof}
    Assume the set of parameters $\theta$ induces a model that exhibits invariance to the intervention $\mathcal{A}_i$. Independently, consider another intervention $\mathcal{A}_j$. Then, $f(\mathcal{E}(x; \{\mathcal{A}_i, \mathcal{A}_j\}); \theta) = f(\mathcal{G}_X \circ \mathcal{A}_i \circ \mathcal{A}_j \circ \mathcal{G}_X^{-1}(x); \theta) = f(\mathcal{G}_X \circ \mathcal{A}_i \circ \mathcal{G}_X^{-1}(\mathcal{E}(x; \mathcal{A}_j)); \theta) = f(\mathcal{E}(\mathcal{E}(x; \mathcal{A}_j); \mathcal{A}_i); \theta) = f(\mathcal{E}(x; \mathcal{A}_j); \theta)$, where the last equality happens due to the assumed invariance of $\mathcal{A}_i$. Now, if $\theta$ exhibits invariance to $\mathcal{A}_j$ as well, we have $f(\mathcal{E}(x; \{\mathcal{A}_i, \mathcal{A}_j\}); \theta) = f(\mathcal{E}(x; \mathcal{A}_j); \theta) = f(x; \theta)$, i.e., the model induced by $\theta$ is invariant to the composition of $\mathcal{A}_i$ and $\mathcal{A}_j$. Meanwhile, if $\theta$ is invariant $\mathcal{A}_i$ but not to $\mathcal{A}_j$, we have $f(\mathcal{E}(x; \{\mathcal{A}_i, \mathcal{A}_j\}); \theta) = f(\mathcal{E}(x; \mathcal{A}_j); \theta) \neq f(x; \theta)$, i.e., $\theta$ induces a model that lack invariance to the simultaneous operation (i.e., composition) of $\mathcal{A}_i$ and $\mathcal{A}_j$. 
    
    Note that the derivation above did not rely on the fact that the interventions are ``unit'' in the sense that they act on independent dimensions. However, if one considers general interventions that can act on multiple dimensions of the latent space simultaneously, then a given intervention can undo the effects of another. For example, assume a model is not invariant to unit interventions on a dimension that rotates an object, but are invariant to unit interventions on all other latent dimensions. Then, if two general interventions involve operation on this latent dimension, they can make an object rotate by equal and opposite angles, while changing some other dimensions of the latent state that the model is invariant to. In this case, the interventions end up undoing each other's effect, and the overall state change does not yield any influence on the model output. By assuming unit interventions that enforce transformations on specific dimensions, we circumvent this failure mode.
\end{proof}

\subsection{Mode Connectivity of Mechanistically Dissimilar Models}

We first repeat the following result from prior work (paraphrased per our notations and setup).

\begin{lemma}
\label{lem:berfin} 
\textbf{\citep{simsek2021geometry}.}
Consider an $L$-layer network $f(.; \theta)$, whose activation function $\phi$ satisfies $\phi(0) \neq 0$, $\phi^{(n)} \neq 0$ for infinitely many odd and even values of $n$, where $\phi^{(n)}$ denotes the $n^{\text{th}}$ derivative of $\phi$. Let $r_1^{*}, r_2^{*}, \dots, r_L^{*}$ be the minimum number of neurons needed in layers $1$ to $L$ for achieving zero error (cross-entropy or mean-square error) on a dataset $\mathcal{D}$ and call a network overparameterized if for all layers $l$, it contains number of neurons $r_l > r_l^{*}$. Then, under overparameterization, there always exists a continuous, zero-loss path that connects two minimizers.
\vspace{-5pt}
\end{lemma}

The result above involves showing permutation symmetry of neural networks yields a single continuous manifold of zero loss, and then proving all parameters that yield zero-loss lie on this manifold. We highlight the amount of overparameterization needed for the claim's validity is rather mild, i.e., just one additional neuron per layer. Also note that while the proof makes assumptions on the analyticity of the activation function used, this constraint is only mandatory for ease of theoretical analysis. Moreover, continuous approximations to ReLU exist which satisfy these assumptions. For example, $\phi(x) = \phi_{\text{softplus}}(x) + \phi_{sigmoid}(4x)$, where $\phi_{\text{softplus}}(x) = \ln (1 + \exp (x))$ and $\phi_{\text{sigmoid}}(x) = \nicefrac{1}{1 + \exp(-x)}$. Similar result was also shown by \citet{nguyen2019connected}, who demonstrates networks with a pyramidal structure, i.e., networks for which the width of any given layer is less than or equal to its preceding layers.

Our claim on mode connectivity of mechanistically dissimilar models now follows as a corollary.

\textbf{Proposition \ref{claim:all_mimima_connect}.}
\emph{\textbf{(Mode Connectivity of Mechanistically Dissimilar Models.)} Assume $\theta_1, \theta_2$ are minimizers of the loss on a dataset $\mathcal{D}$ and induce mechanistically dissimilar models. Given sufficient overparameterization, there exists a continuous path along which the minimizers are mode connected.}
\vspace{-5pt}
\begin{proof}
    By definition, $\mathcal{L}(f(\mathcal{D}; \theta)) = 0$ for $\theta \in {\theta_1, \theta_2}$. Since the distribution of data plays no role in the proof of Lemma~\ref{lem:berfin}, the result must hold for two minimizers that rely on entirely disparate mechanisms (e.g., background vs.\ shape) for achieving zero-loss on a dataset $\mathcal{D}$. The claim then directly follows as a corollary of Lemma~\ref{lem:berfin}, assuming the model is overparameterized in the sense defined there and the loss is either cross-entropy or mean-square error.
\end{proof}

\subsection{Lack of Linear Connectivity and Mechanistic Dissimilarity}
\textbf{Conjecture \ref{claim:lmc}.}
\emph{
\textbf{(Lack of Linear Connectivity implies Mechanistic Dissimilarity.)}
If two minimizers $\theta_1$ and $\theta_2$ of the loss $\mathcal{L}(f(\mathcal{D}; \theta))$ on a dataset $\mathcal{D}$ cannot be linear mode-connected (up to permutations of neurons), their corresponding models $f(.; \theta_1), f(.; \theta_2)$ must be mechanistically dissimilar.
}

As we show next, the conjecture above can be proven in a simplified setting.

\textbf{Model Setup:} We consider a binary classification task on a dataset $\mathcal{D} = \{x_{i}, y_{i}\}_{i=1}^{M}$, where $x_i \in \mathbb{R}^{D}$, $y \in \mathcal{Y} = \{0, 1\}$, and $M$ is the number of samples. The model is a 1-hidden layer, fully connected network defined as follows: $f(x; W) = \frac{1}{N}\mathbf{1}^{T} \phi(W^{T} x)$. Here, $W \in \mathbb{R}^{D \times N}$ denotes the hidden layer with $N$ neurons, $\mathbf{1} \in \mathbb{R}^{N}$ is an all ones vector, and $\phi(.)$ is the ReLU activation function. The model is trained to minimize a loss $\mathcal{L}\left(f(\mathcal{D}; W)\right) = \frac{1}{M} \sum_{i=1}^{M} l\left(y_i, f(x_i; W)\right)$, where $l(.,.)$ denotes a sample-wise loss function whose global minimizer yields $y_i = f(x_i; W)$ for all $x_i \in \mathcal{D}$. This property is satisfied by several loss functions, e.g., mean-square error, L-1 loss, etc. We assume the models are overparamterized such that all minimizers are global and interpolating, i.e., they achieve zero loss~\citep{kawaguchi2016deep, kawaguchi2020elimination, nguyen2018loss, nguyen2020global, arora2019fine}. This implies if $W_{*}$ is a minimizer, $\forall i \in [M], y_i = f(x_i; W_{*}) = \frac{1}{N} \mathbf{1}^{T} \phi(W_{*}^{T} x_i)$. 

We next describe the data-generating process that we will focus on in the following discussion.

\textbf{Data-Generating Process:} We consider a data-generating process with multiple predictive attributes of different complexity, inspired by the one proposed by \citet{shah2020pitfalls}.

Consider a non-negative even integer $K$. Define the sets $S_{0}(K)$ and $S_{1}(K)$ that respectively include odd and even integers between $[-\frac{K}{2}, \frac{K}{2}]$. We use $\text{sign}(.)$ to denote the sign function, which outputs $1$ if $x > 0$, $0$ if $x = 0$, and $-1$ if $x < 0$. $\text{Unif}(S)$ denotes a uniform distribution over the set $S$. We define a randomized process $s_{K}$, such that $s_{K}(0) \sim \text{Unif}(S_{0}(K))$, $s_{K}(1) \sim \text{Unif}(S_{1}(K))$. Correspondingly, given a margin hyperparameter $\delta \in [0, 0.5]$, we define the randomized function $T_{K}(z): \{0, 1\} \to \mathbb{R}$ as follows.
\begin{equation}
\label{eq:syndgp}
\begin{split}
    T_K(z) :=
    \begin{cases}
        \frac{\sqrt{3}}{\sqrt{D}}\left(z - \epsilon \, \text{sign}(z)\right), &\text{where } \epsilon \sim \text{Unif}([0, 2\,\delta]), \text{ if } K = 0, \\
        \frac{2\sqrt{3}}{K\sqrt{D}} \left(s_{K}(z) + \epsilon \right), &\text{where } \epsilon \sim \text{Unif}([-\delta, \delta]), \text{ if } K \geq 1, |s_{K}(z)| \neq \frac{K}{2}, \\
        \frac{2\sqrt{3}}{K\sqrt{D}} \left(s_{K}(z) - \epsilon \, \text{sign}(z)\right), &\text{where } \epsilon \sim \text{Unif}([0, \delta]), \text{ if } K \geq 1, |s_{K}(z)| =  \frac{K}{2}.
    \end{cases}
\end{split}
\end{equation}

Note that $T_{K}(z)$ produces a zero-mean output with variance $\nicefrac{1}{D}$. The margin $\delta$ allows us to draw infinite samples from the function. More importantly, $T_{K}(z)$ is invertible, i.e., given its output, we can infer $z$. Correspondingly, if $z$ defines the target label $y$, inverting the attribute $T_{K}(z)$ will allow us to solve a classification task defined on this attribute. 
However, this inversion process requires inference of $K$ piece-wise linear splines to model the optimal decision boundaries (see Fig.~\ref{fig:slabs}). The scalar $K$ can thus can be considered a measure of the complexity of the attribute, inline with prior work on simplicity bias in neural networks~\citep{kalimeris2019sgd, shah2020pitfalls, valle2018deep, scimeca2021shortcut, hu2020surprising}. For example, if $K=0$, the attribute is linearly separable and of least complexity. This notion of complexity is particularly natural for studying neural networks with ReLU activations because each neuron in such a model represents a spline function and several such neurons can approximate complex decision boundaries by representing them with such piece-wise spline functions~\citep{balestriero2018spline, balestriero2018mad,  balestriero2017neural, wang2018max}.

\setlength{\intextsep}{2pt}%
\setlength{\columnsep}{8pt}%
\begin{wrapfigure}{r}{0.42\textwidth}
    \begin{subfigure}[b]{0.40\textwidth}
      \centering
      \centerline{\includegraphics[width=\columnwidth]{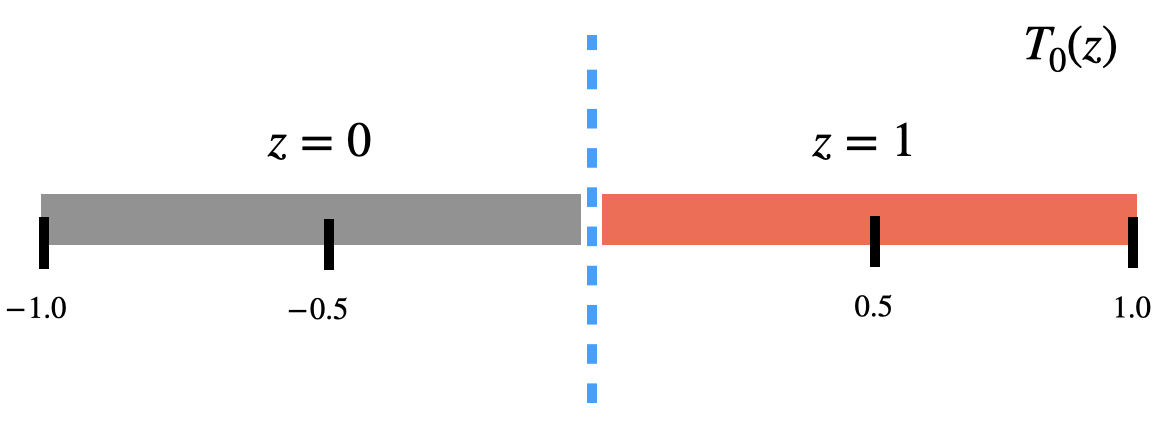}}
    \end{subfigure}
    \begin{subfigure}[b]{0.45\textwidth}
      \centering
      \centerline{\includegraphics[width=\columnwidth]{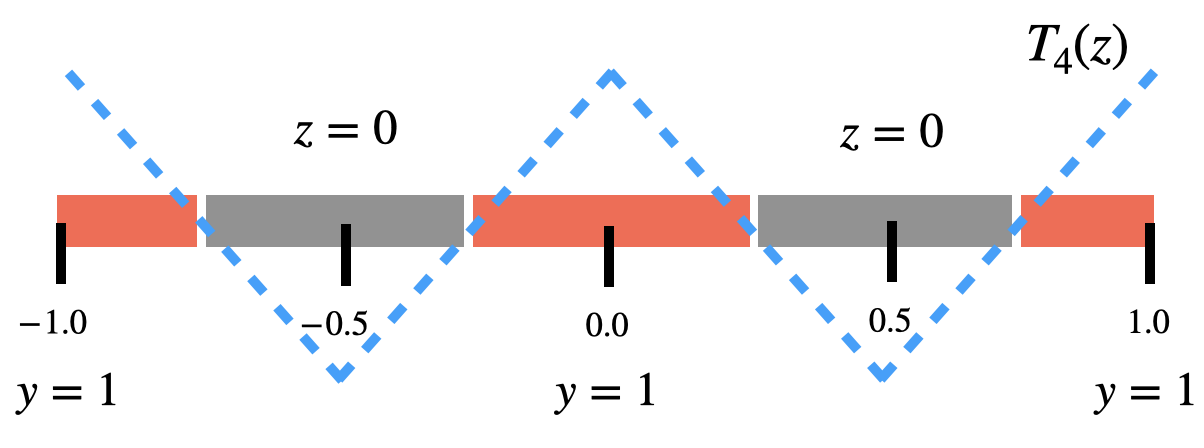}}
    \end{subfigure}
    \vspace{-1em}
  \caption{\label{fig:slabs}\textbf{K-Complex Outputs.} We illustrate the output range of randomized function $T_{K}(z)$ (Eq.~\ref{eq:syndgp}) for $K=0$ (top) and $K=4$ (bottom). Note the input $z \in \{0, 1\}$ (shown grey, red respectively) deterministically controls the output values. Thus, if $y = z$, the output $T_{K}(z)$ is perfectly predictive of the target label. However, inverting this function requires inference of $K$ piece-wise linear splines to model the optimal decision boundaries (shown blue, dotted lines). Increasing the value of $K$ makes the task consequently harder.}
  \vspace{-0em}
\end{wrapfigure}

The overall data-generating process $\mathcal{G}(Z)$ transforms the $n$-dimensional random variable $Z$ from the latent space $z_1 \times z_2 \times \dots \times z_n \in \{0, 1\}^{n}$ to produce samples with $n$ attributes of different complexities $T_{K_1}(z_1), T_{K_2}(z_2), \dots, T_{K_n}(z_n)$ and appends $D-n$ noisy attributes, sampled from a symmetric, zero-mean distribution $\mathcal{V}$ with variance $\frac{1}{d}$; e.g., the Gaussian distribution $\mathcal{N}\left(0, \nicefrac{1}{\sqrt{D}}\right)$ or uniform distribution $\mathcal{U}\left(-\sqrt{\nicefrac{3}{D}}, \sqrt{\nicefrac{3}{D}}\right)$. Correspondingly, generation of a sample $(x, y)$ can be represented as follows:
\begin{equation}
\footnotesize
\label{eq:dgpsample}
    (x, y) := \mathcal{G}(Z)
      =  \left[
           T_{K_1}(z_1), 
           T_{K_2}(z_2), 
           \dots,
           T_{K_n}(z_n),
           \nu_1,
           \nu_2,
           \dots,
           \nu_{d-n} \right]^{T}, 
\end{equation}
where $\nu_i \sim \mathcal{V}$ for all $i \in \{1, 2, \dots, D-n\}$. In the above, the target label $y$ is assumed to be generated by the function $\mathcal{G}_{y}(.) = T_{K_i}^{-1}(.)$, which inverts the attribute $T_{K_i}(z_i)$ that is assumed to define the label. Note again that another attribute, say $T_{K_j}(z_j)$, will also be predictive of the label if the corresponding latent, $z_j$, is correlated with $z_i$. This is similar to putting a correlated \textit{cue} attribute in the data, as was done in our experiments in the main paper. Thus, while the process above is relatively simplified, it is a valid abstraction of our empirical setup from \S\ref{sec:dgpsetup}, Fig.~\ref{fig:dgp}. In the following, we will consider perfectly predictive attributes, i.e., we do not assume partial correlation.

We next define the notion of an \textit{activation pattern} and \textit{matching of two models in their activation patterns}. In the following, $\Sigma_{N}$ denotes the permutation group of order $N$ over the set $\{1, 2, \dots, N\}$~\citep{bronstein2021geometric}. 

\begin{definition}
\textbf{(Activation Pattern.)}
The activation pattern of a model $f(.; W)$ on input $x$ is defined as the vector $\phi'(W^{T}x)$ whose elements are indicator variables denoting whether the $j^{\text{th}}$ hidden neuron is activated for input $x$, i.e., $\phi'(W^{T}x)[j] = 1\left(W_{j}^{T}x > 0\right)$.
\vspace{-5pt}
\end{definition}
Note that $\phi(W^{T}x) = \phi'(W^{T}x) \odot (W^{T} x)$, where $\odot$ denotes the element-wise product, if $\phi(.)$ is the ReLU function.

\begin{definition}
\textbf{(Matching in Activation Patterns.)}
Consider a dataset $\mathcal{D}$ and two models $f(.; W_1)$ and $f(.; W_2)$. We call the models \emph{matching in activation patterns} on dataset $\mathcal{D}$ if there exists a permutation $\pi \in \Sigma_{N}$ that rearranges neurons of $f(.; W_2)$ to match $f(.; W_1)$'s activation patterns, i.e., $\phi'(W_1^T x) = \phi'(\pi(W_2)^T x)$ for all $x \in \mathcal{D}$.
\end{definition}

Next, we establish the relationship between two models' activation patterns and linear mode connectivity.
\begin{lemma} 
\label{lem:align}
\textbf{(Alignment Constraint for Linear Mode Connectivity.)}
    If minimizers $W_{\alpha}$ and $W_{\beta}$ exhibit linear mode connectivity on dataset $\mathcal{D}$, then the models $f(.; W_{\alpha}), f(.; W_{\beta})$ are matching in activation patterns on the dataset. That is, for all $x \in \mathcal{D}$, we have $\phi'(W_{\alpha}^{T}x) = \phi'(W_{\beta}^{T}x)$.
    \vspace{-10pt}
\end{lemma}
\begin{proof}
    Note that linear mode connectivity is a translation invariance property of the loss in the parameter space. Since we assume interpolating minimizers, this invariance extends to model predictions. Consequently, the derivative of the model prediction along the linear path $\gamma_{W_{\alpha} \to W_{\beta}}(t) = W(t) = W_{\alpha} + t (W_{\beta} - W_{\alpha})$ is zero; that is, $\frac{\partial}{\partial t} f(x; W(t)) = 0$. This implies, 
    \begin{equation}
    \label{eq:paramconserve}
    \frac{\partial}{\partial t}\mathbf{1}^{T} \phi(W(t)^{T} x) = \phi'(W(t)^{T}x)^{T} (W_{\beta} - W_{\alpha})^{T}x = 0.        
    \end{equation}
    Substituting $t=1$ in Eq.~\ref{eq:paramconserve}, we get, 
    \begin{equation}
    \begin{split}
    \phi'(W_{\beta}^{T}x)^{T} (W_{\alpha}^{T}x) &= \phi'(W_{\beta}^{T}x)^{T} (W_{\beta}^{T}x) = \mathbf{1}^{T}(\phi'(W^{T}_{\beta}x) \odot \phi(W_{\beta}^{T}x)) = \mathbf{1}^{T}\phi(W_{\alpha}^{T}x).
    \end{split}
    \end{equation}
    This implies,
    \begin{equation}
    \label{eq:constrainta}
    \begin{split}
        (\phi'(W_{\alpha}^{T}x) - \phi'(W_{\beta}^{T}x))^{T} (W_{\alpha}^{T}x) = 0.
    \end{split}
    \end{equation}
    Next, we define the following vector.
    \begin{equation}
    \begin{split}
    \mathbf{1}_{(\alpha_{+}\, \beta_{-})} = 
        \begin{cases} 
        1, & \,\,\text{if}\,\,\ W_{\alpha}^{T}x > 0 \text{ and } W_{\beta}^{T}x \leq 0; \\
        0, & \text{ otherwise}. \\
        \end{cases}\\
    \end{split}
    \end{equation}
    Define the vector $\mathbf{1}_{(\alpha_{-}\, \beta_{+})}$ in a similar manner. Then, it is easy to see that
    \begin{equation}
    \label{eq:indicators}
    \phi'(W_{\alpha}^{T}x) - \phi'(W_{\beta}^{T}x) = \mathbf{1}_{(\alpha_{+}\, \beta_{-})} - \mathbf{1}_{(\alpha_{-}\, \beta_{+})}.
    \end{equation}
    Substituting the above relationship in Eq.~\ref{eq:constrainta} gives
    \begin{equation}
    \begin{split}
        \mathbf{1}_{(\alpha_{+} \, \beta_{-})}^{T}(W_{\alpha}^{T} x) = \mathbf{1}_{(\alpha_{-}\, \beta_{+})}^{T}(W_{\alpha}^{T} x). \\
    \end{split}
    \end{equation}
    Note that in the above equation, the left-hand side is a sum of positive reals, while the right-hand side is a sum of negative reals. That is, the equality cannot hold unless both are equal to zero for all $x \in \mathcal{D}$. This implies $\mathbf{1}_{(\alpha_{+}\, \beta_{-})} = \mathbf{1}_{(\alpha_{-}\, \beta_{+})} = \mathbf{0}$. That is, there is no neuron in model $f(.; W_{\alpha})$ that is active while the corresponding index neuron in $f(.; W_{\beta})$ is inactive. Consequently, for linear mode connectivity to hold, the neurons at the same index in the two models should activate/inactivate together for any given sample, hence producing the same set of activation patterns. This completes the proof.
\end{proof}

\begin{corollary} 
\label{rmk:wdist}
\textbf{Small Wasserstein-1 distance between activation patterns of two models implies they can be linear mode connected.}
\vspace{-5pt}
\end{corollary}
The activation pattern of a model for a given sample is a vector of binary variables. Thus, the difference between two activation patterns $\phi'(W_{\alpha}^T x)$ and $\phi'(W_{\beta}^T x)$ can be computed by simply comparing their means $\frac{1}{N} |\mathbf{1}^{T}\phi'(W_{\alpha}^T x) - \mathbf{1}^{T}\phi'(W_{\beta}^T x)|$, which is in fact the Wasserstein-1 distance between two Bernoulli distributions for which $p=\frac{1}{N} |\mathbf{1}^{T}\phi'(W^T x)|$. This value $p$ can be regarded as the probability a neuron in the model is activated. Correspondingly, when the Wasserstein-1 distance between two activation patterns is low, we can expect that there exists a permutation of neurons that allows the two models to be linear mode connected. The W-1 distance can thus be regarded as a proxy for assessing whether two models can be linear mode connected. Further, we highlight that even though this result is derived for a specific model architecture, it is actually quite general: any two models with zero W-1 distance must be linear mode connectable (up to permutations) because their activation patterns will necessarily be the same.

\begin{remark} 
\label{rmk:permutations}
\textbf{(Lemma \ref{lem:align} highlights why neurons must be permuted for linear mode connectivity.)}
The lemma above shows that if two models produce the same activation patterns, the models are ``effectively linear'' with respect to each other. This enables linear interpolation of the two models without increasing error. We also highlight that if two models produce activation patterns that are a permutation of each other (e.g., this can happen if their initializations were permutations of each other), then un-permuting them will make the models linear mode connected. Thus, Lemma~\ref{lem:align} is inherently a neuronal alignment constraint and can be regarded as a precise condition under which the conjecture by \citet{entezari2021role} holds. Even though the result above was shown for a two-layer model, it is easy to see that a more general statement is true: \textit{if two minimizers induce models that produce the same activation patterns, then there exists a permutation of neurons under which the two models can be linear mode connected.}
\end{remark}


Next, we rephrase the result on simplicity bias of neural networks by \citet{shah2020pitfalls, valle2018deep, kalimeris2019sgd, scimeca2021shortcut} using the notations defined in this paper.
\begin{lemma}
\label{lem:simbias}
\textbf{(Simplicity Bias.)}
    Assume a data-generating process $\mathcal{G}$ produces $n$ perfectly predictive attributes with respective complexities $[K] = \{K_1, K_2, \dots, K_n\}$. Let $m$ be the index of the latent corresponding to the simplest attribute, i.e., $m := \argmin\, [K]$. If $W$ is a minimizer identified using gradient descent on a dataset that contains IID samples retrieved from $\mathcal{G}$, then the corresponding model $f(.; W)$ will be invariant to unit interventions on all but the latents of the simplest predictive attribute, i.e., $\mathcal{I}(W) = \{\mathcal{A}_{i}: i \neq m\}$.
    \vspace{-5pt}
\end{lemma}

Thus, even if a dataset contains multiple predictive attributes, minimizers identified using gradient descent induce models that only utilize the simplest attributes for making their predictions.

Now consider a setting where two models make their predictions using different simplest predictive attributes from a dataset containing multiple predictive attributes. Then, if two such models rely on attributes of different complexities, we can be certain they produce different activation patterns.
\begin{lemma}
\label{lem:diffacts}
\textbf{(Disparate Complexity of Mechanisms Disallows Matching in Activations).}
    Consider an IID sampled dataset $\mathcal{D}_{\alpha, \beta}$ from a data-generating process that produces predictive attributes $T_{K_{\alpha}}(.), T_{K_{\beta}}(.)$, where, without loss of generality, $K_{\alpha} > K_{\beta}$. Let $W_{\alpha}$ denote a minimizer of the loss $\mathcal{L}(f(\mathcal{D}_{\alpha, \beta}; W))$ and assume its induced model relies on $T_{K_{\alpha}}(.)$ for making its predictions; similarly define $W_{\beta}$. Then, there exists no permutation $\pi \in \Sigma_{N}$ such that $f(.; W_{\alpha})$ and $f(.; \pi(W_{\beta}))$ are matching in activation patterns on $\mathcal{D}_{\alpha, \beta}$.
    \vspace{-5pt}
\end{lemma}
\begin{proof}
    The claim follows via contradiction. Assume a permutation $\pi$ exists such that the two models are matching in activation patterns on $\mathcal{D}_{\alpha, \beta}$. Denote the weights of the $i^{\text{th}}$ neuron in $W_{\alpha}$ via $W_{\alpha}^{i}$. Then $W_{\alpha}^{i}, \pi(W_{\beta})^{i}$ are the weights of the neurons matched via $\pi$. Since using the attribute $T_{K_{\alpha}}(.)$ for predicting the label corresponds to inference of $2 K_{\alpha}$ piece-wise spline functions, the probability the $i^{\text{th}}$ neuron with weights $W_{\alpha}^{i}$ will be activated for an IID sampled input $x$ from the data-generating process is $\frac{1}{2K_{\alpha}}$. However, since $K_{\alpha} \neq K_{\beta}$, the neuron with weights $\pi(W_{\beta})^{i}$ does not activate with the same probability. That is, there exist samples for which $W_{\alpha}^{i}$ is activated, but $\pi(W_{\beta})^{i}$ is not. This contradicts our assumption that there exists a permutation that allows matching in activation patterns for the two models.
\end{proof}

\setlength{\tabcolsep}{4.5pt}
\begin{table}
    \caption{\label{tab:simbias}\textbf{Illustrating Simplicity Bias.} We train models on a dataset with predictive attributes of complexities 0 and 4. Column titles indicate which attributes were allowed to remain predictive during training, i.e., were not randomized via interventions: e.g., $K_{1}=0$ implies only the linearly separable attribute is predictive in the training data. Rows report difference in loss on a test dataset $\mathcal{D}_{K_1}$ which contains attributes of complexity $K_1$ and another test dataset $\mathcal{D}_{K_2}$ which contains attributes of complexity $K_2$. Results are computed up to 4 digits of precision and averaged over 3 seeds. We see models trained on data with both predictive attributes behave similarly to models trained on $K=0$ attribute only; that is, they are invariant to the more complex attribute for which $K=4$.}
\centering
\begin{tabular}{@{}c|cc|cc|cc@{}}
    \toprule
    Complexity of Train Attribute   & \multicolumn{2}{c|}{$K_{1} = 0$} & \multicolumn{2}{c|}{$K_{1} = 4$} & \multicolumn{2}{c}{$K_{1} = 0, 4$} \\ \midrule
    Complexity of Test Attribute    & \multicolumn{1}{c|}{$K_{2} = 0$} & $K_{2} = 4$     & \multicolumn{1}{c|}{$K_{2} = 0$} & $K_{2} = 4$   & \multicolumn{1}{c|}{$K_{2} = 0$} & $K_{2} = 4$     \\ \midrule
    $|\mathcal{L}(f(\mathcal{D}_{K_1}; W)) - \mathcal{L}(f(\mathcal{D}_{K_2}; W))|$ & \multicolumn{1}{c|}{0.0} & 22.79 & \multicolumn{1}{c|}{26.31} & 0.0 & \multicolumn{1}{c|}{0.0} & 18.84 \\ \bottomrule
    \end{tabular}
\vspace{-8pt}
\end{table}

Combining the results above, we have the following theorem.
\begin{theorem}
\label{thm:preclude}
\textbf{(Disparity in Simplest Attributes Precludes Matching).}
    Consider a dataset $\mathcal{D}$ that contains multiple predictive attributes. Assume two minimizers of the loss $\mathcal{L}(f(\mathcal{D}; W))$ induce mechanistically dissimilar models that identify attributes of different complexity to make their predictions. Then, their exists no permutations of neurons for which the models exhibit linear mode connectivity.
    \vspace{-5pt}
\end{theorem}
\begin{proof}
    The result follows directly from the application of Lemmas~\ref{lem:align},~\ref{lem:simbias},~\ref{lem:diffacts}. Specifically, Lemma~\ref{lem:align} shows matching in activation patterns is required for two models to exhibit linear mode connectivity (up to permutations). Lemma~\ref{lem:simbias} shows one need only analyze mechanistic dissimilarity with respect to the simplest attributes to compare the activation patterns between two models. Lemma~\ref{lem:diffacts} shows if two models use attributes of different complexity to make their predictions, they cannot match in activation patterns. 
\end{proof}

Let us now revisit Conjecture~\ref{claim:lmc} for our simplified setup. In Theorem~\ref{thm:preclude}, we have shown models with dissimilar mechanisms of different complexity must also produce different activation patterns. Correspondingly, via Lemma~\ref{lem:align}, we have these models cannot be linear mode connected. This verifies our claim for the simplified setup for a 1-hidden layer model if the learned mechanisms are of different complexity. However, it remains possible that two mechanistically dissimilar models learn mechanisms to identify attributes that are \textit{different}, but have the \textit{same complexity}, producing similar activation patterns and hence exhibiting linear connectivity. This possibility, though viable in theory, is practically not feasible. For example, if a dataset $\mathcal{D}_1$ contains multiple attributes of similar complexity that allow minimizers retrieved from $\mathcal{D}_2$ to perform well on it, then given SGD (and related algorithms) force neural networks to converge to max-margin solutions, we see that minimizers retrieved via training on $\mathcal{D}_1$ will have already learned mechanisms to identify \textit{all attributes of same complexity}~\citep{soudry2018implicit, lyu2019gradient, gunasekar2018implicit, nacson2019convergence}. In that case, use of $\mathcal{D}_2$ to create a minimizer that learns a different mechanism is practically moot, since we will already learn the relevant mechanism from $\mathcal{D}_1$ itself. Combined with Theorem 1, this argument completes the result.

\subsection{Empirical Verification}

\begin{figure}
\centering
\begin{subfigure}[b]{0.49\textwidth}
  \centering
  \centerline{\includegraphics[width=\columnwidth]{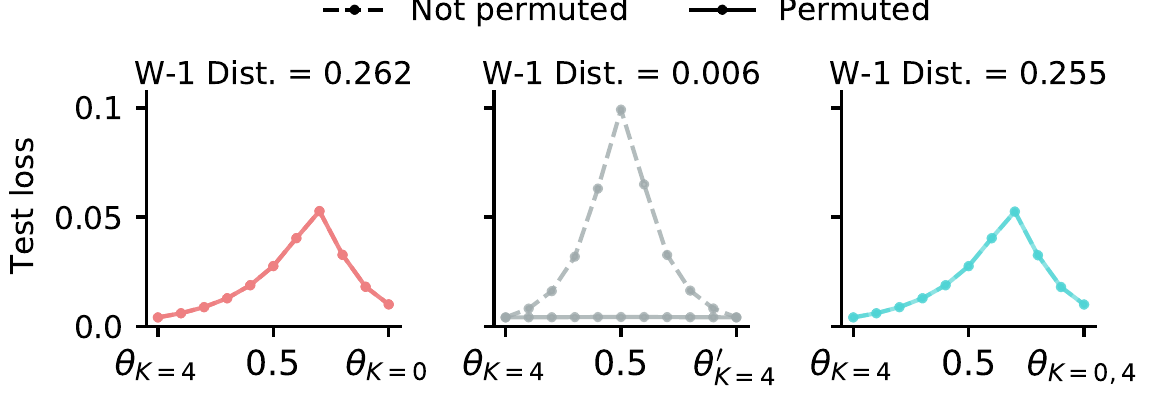}}
  \caption{}
\end{subfigure}%
\begin{subfigure}[b]{0.49\textwidth}
  \centering
  \centerline{\includegraphics[width=\columnwidth]{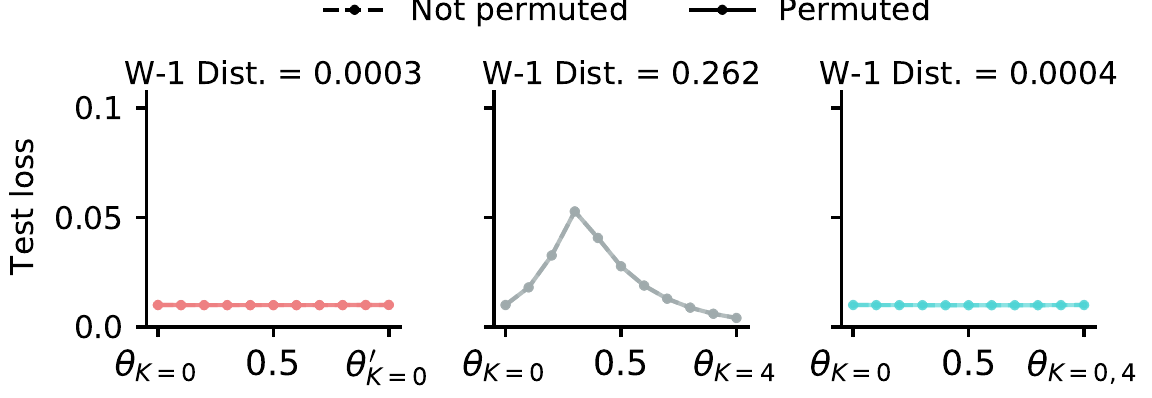}}
  \caption{}
\end{subfigure}
\vspace{-1mm}
\caption{\label{fig:LMCverify}\textbf{Lack of Linear Connectivity implies Mechanistic Dissimilarity.} Plot titles denote Wasserstein-1 distance between the two models whose linear mode connectivity is being assessed. We see that models which have learned mechanisms to identify attributes of different complexity have a large Wasserstein-1 distance between their activation patterns; consequently, they cannot be linear mode connected, even after permutation of neurons. Meanwhile, models reliant on the same mechanisms have a small Wasserstein distance and can indeed be linear mode connected. For example, $\theta_{K=0}$ and $\theta_{K=0,4}$ learn the same mechanisms due to simplicity bias and can be linearly connected (see Tab.~\ref{tab:simbias}), but they do not exhibit linear connectivity with $\theta_{K=4}$; meanwhile, $\theta_{K=4}$ and $\theta'_{K=4}$ can be linearly connected. Note however the latter case of more complex, i.e., $K=4$ attribute required permutations to match the neurons for linear connectivity, while the former case of linearly separable attribute did not. This behavior emerges due to the fact that all neurons learn to be \textit{always active} for the $K=0$ predictive attribute--see \citet{soudry2018implicit, shah2020pitfalls} for proof.} 
\end{figure}

\textbf{Illustrating Simplicity Bias:} See Tab.~\ref{tab:simbias}. Specifically, we train models using SGD for our assumed $f(.;W)$ architecture, using 512 neurons in the hidden layer. We sample 50000, 128-dimensional inputs from the data-generating process discussed in Eq.~\ref{eq:dgpsample} with $K=0$ and  $K=4$ complex predictive attributes present in the dataset. We analyze three training scenarios: (i) when only $K=0$ attribute is allowed to be predictive and the $K=4$ attribute is randomized via interventions; (ii) when only $K=4$ attribute is allowed to be predictive and the $K=0$ attribute is randomized via interventions; and (iii) when both attributes are allowed to be predictive. Evaluation involves assessing invariance of the trained model to the two predictive attributes by computing loss on a test dataset that contains both predictive attributes and a counterfactual variant of the dataset for which either (a) $K=0$ or (b) $K=4$ complexity attributes have been randomized via interventions. If loss remains the same, the model is invariant to interventions on the attribute of that complexity, thus implying the model has not learned a mechanism to identify that attribute. Results are shown in Tab.~\ref{tab:simbias}. We see intervening on the $K=0$ attribute yields an increase in loss in scenarios (i), (iii), indicating those models have learned a mechanism to identify that attribute. Meanwhile, intervening on the $K=4$ attribute yields an increase in loss only in scenario (ii), indicating the models trained from the other two scenarios are invariant to the $K=4$ complex attribute. While this is expected for scenario (i), the fact that this behavior emerges for the scenario (iii), where both predictive attributes can be used for training, is a consequence of simplicity bias of SGD.

\textbf{Lack of Linear Connectivity implies Mechanistic Dissimilarity.} See Fig.~\ref{fig:LMCverify}. We train models on 50000 samples drawn from the data-generating process discussed in Eq.~\ref{eq:syndgp}, allowing two predictive attributes of complexity $K=0$ and $K=4$. Models are also trained on counterfactual variants of this dataset, where one of the predictive attributes has been randomized via interventions. For example, $\theta_{K=0, 4}$ denotes a minimizer trained on data with both attributes, while $\theta_{K=0}$ denotes a minimizer identifed via training on the counterfactual dataset that contains only the $K=0$ predictive attribute; note $\theta'_{0}$ denotes use of a different initialization seed. Subsequently, we assess linear mode connectivity (before and after permutation) of models by using an evaluation dataset of 10000 samples similar to the base training dataset, i.e., both $K=0$ and $K=4$ predictive attributes are allowed. We see that models which have learned mechanisms to identify attributes of different complexity have a large Wasserstein-1 distance between their activation patterns; consequently, they cannot be linear mode connected, even after permutation of neurons. Meanwhile, models reliant on the same mechanisms have a small Wasserstein distance and can indeed be linear mode connected. For example, $\theta_{K=0}$ and $\theta_{K=0,4}$ learn the same mechanisms due to simplicity bias and can be linearly connected (see Tab.~\ref{tab:simbias}), but they do not exhibit linear connectivity with $\theta_{K=4}$; meanwhile, $\theta_{K=4}$ and $\theta'_{K=4}$ can be linearly connected. Note however the latter case of more complex, i.e., $K=4$ attribute required permutations to match the neurons for linear connectivity, while the former case of linearly separable attribute did not. This behavior emerges due to the fact that all neurons learn to be \textit{always active} for the $K=0$ predictive attribute--see \citet{soudry2018implicit, shah2020pitfalls} for proof.

\section{Further Results: Non-Linear Connectivity of Mechanistically Dissimilar Minimizers}
\label{app:smc_results}
We train VGG-13 and ResNet-18 models on our synthetic CIFAR-10 / CIFAR-100 / Dominoes datasets with cues (see Figs.~\ref{fig:c10viz}, \ref{fig:c100viz}, and \ref{fig:dominoesviz}) and the original datasets themselves. Parameters of the corresponding models are denoted $\theta_{\text{C}}$ and $\theta_{\text{NC}}$. We identify connectivity paths along pairs of parameters, specifically evaluating quadratic paths identified using data without cues (denoted Quadratic w/o Cues), 
quadratic path identified using data with cue (denoted Quadratic w/ Cue), linear path (denoted Linear), and linear path after permuting $\theta_{\text{C}}$ to maximally match $\theta_{\text{NC}}$'s activations (denoted Linear Permuted). In the following, plot titles denote evaluation dataset, including datasets where either the cue is present (denoted w/ Cue), absent (denoted w/o Cue), randomized (denoted Rand.\ Cue), or the underlying image is randomized but the cue remains the same (denoted Rand.\ Image). Line colors denote the proportion of dataset that has synthetic cues. 

Our results show the minimzier $\theta_{\text{NC}}$ yields the same performance upon randomization of the cue, while the performance of $\theta_{\text{C}}$ degrades substantially--i.e., the two modes are mechanistically dissimilar due to lack of shared invariances (see Def.~\ref{def:mechsim}). Nonetheless, we can identify quadratic (but not linear) paths that mode-connect these mechanistically dissimilar minimizers, hence corroborating Prop.~\ref{claim:all_mimima_connect} across several datasets and model architectures, showing \textit{mechanistically dissimilar modes can also be mode connected via relatively simple paths as well}. However, different points on the connectivity paths respond differently to counterfactuals, indicating \textit{lack of mechanistic connectivity}.

\begin{figure}[H]
\centering
\begin{subfigure}[b]{\textwidth}
  \centering
  \centerline{\includegraphics[width=0.9\columnwidth]{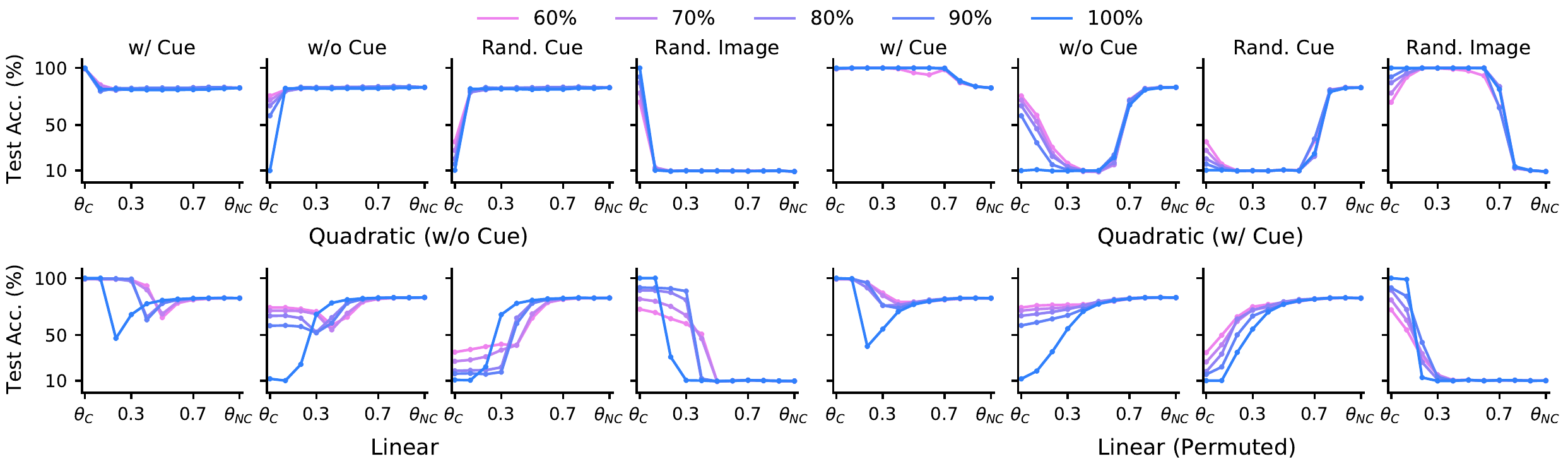}}
  \caption{Test Accuracy.}
\end{subfigure}
\begin{subfigure}[b]{\textwidth}
  \centering
  \centerline{\includegraphics[width=0.9\columnwidth]{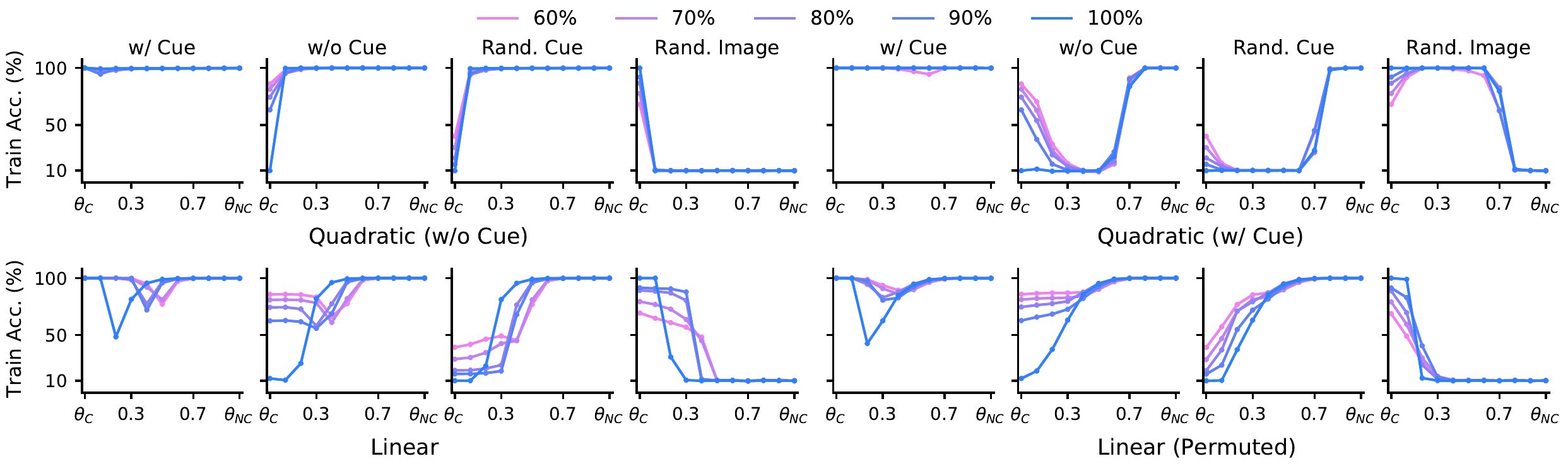}}
  \caption{Train Accuracy.}
\end{subfigure}
\vspace{-1mm}
\begin{subfigure}[b]{\textwidth}
  \centering
  \centerline{\includegraphics[width=0.8\columnwidth]{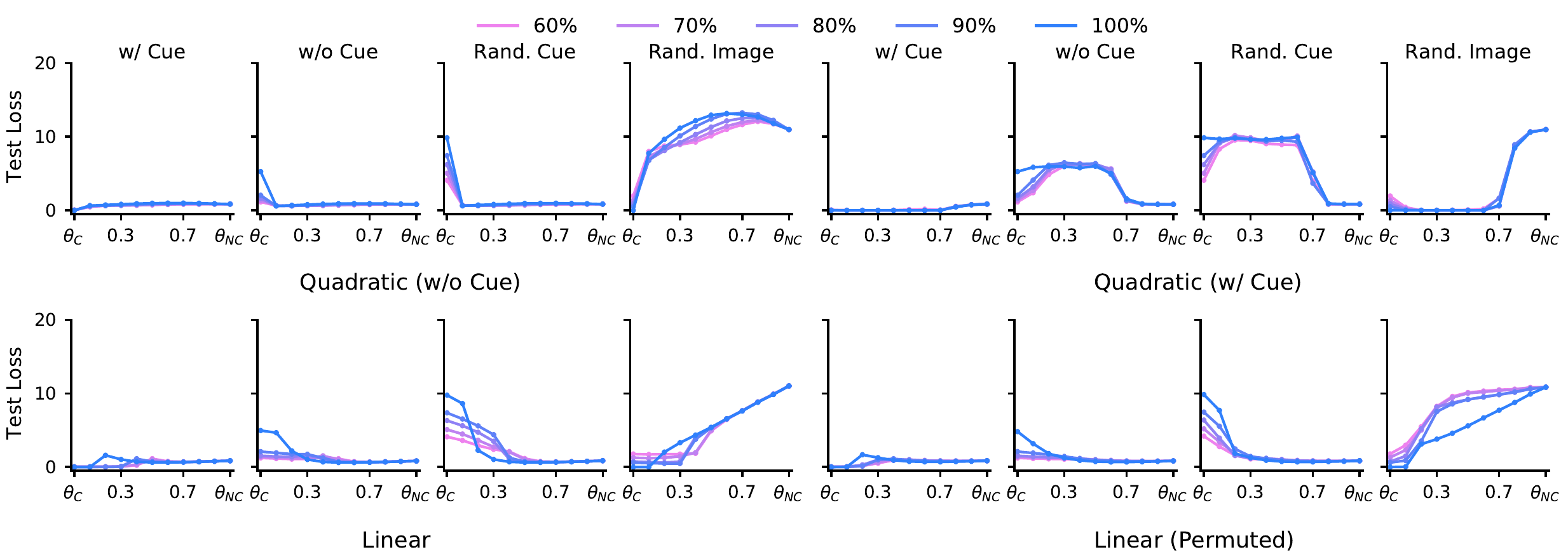}}
  \caption{Test Loss.}
\end{subfigure}
\begin{subfigure}[b]{\textwidth}
  \centering
  \centerline{\includegraphics[width=0.8\columnwidth]{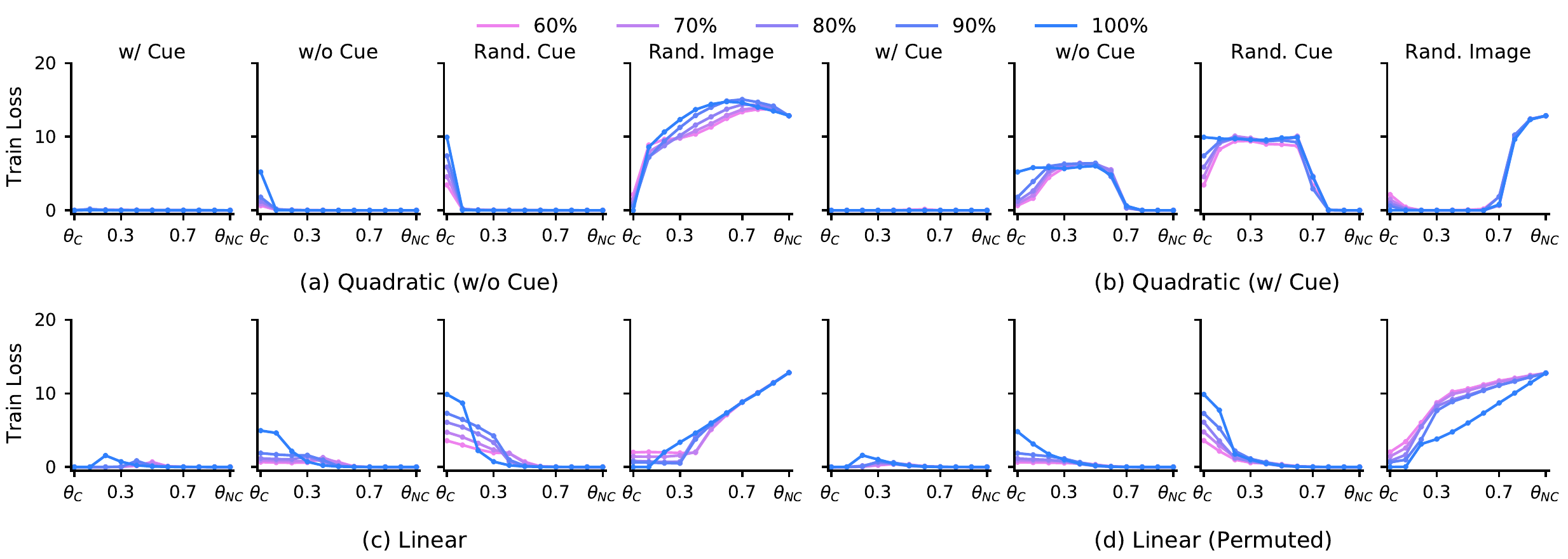}}
  \caption{Train Loss.}
\end{subfigure}
\vspace{-1mm}
\caption{\label{fig:c10_vgg}\textbf{VGG-13 on CIFAR-10 with Box Cue}. We plot test/train accuracy/loss curves along different connectivity paths and see thorough corroboration of our claims in the main text: Mechanistically dissimilar minimizers can be connected via nonlinear paths on a given dataset, but behave different on counterfactuals, indicating lack of mechanistic connectivity.}
\vspace{-10pt}
\end{figure}

\begin{figure}[H]
\centering
\begin{subfigure}[b]{\textwidth}
  \centering
  \centerline{\includegraphics[width=0.9\columnwidth]{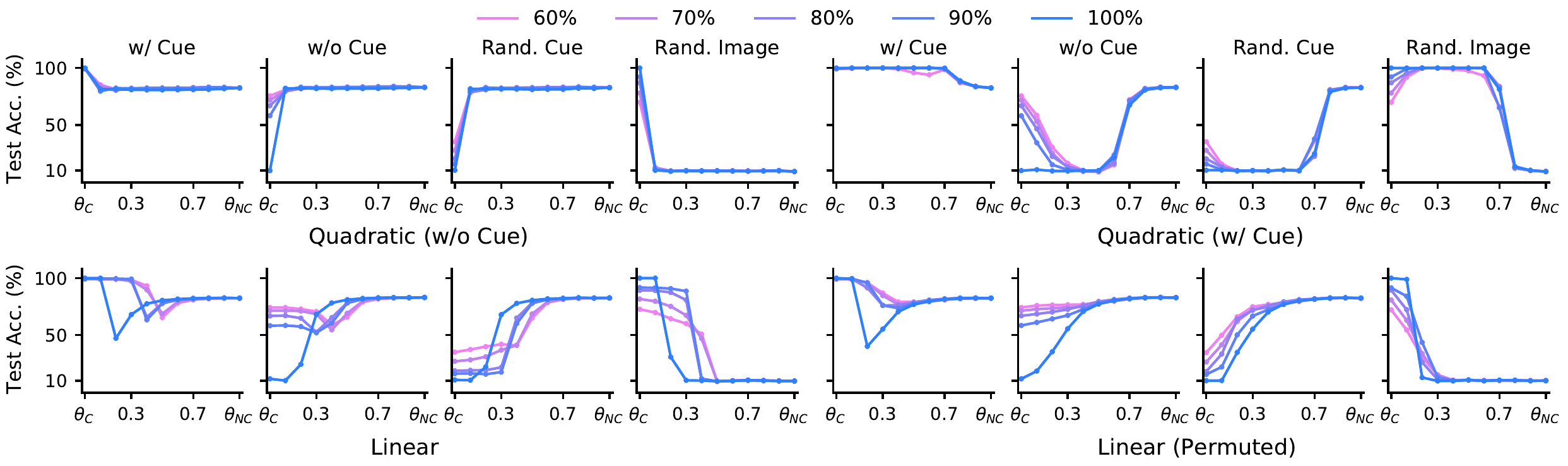}}
  \caption{Test Accuracy.}
\end{subfigure}
\begin{subfigure}[b]{\textwidth}
  \centering
  \centerline{\includegraphics[width=0.9\columnwidth]{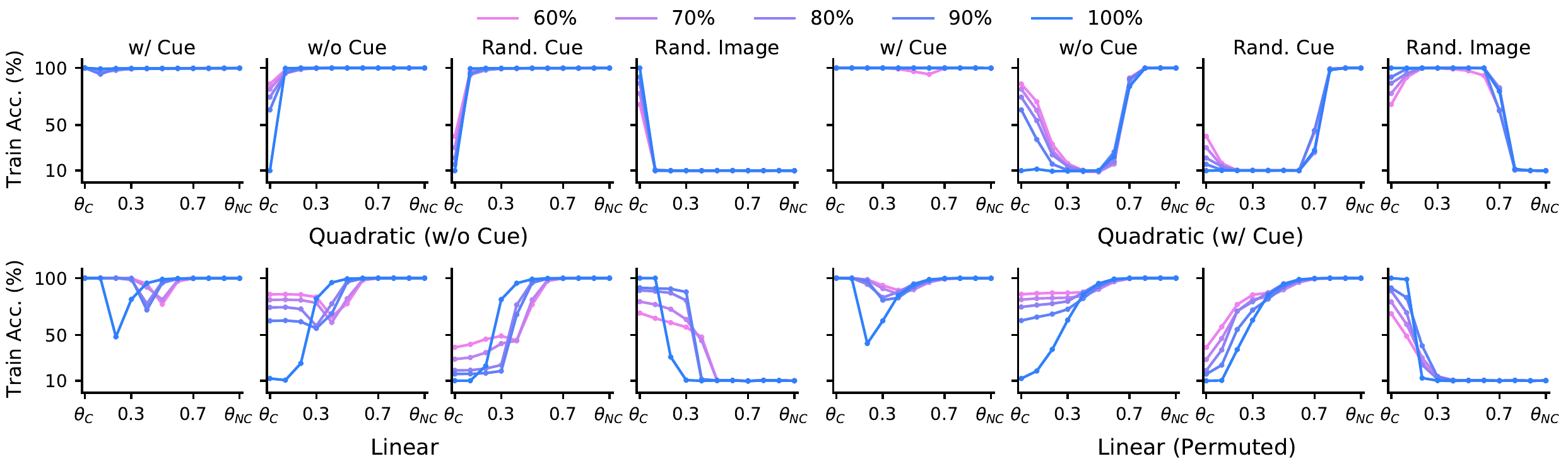}}
  \caption{Train Accuracy.}
\end{subfigure}
\vspace{-1mm}
\begin{subfigure}[b]{\textwidth}
  \centering
  \centerline{\includegraphics[width=0.8\columnwidth]{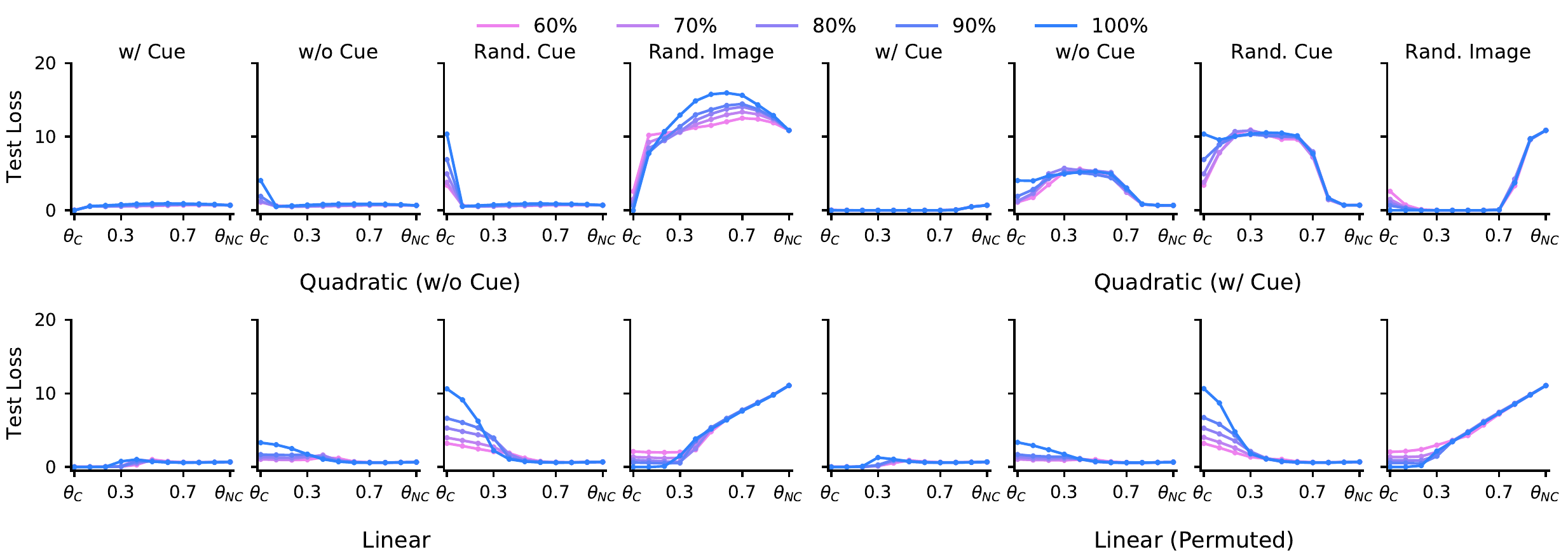}}
  \caption{Test Loss.}
\end{subfigure}
\begin{subfigure}[b]{\textwidth}
  \centering
  \centerline{\includegraphics[width=0.8\columnwidth]{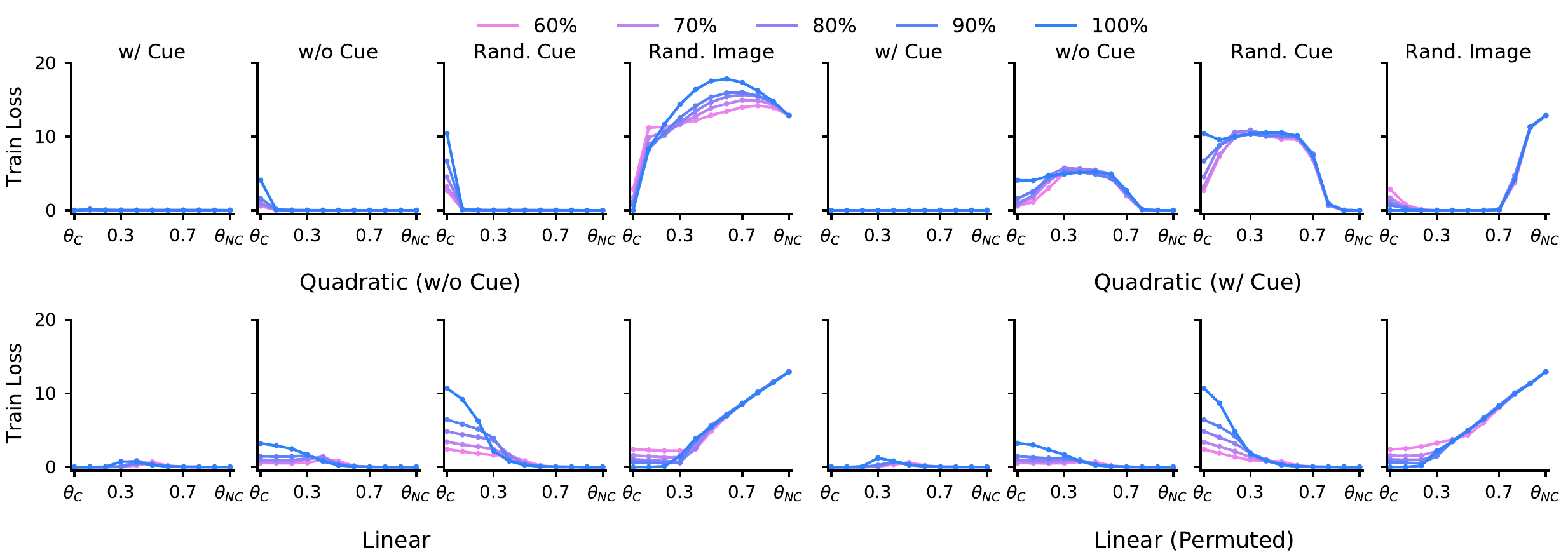}}
  \caption{Train Loss.}
\end{subfigure}
\vspace{-1mm}
\caption{\label{fig:c10_res18}\textbf{ResNet-18 on CIFAR-10 with Box Cue}. We plot test/train accuracy/loss curves along different connectivity paths and see thorough corroboration of our claims in the main text: Mechanistically dissimilar minimizers can be connected via nonlinear paths on a given dataset, but behave different on counterfactuals, indicating lack of mechanistic connectivity.}
\vspace{-10pt}
\end{figure}

\begin{figure}[H]
\centering
\begin{subfigure}[b]{\textwidth}
  \centering
  \centerline{\includegraphics[width=0.9\columnwidth]{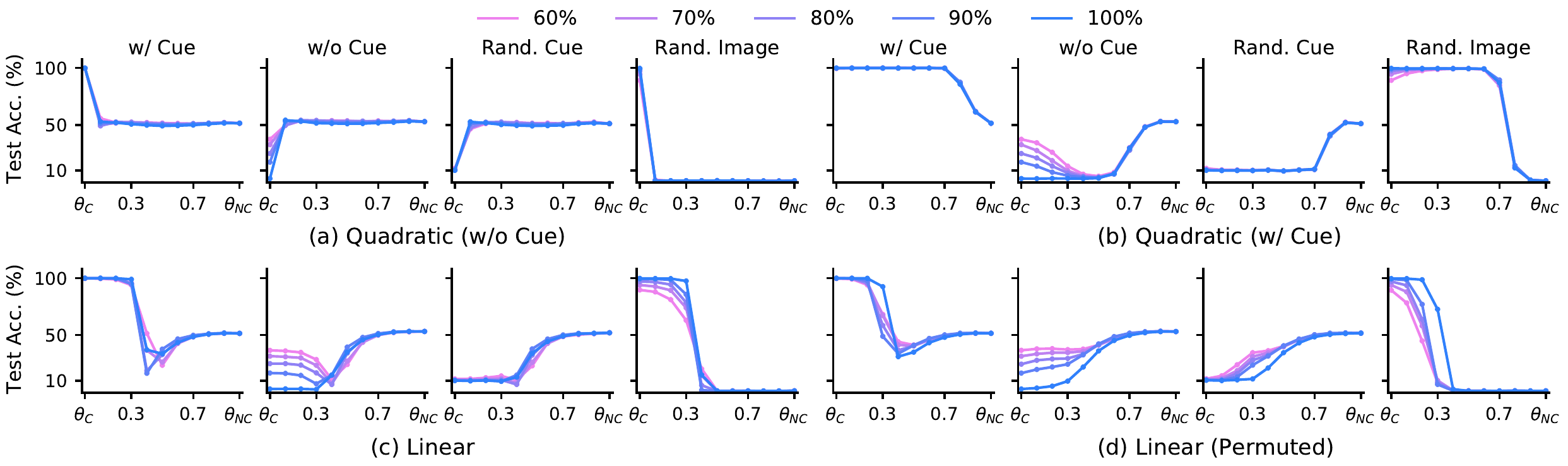}}
  \caption{Test Accuracy.}
\end{subfigure}
\begin{subfigure}[b]{\textwidth}
  \centering
  \centerline{\includegraphics[width=0.9\columnwidth]{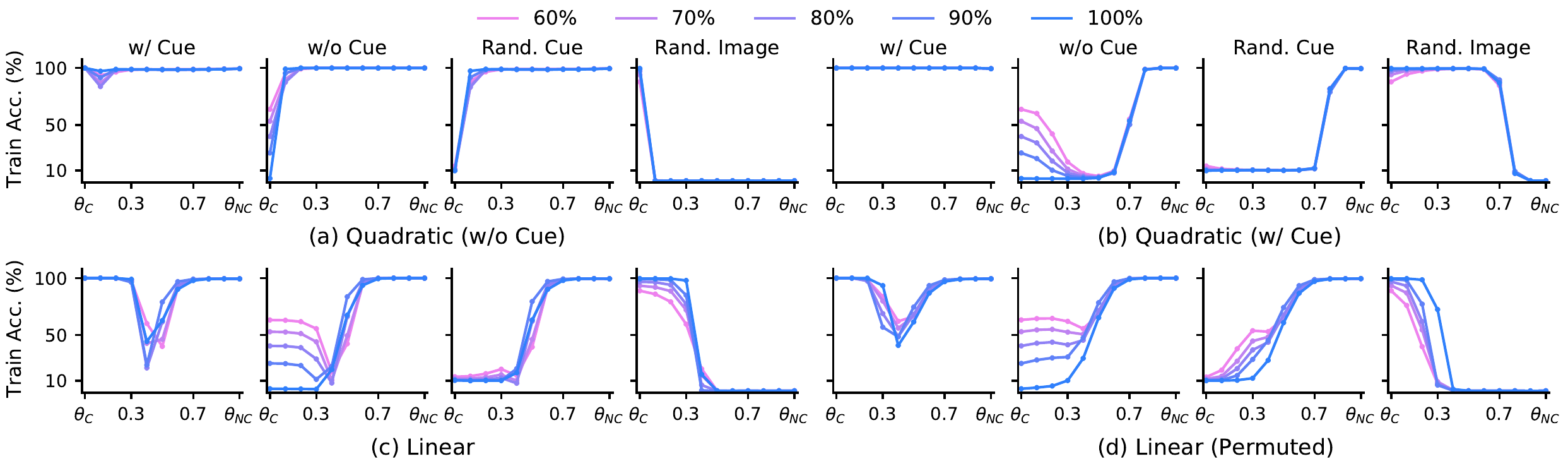}}
  \caption{Train Accuracy.}
\end{subfigure}
\vspace{-1mm}
\begin{subfigure}[b]{\textwidth}
  \centering
  \centerline{\includegraphics[width=0.9\columnwidth]{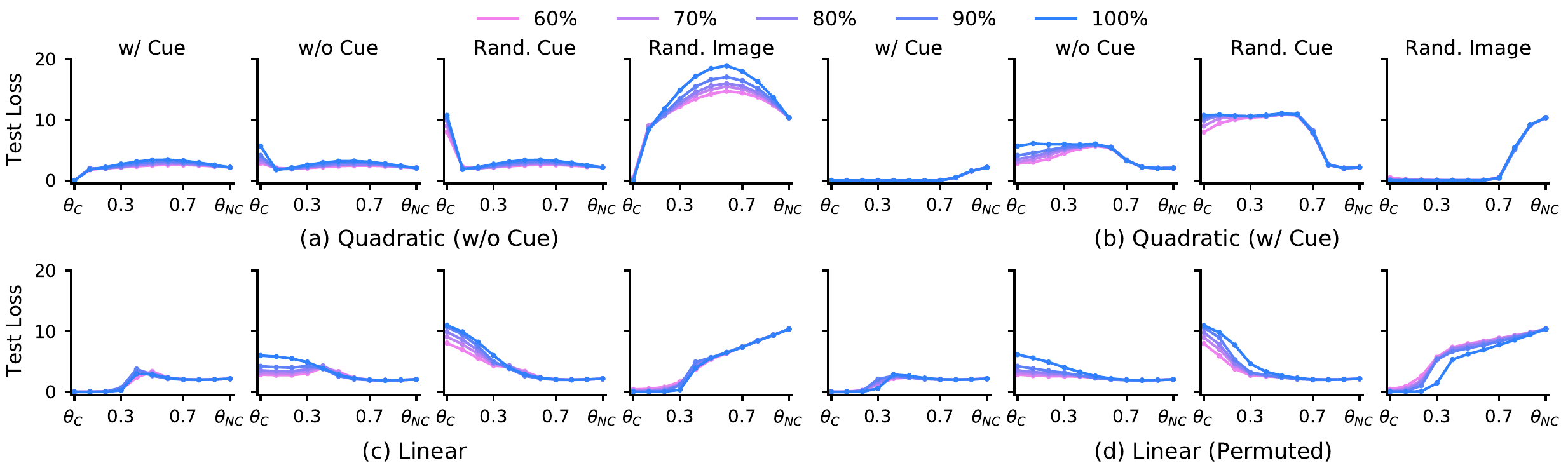}}
  \caption{Test Loss.}
\end{subfigure}
\begin{subfigure}[b]{\textwidth}
  \centering
  \centerline{\includegraphics[width=0.9\columnwidth]{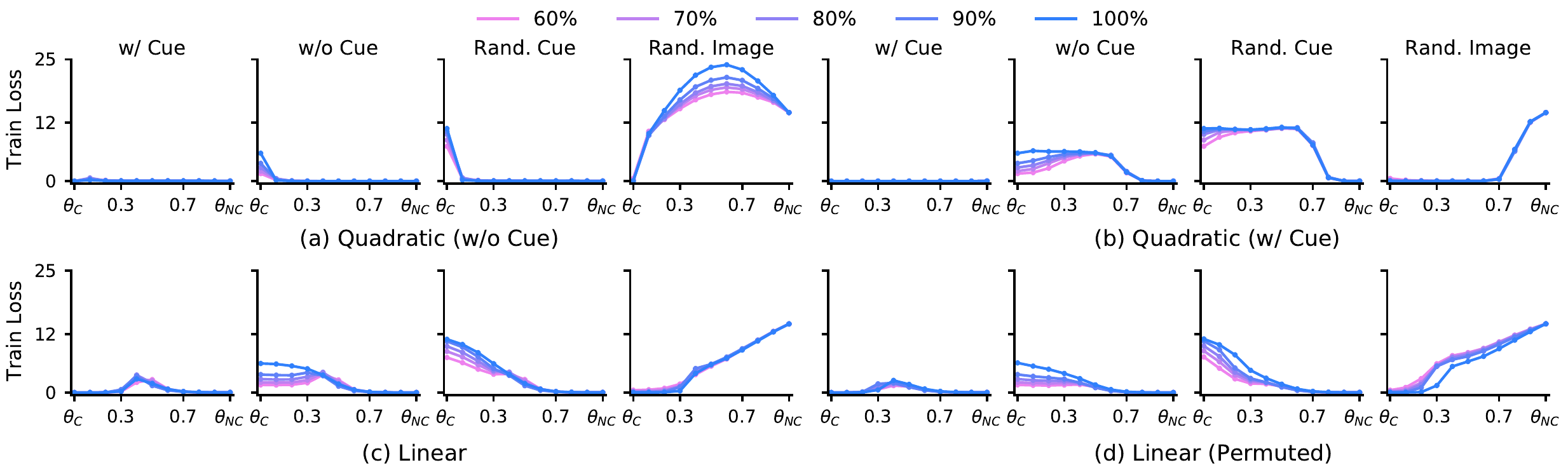}}
  \caption{Train Loss.}
\end{subfigure}
\vspace{-1mm}
\caption{\label{fig:c100_vgg}\textbf{VGG-13 on CIFAR-100 with Box/Color Cue}. We plot test/train accuracy/loss curves along different connectivity paths and see thorough corroboration of our claims in the main text: Mechanistically dissimilar minimizers can be connected via nonlinear paths on a given dataset, but behave different on counterfactuals, indicating lack of mechanistic connectivity.}
\vspace{-10pt}
\end{figure}

\begin{figure}[H]
\centering
\begin{subfigure}[b]{\textwidth}
  \centering
  \centerline{\includegraphics[width=0.9\columnwidth]{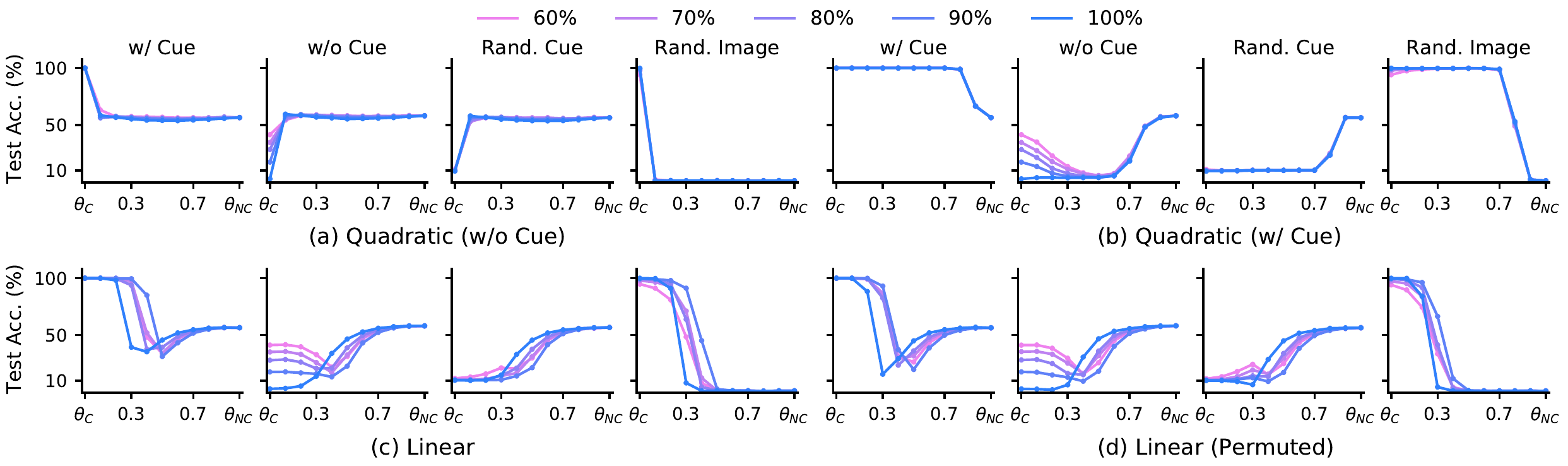}}
  \caption{Test Accuracy.}
\end{subfigure}
\begin{subfigure}[b]{\textwidth}
  \centering
  \centerline{\includegraphics[width=0.9\columnwidth]{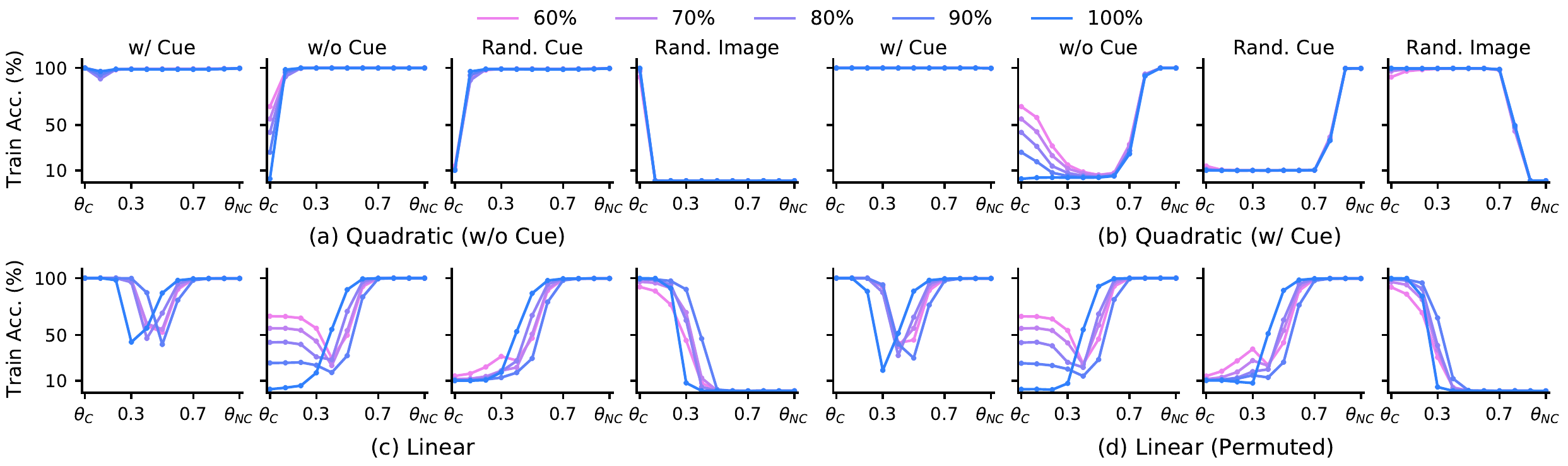}}
  \caption{Train Accuracy.}
\end{subfigure}
\vspace{-1mm}
\begin{subfigure}[b]{\textwidth}
  \centering
  \centerline{\includegraphics[width=0.9\columnwidth]{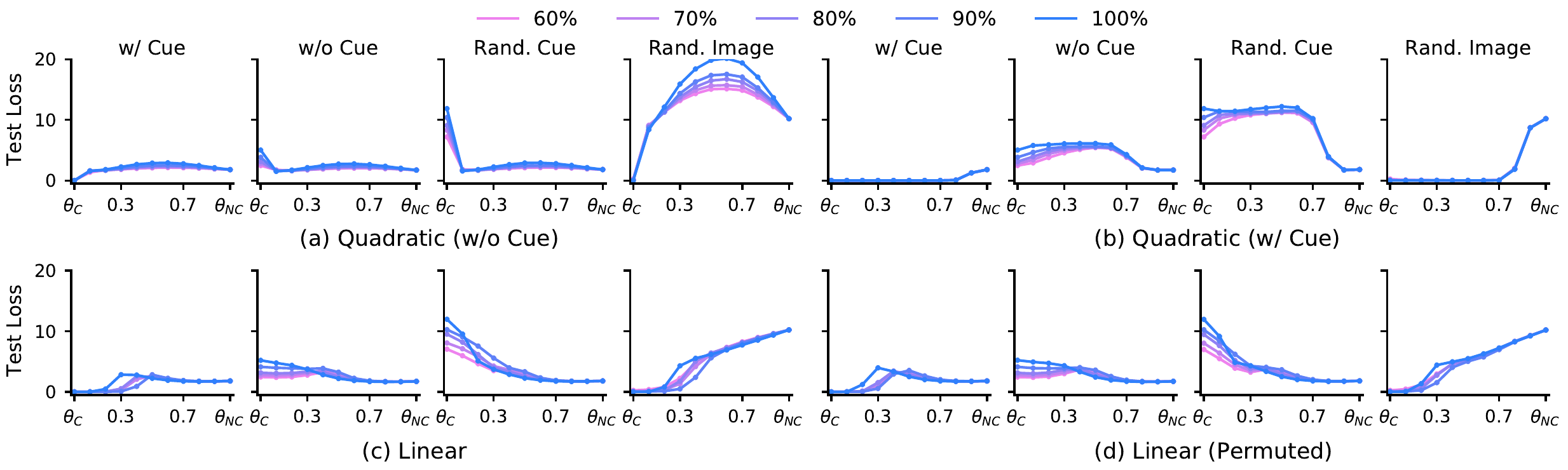}}
  \caption{Test Loss.}
\end{subfigure}
\begin{subfigure}[b]{\textwidth}
  \centering
  \centerline{\includegraphics[width=0.9\columnwidth]{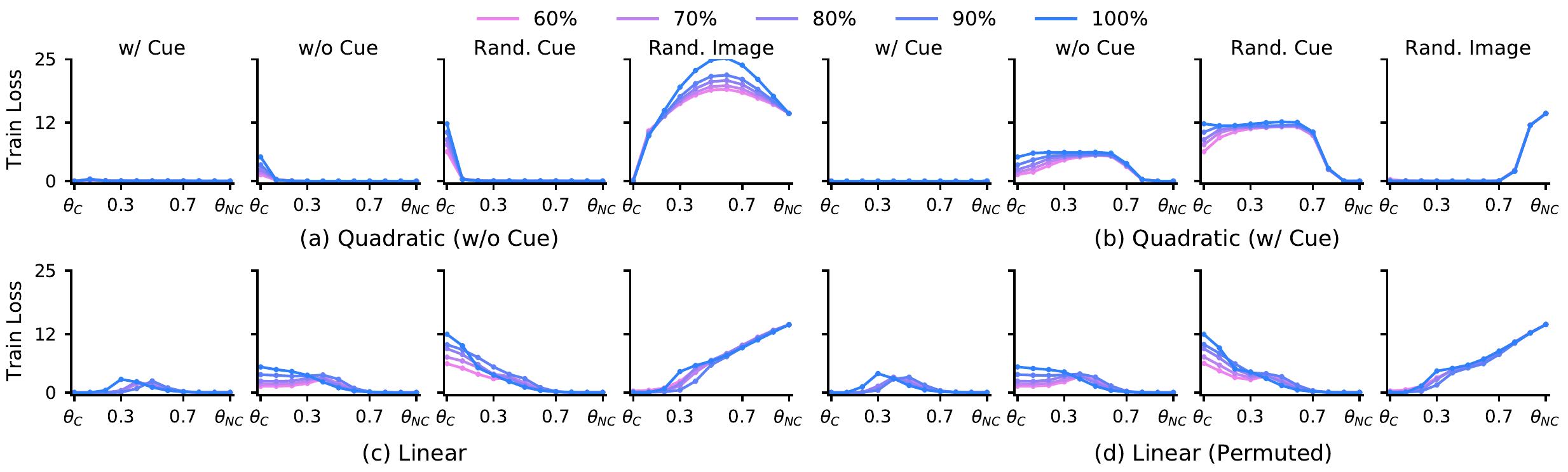}}
  \caption{Train Loss.}
\end{subfigure}
\vspace{-1mm}
\caption{\label{fig:c100_res18}\textbf{ResNet-18 on CIFAR-100 with Box/Color Cue}. We plot test/train accuracy/loss curves along different connectivity paths and see thorough corroboration of our claims in the main text: Mechanistically dissimilar minimizers can be connected via nonlinear paths on a given dataset, but behave different on counterfactuals, indicating lack of mechanistic connectivity.}
\vspace{-10pt}
\end{figure}

\begin{figure}[H]
\centering
\begin{subfigure}[b]{\textwidth}
  \centering
  \centerline{\includegraphics[width=0.9\columnwidth]{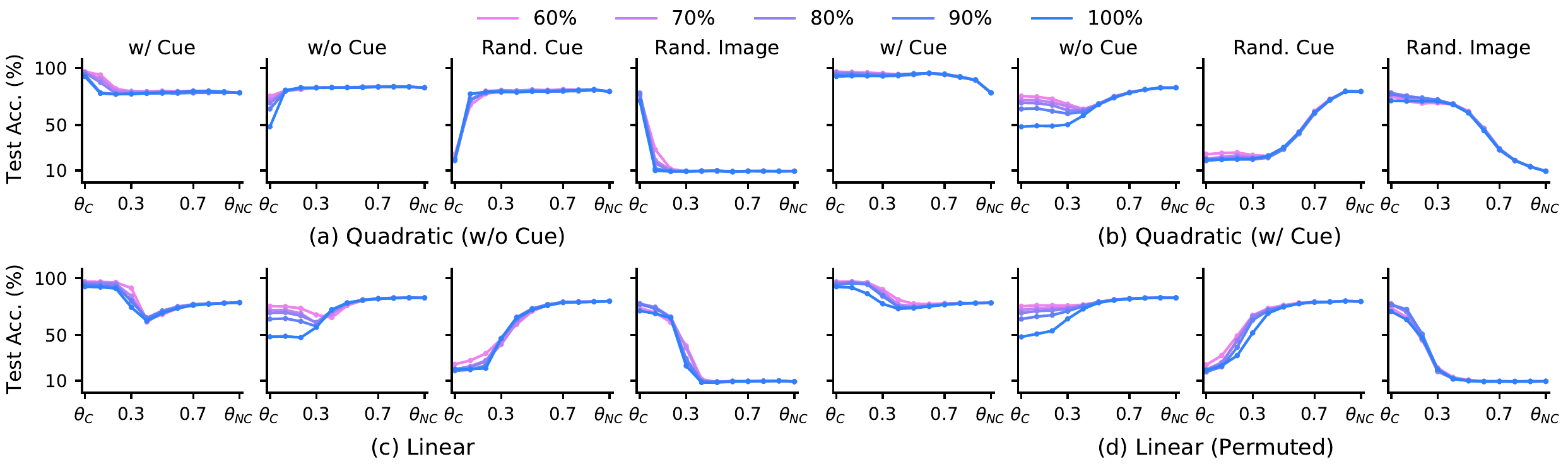}}
  \caption{Test Accuracy.}
\end{subfigure}
\begin{subfigure}[b]{\textwidth}
  \centering
  \centerline{\includegraphics[width=0.9\columnwidth]{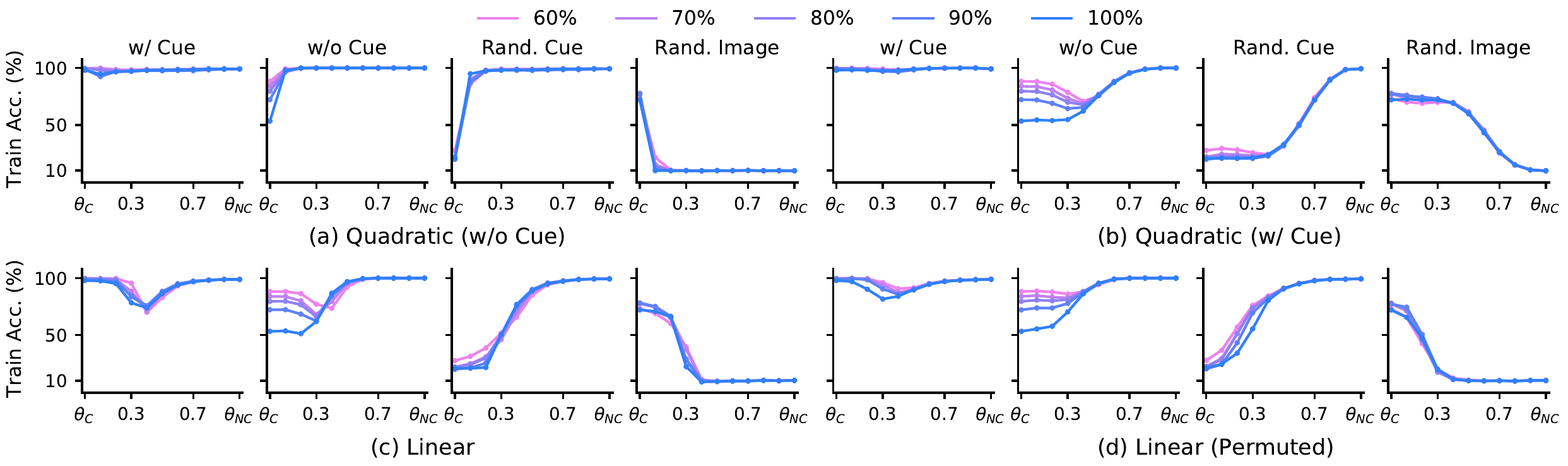}}
  \caption{Train Accuracy.}
\end{subfigure}
\vspace{-1mm}
\begin{subfigure}[b]{\textwidth}
  \centering
  \centerline{\includegraphics[width=0.9\columnwidth]{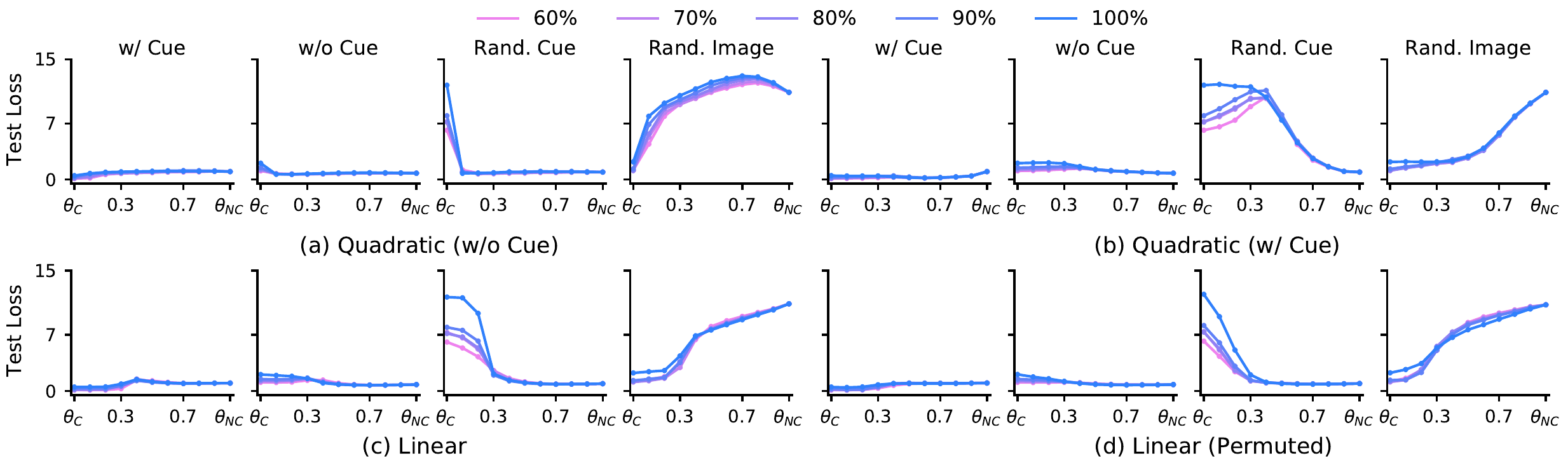}}
  \caption{Test Loss.}
\end{subfigure}
\begin{subfigure}[b]{\textwidth}
  \centering
  \centerline{\includegraphics[width=0.9\columnwidth]{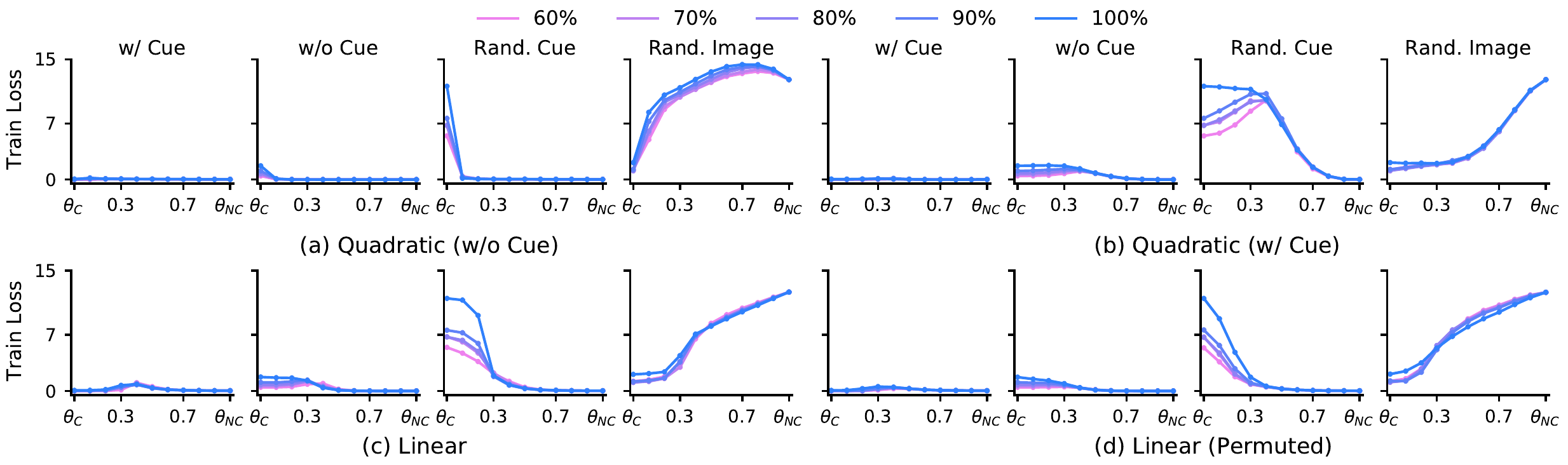}}
  \caption{Train Loss.}
\end{subfigure}
\vspace{-1mm}
\caption{\label{fig:dominoes_vgg}\textbf{VGG-13 on Dominoes}. We plot test/train accuracy/loss curves along different connectivity paths and see thorough corroboration of our claims in the main text: Mechanistically dissimilar minimizers can be connected via nonlinear paths on a given dataset, but behave different on counterfactuals, indicating lack of mechanistic connectivity.}
\vspace{-10pt}
\end{figure}

\begin{figure}[H]
\centering
\begin{subfigure}[b]{\textwidth}
  \centering
  \centerline{\includegraphics[width=0.9\columnwidth]{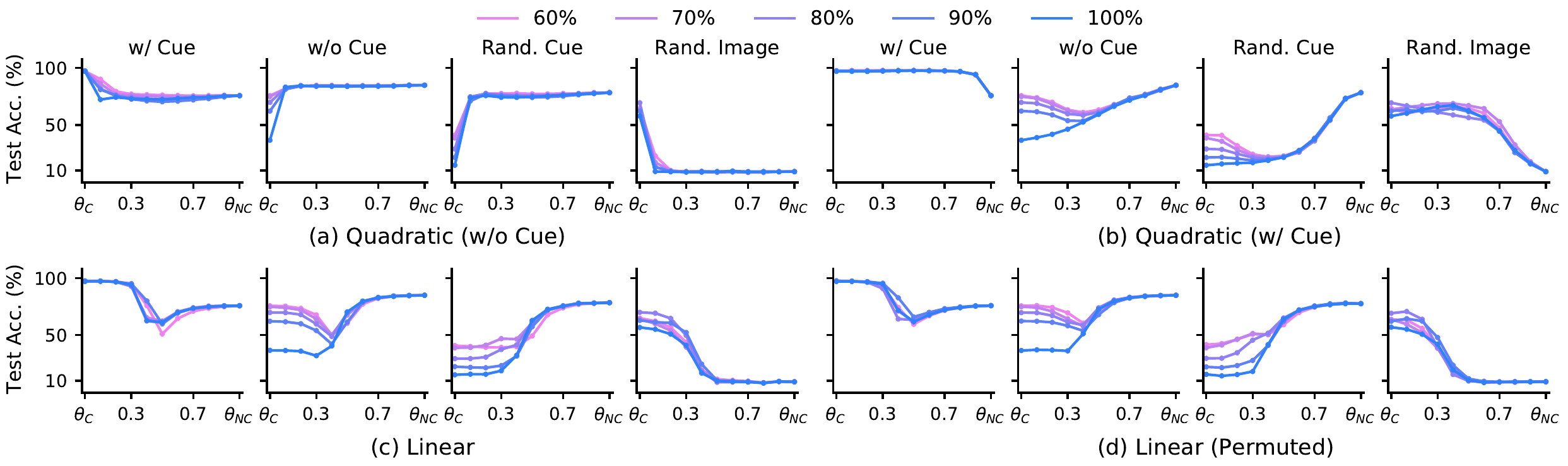}}
  \caption{Test Accuracy.}
\end{subfigure}
\begin{subfigure}[b]{\textwidth}
  \centering
  \centerline{\includegraphics[width=0.9\columnwidth]{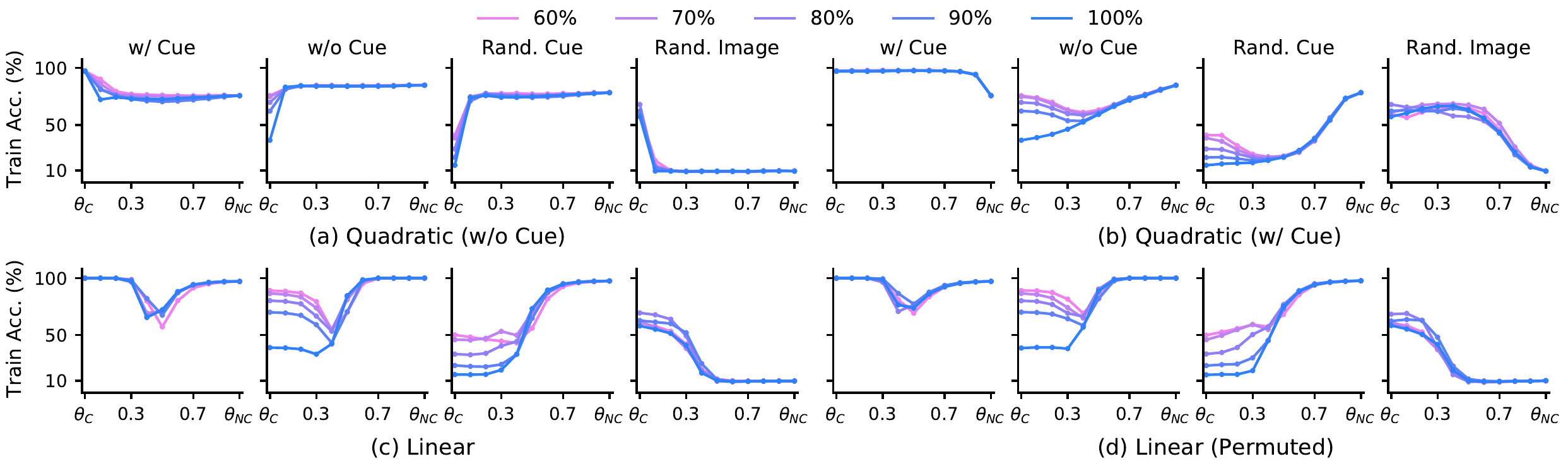}}
  \caption{Train Accuracy.}
\end{subfigure}
\vspace{-1mm}
\begin{subfigure}[b]{\textwidth}
  \centering
  \centerline{\includegraphics[width=0.9\columnwidth]{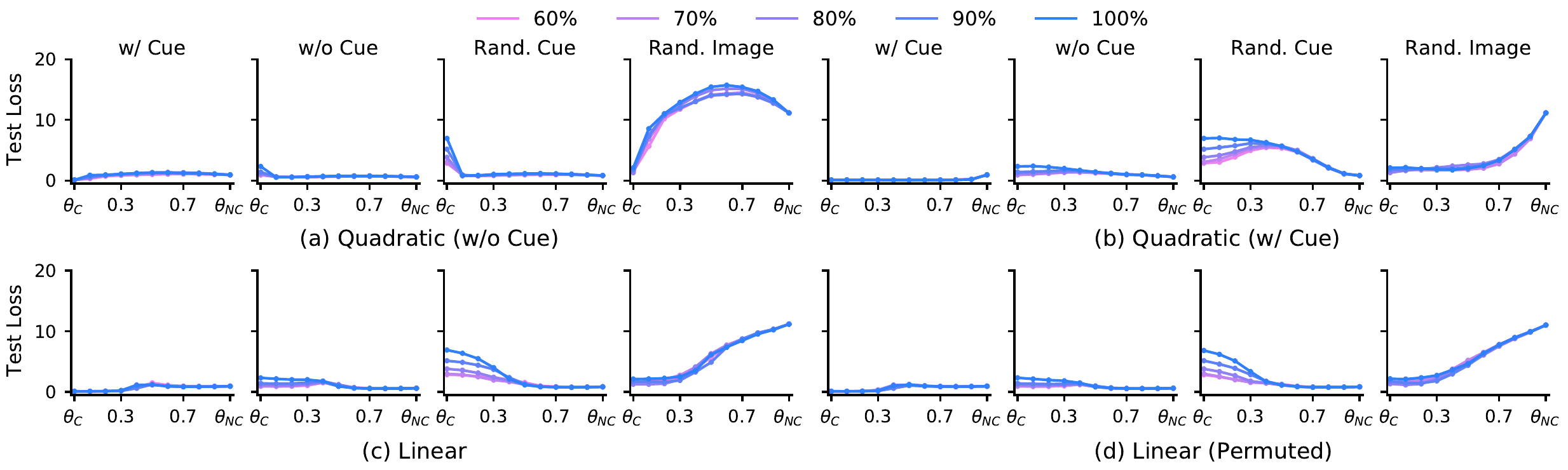}}
  \caption{Test Loss.}
\end{subfigure}
\begin{subfigure}[b]{\textwidth}
  \centering
  \centerline{\includegraphics[width=0.9\columnwidth]{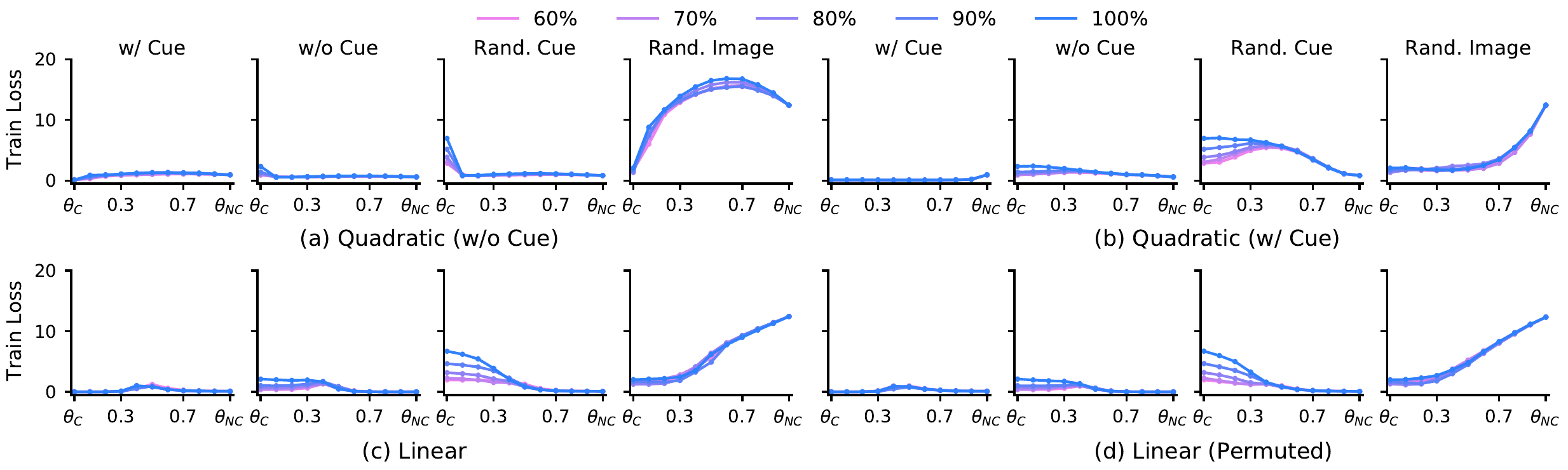}}
  \caption{Train Loss.}
\end{subfigure}
\vspace{-1mm}
\caption{\label{fig:dominoes_res18}\textbf{ResNet-18 on Dominoes}. We plot test/train accuracy/loss curves along different connectivity paths and see thorough corroboration of our claims in the main text: Mechanistically dissimilar minimizers can be connected via nonlinear paths on a given dataset, but behave different on counterfactuals, indicating lack of mechanistic connectivity.}
\vspace{-10pt}
\end{figure}

\section{Further Results: Lack of Linear Connectivity implies Mechanistic Dissimilarity}
\label{app:lmc_results}
We train VGG-13 and ResNet-18 models on our synthetic CIFAR-10 / CIFAR-100 / Dominoes datasets with cues (see Figs.~\ref{fig:c10viz}, \ref{fig:c100viz}, and \ref{fig:dominoesviz}). Corresponding models are denoted $\theta_{\text{C}}$. These models are then fine-tuned on the original CIFAR-10 / CIFAR-100 datasets that do not have any cue features. Specifically, we use different learning rates (LR) and train for 100 epochs with a step-decay schedule (decay at epoch 40 and 80 by a factor of $0.1$). Corresponding models are denoted $\theta_{\text{FT}}$. In the following, plot titles denote evaluation dataset, including datasets where either the cue is present (denoted w/ Cue), absent (denoted w/o Cue), randomized (denoted Rand.\ Cue), or the underlying image is randomized but the cue remains the same (denoted Rand.\ Image). Line colors denote the proportion of dataset that has contains our synthetically embedded cues. 

Across all our results, we see that using a large enough learning rate or enforcing perfect correlation between the cue and label induces loss barriers along the linear path, i.e., linear mode connectivity does not hold. Correspondingly, we see the models respond differently to counterfactuals, i.e, they are mechanistically dissimilar and not connected. For a small enough learning rate, $\theta_{\text{FT}}$ remains mechanistically similar to $\theta_{\text{C}}$, responding similarly on counterfactuals. Correspondingly, we see linear mode connectivity holds between the models on data with cues. 

\begin{figure}[H]
\vspace{10pt}
\centering
\begin{subfigure}[b]{0.49\textwidth}
  \centering
  \centerline{\includegraphics[width=\columnwidth]{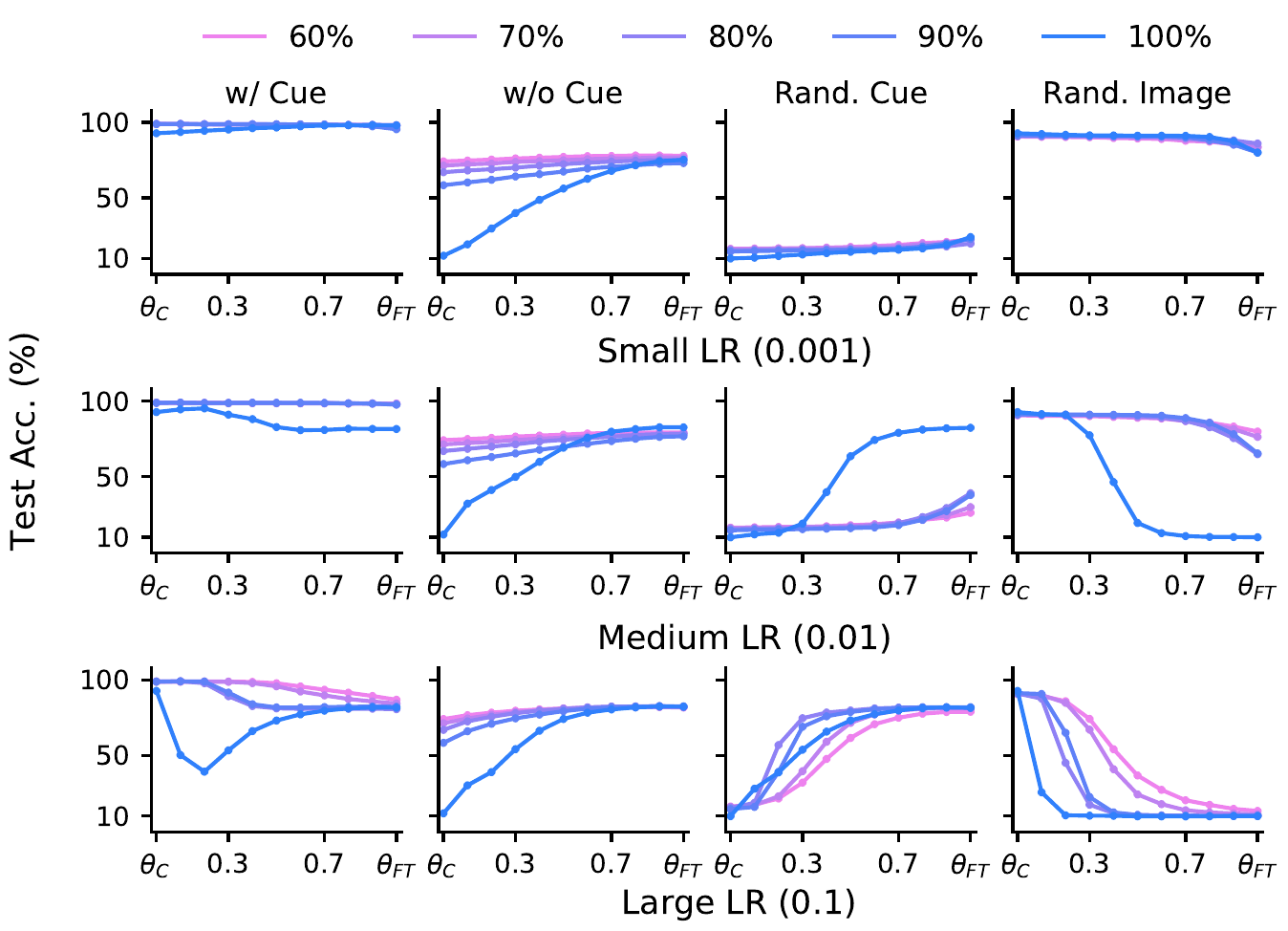}}
  \caption{VGG.}
\end{subfigure}%
\begin{subfigure}[b]{0.49\textwidth}
  \centering
  \centerline{\includegraphics[width=\columnwidth]{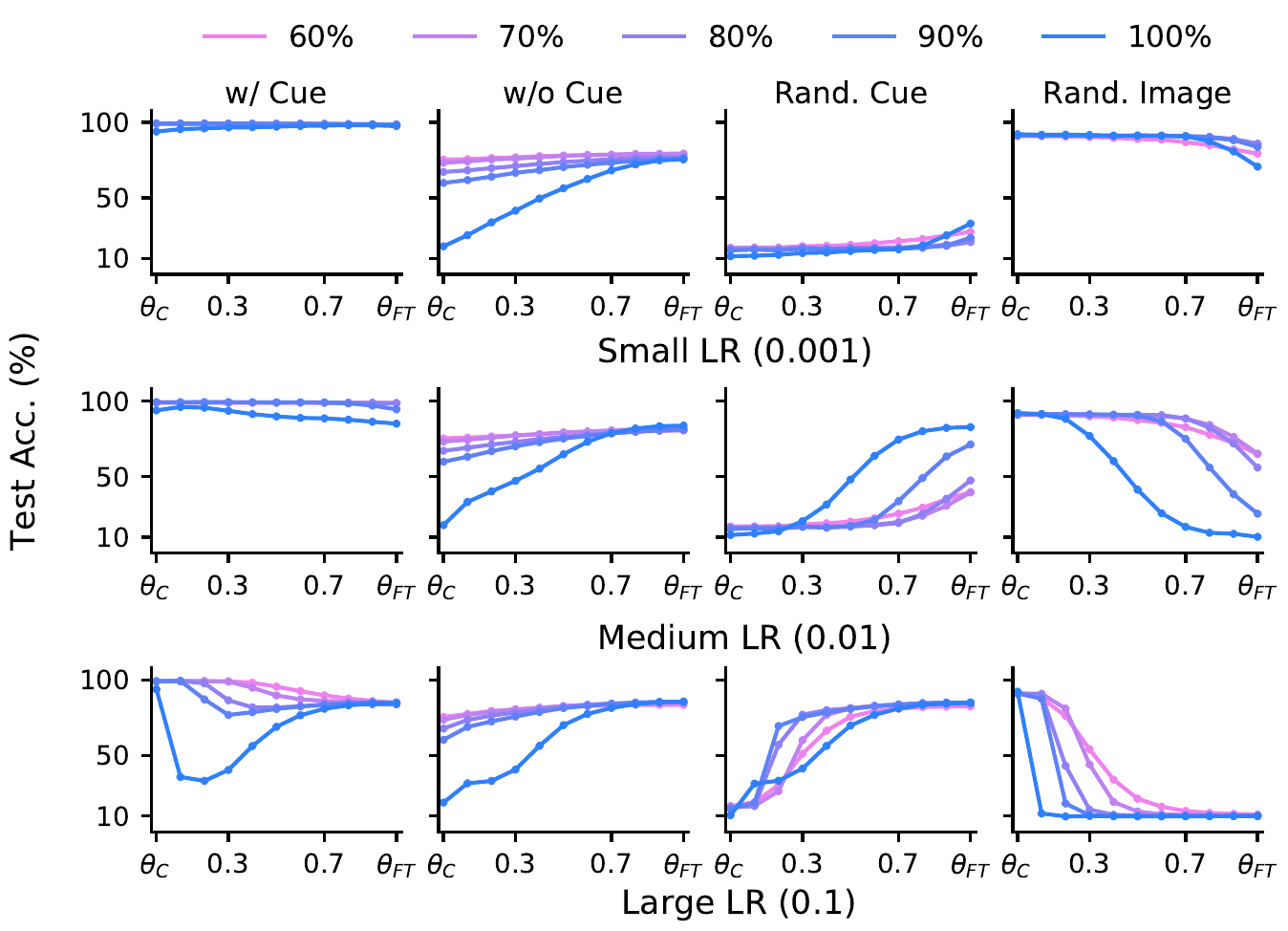}}
  \caption{ResNet18.}
\end{subfigure}
\vspace{-1mm}
\caption{\label{fig:test_ft_c10}\textbf{Fine-tuning of models trained on CIFAR-10 with Box Cue}. We plot test accuracy curves along the linear path between $\theta_{\text{C}}$ and $\theta_{\text{FT}}$ and see thorough corroboration of our claims in the main text: Linearly connected minimizers exhibit mechanistic similarity, behaving identically on counterfactual datasets, indicating mechanistic connectivity.}
\vspace{5pt}
\end{figure}

\begin{figure}[H]
\centering
\begin{subfigure}[b]{0.49\textwidth}
  \centering
  \centerline{\includegraphics[width=\columnwidth]{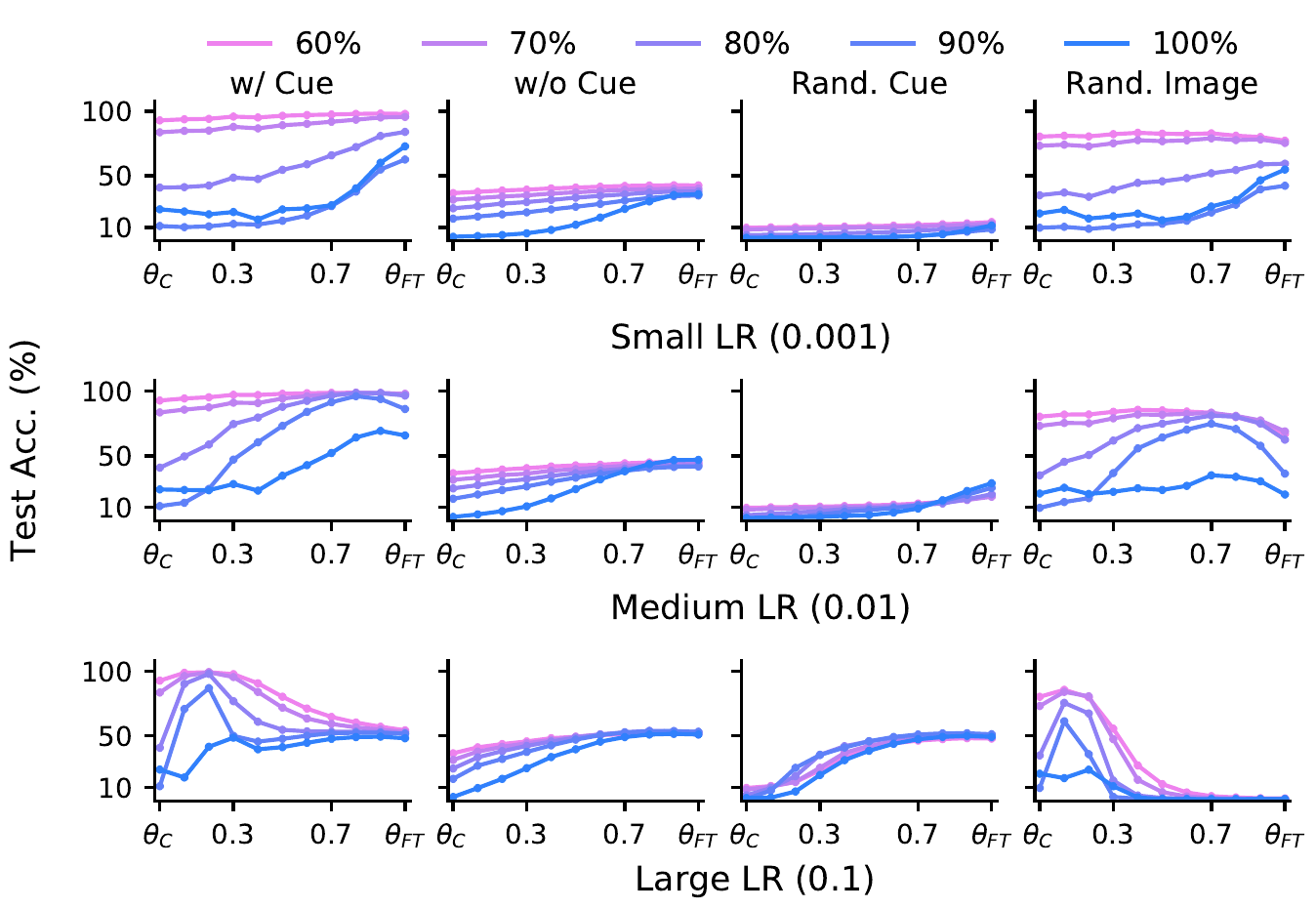}}
  \caption{VGG.}
\end{subfigure}%
\begin{subfigure}[b]{0.49\textwidth}
  \centering
  \centerline{\includegraphics[width=\columnwidth]{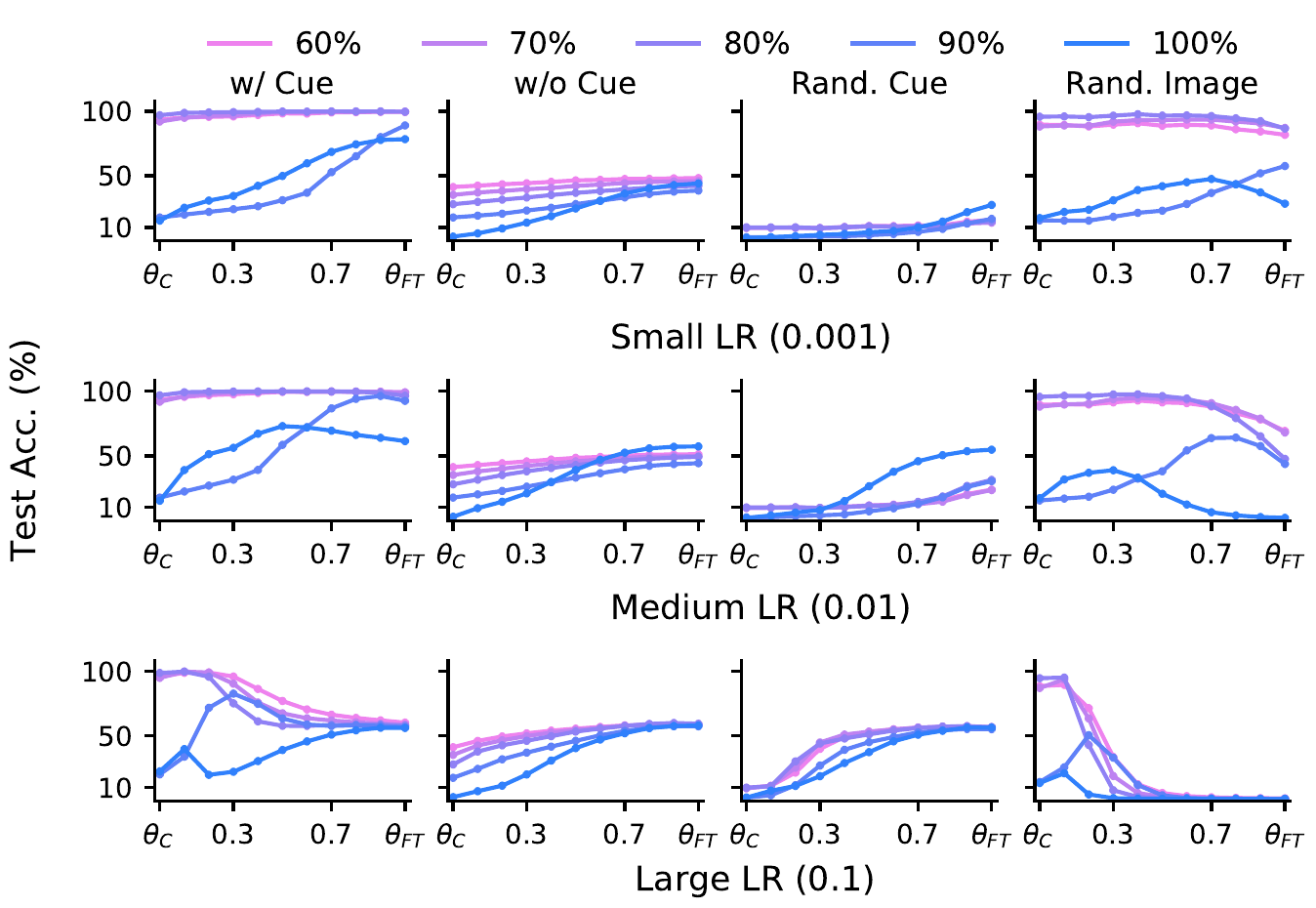}}
  \caption{ResNet18.}
\end{subfigure}
\vspace{-1mm}
\caption{\label{fig:test_ft_c100}\textbf{Fine-tuning of models trained on CIFAR-100 with Box/Color Cue}. We plot test accuracy along the linear path between $\theta_{\text{C}}$ and $\theta_{\text{FT}}$ and see thorough corroboration of our claims in the main text: Linearly connected minimizers exhibit mechanistic similarity, behaving identically on counterfactual datasets, indicating mechanistic connectivity.}
\vspace{5pt}
\end{figure}

\begin{figure}[H]
\centering
\begin{subfigure}[b]{0.49\textwidth}
  \centering
  \centerline{\includegraphics[width=\columnwidth]{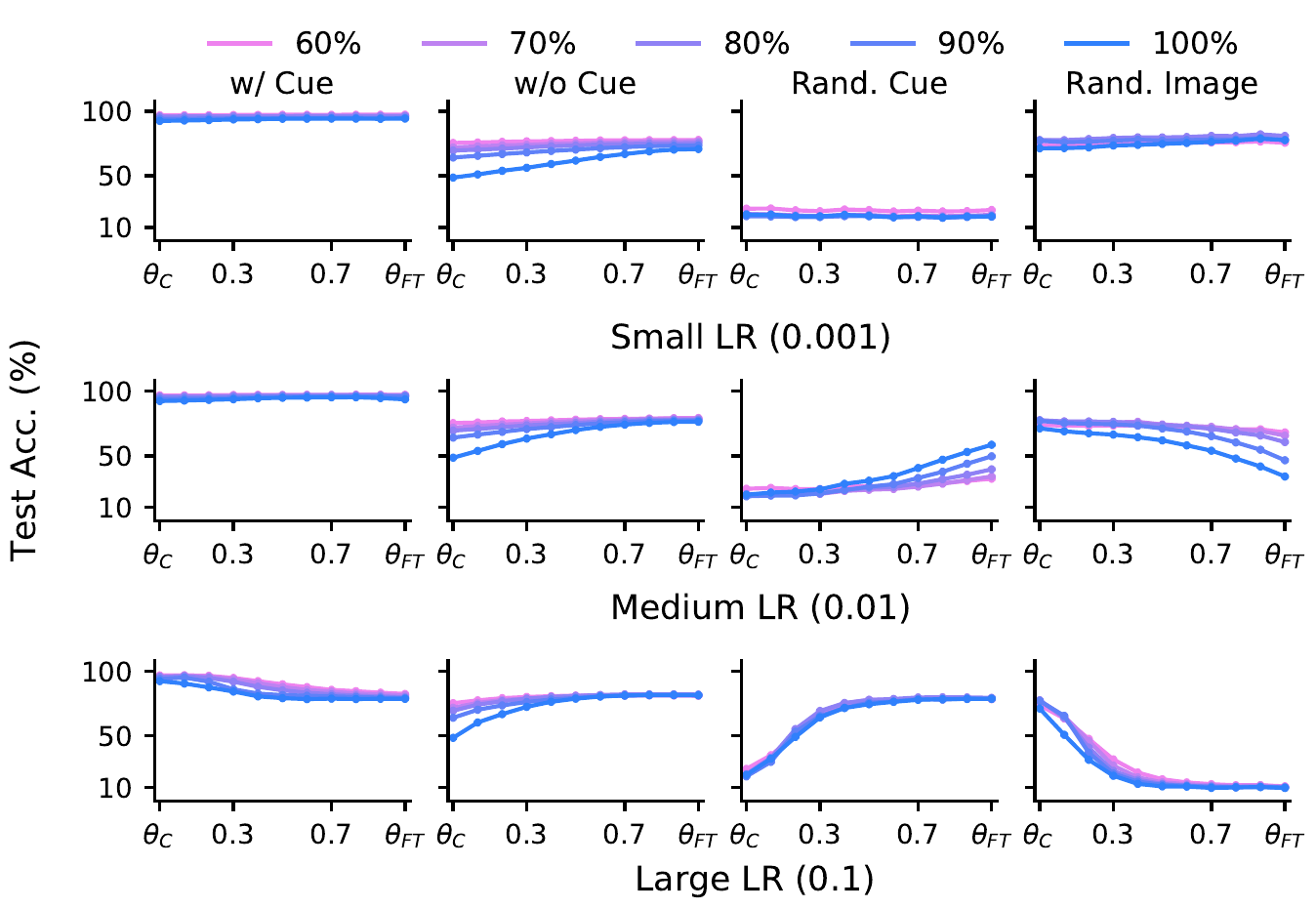}}
  \caption{VGG.}
\end{subfigure}%
\begin{subfigure}[b]{0.49\textwidth}
  \centering
  \centerline{\includegraphics[width=\columnwidth]{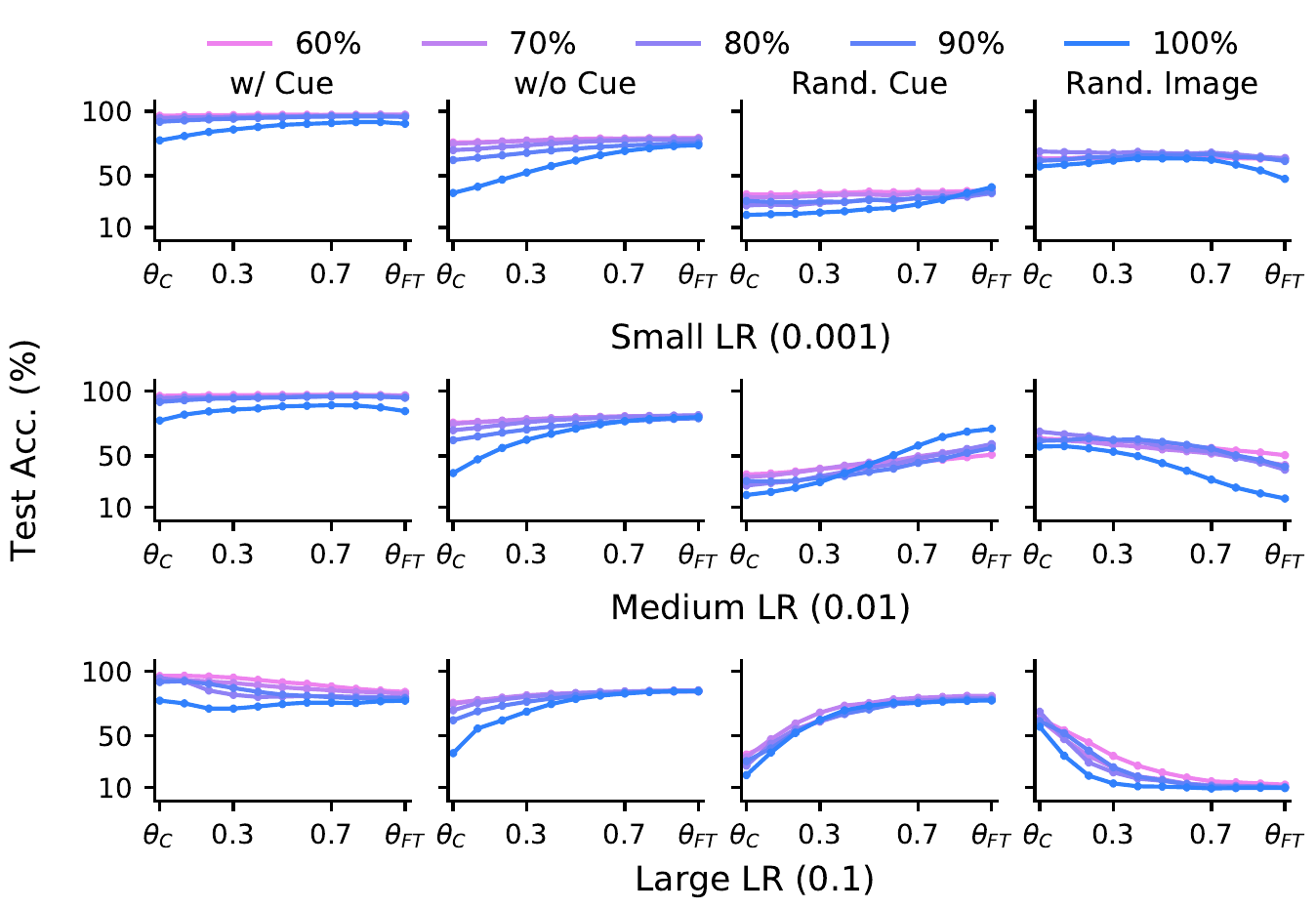}}
  \caption{ResNet18.}
\end{subfigure}
\vspace{-1mm}
\caption{\label{fig:test_ft_dominoes}\textbf{Fine-tuning of models trained on Dominoes}. We plot test accuracy along the linear path between $\theta_{\text{C}}$ and $\theta_{\text{FT}}$ and see thorough corroboration of our claims in the main text: Linearly connected minimizers exhibit mechanistic similarity, behaving identically on counterfactual datasets, indicating mechanistic connectivity.}
\vspace{5pt}
\end{figure}

\begin{figure}[H]
\centering
\begin{subfigure}[b]{0.49\textwidth}
  \centering
  \centerline{\includegraphics[width=\columnwidth]{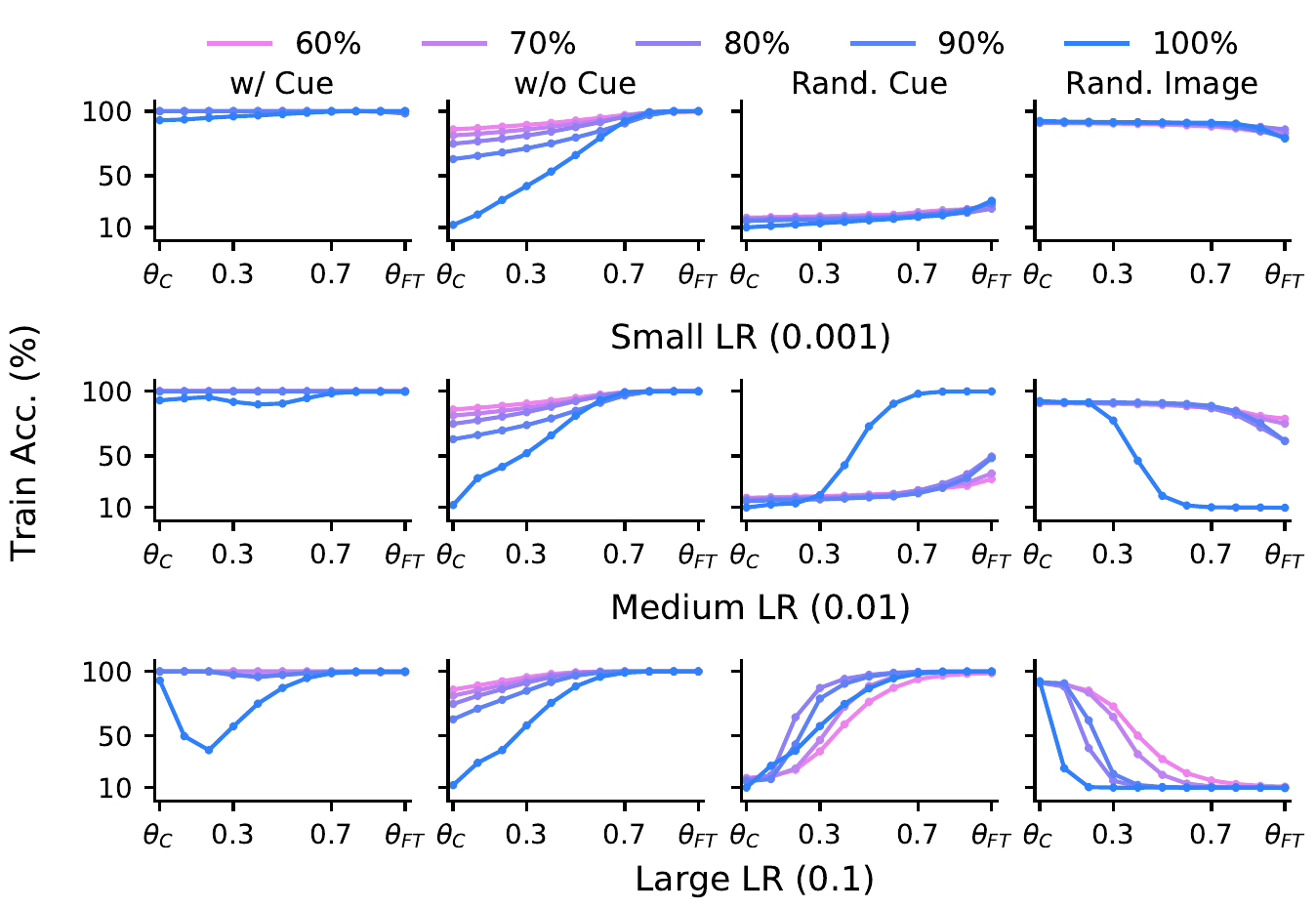}}
  \caption{VGG.}
\end{subfigure}%
\begin{subfigure}[b]{0.49\textwidth}
  \centering
  \centerline{\includegraphics[width=\columnwidth]{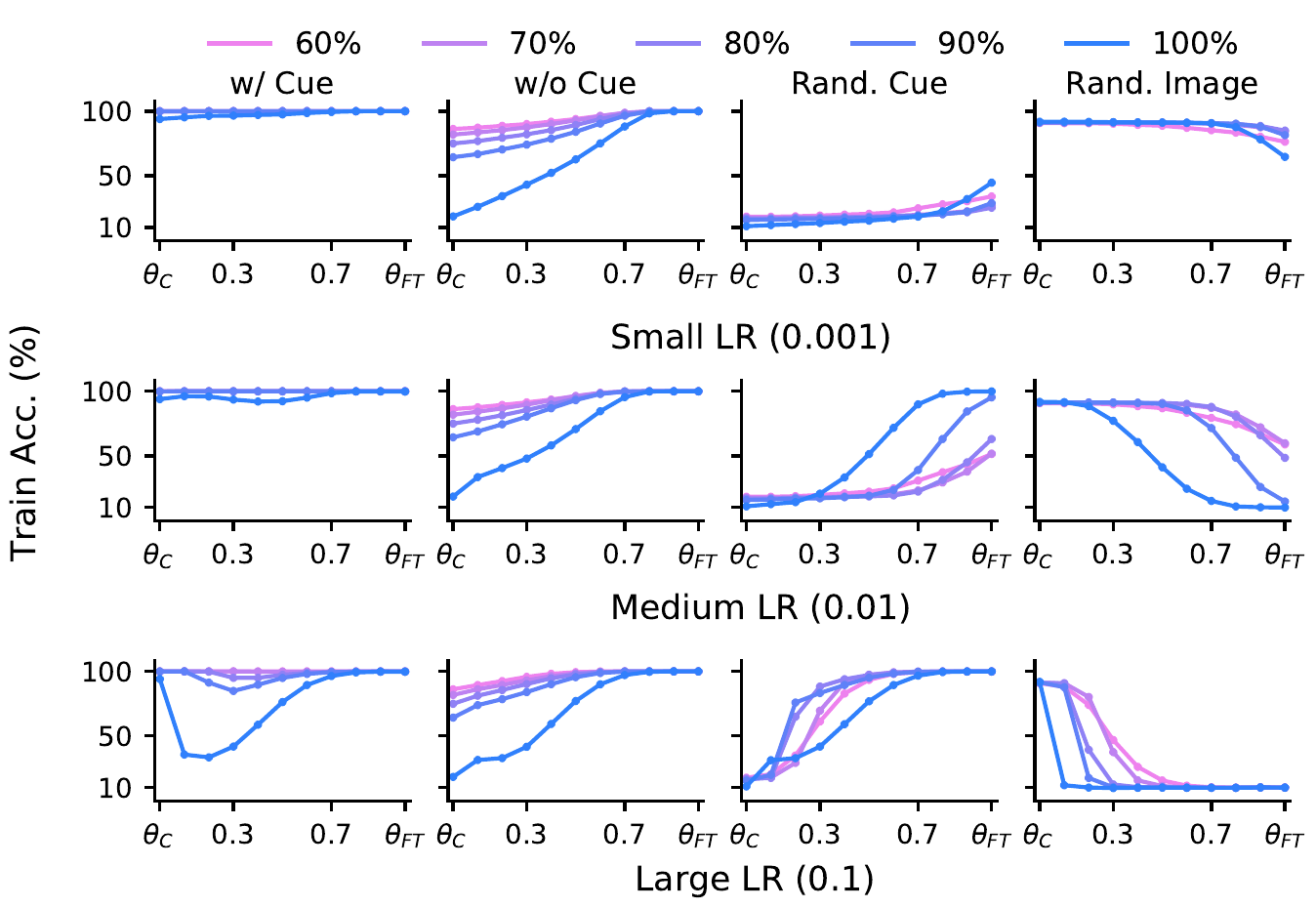}}
  \caption{ResNet18.}
\end{subfigure}
\vspace{-1mm}
\caption{\label{fig:train_ft_c10}\textbf{Fine-tuning of models trained on CIFAR-10 with Box Cue}. We plot train accuracy curves along the linear path between $\theta_{\text{C}}$ and $\theta_{\text{FT}}$ and see thorough corroboration of our claims in the main text: Linearly connected minimizers exhibit mechanistic similarity, behaving identically on counterfactual datasets, indicating mechanistic connectivity.}
\vspace{5pt}
\end{figure}

\begin{figure}[H]
\centering
\begin{subfigure}[b]{0.49\textwidth}
  \centering
  \centerline{\includegraphics[width=\columnwidth]{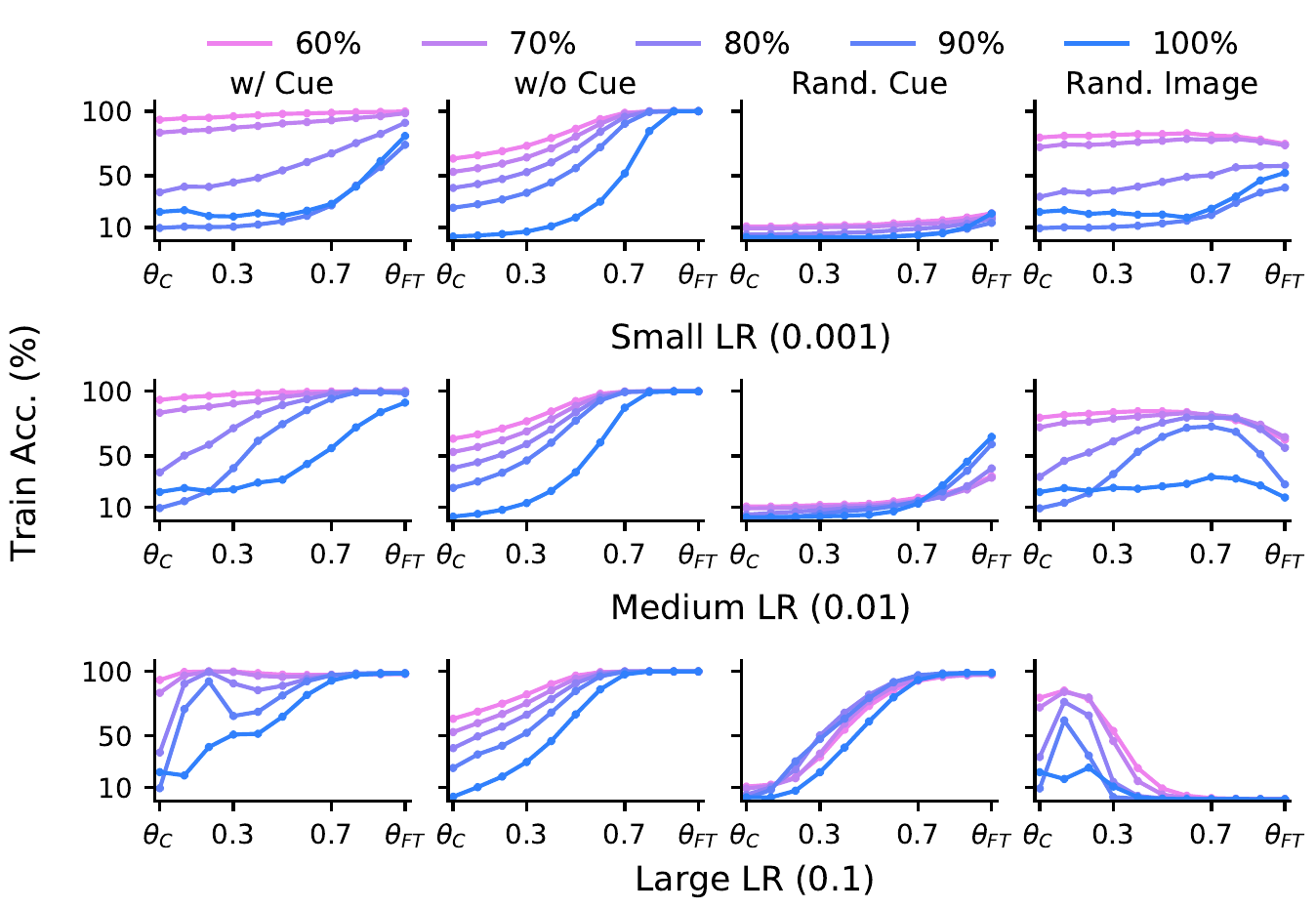}}
  \caption{VGG.}
\end{subfigure}%
\begin{subfigure}[b]{0.49\textwidth}
  \centering
  \centerline{\includegraphics[width=\columnwidth]{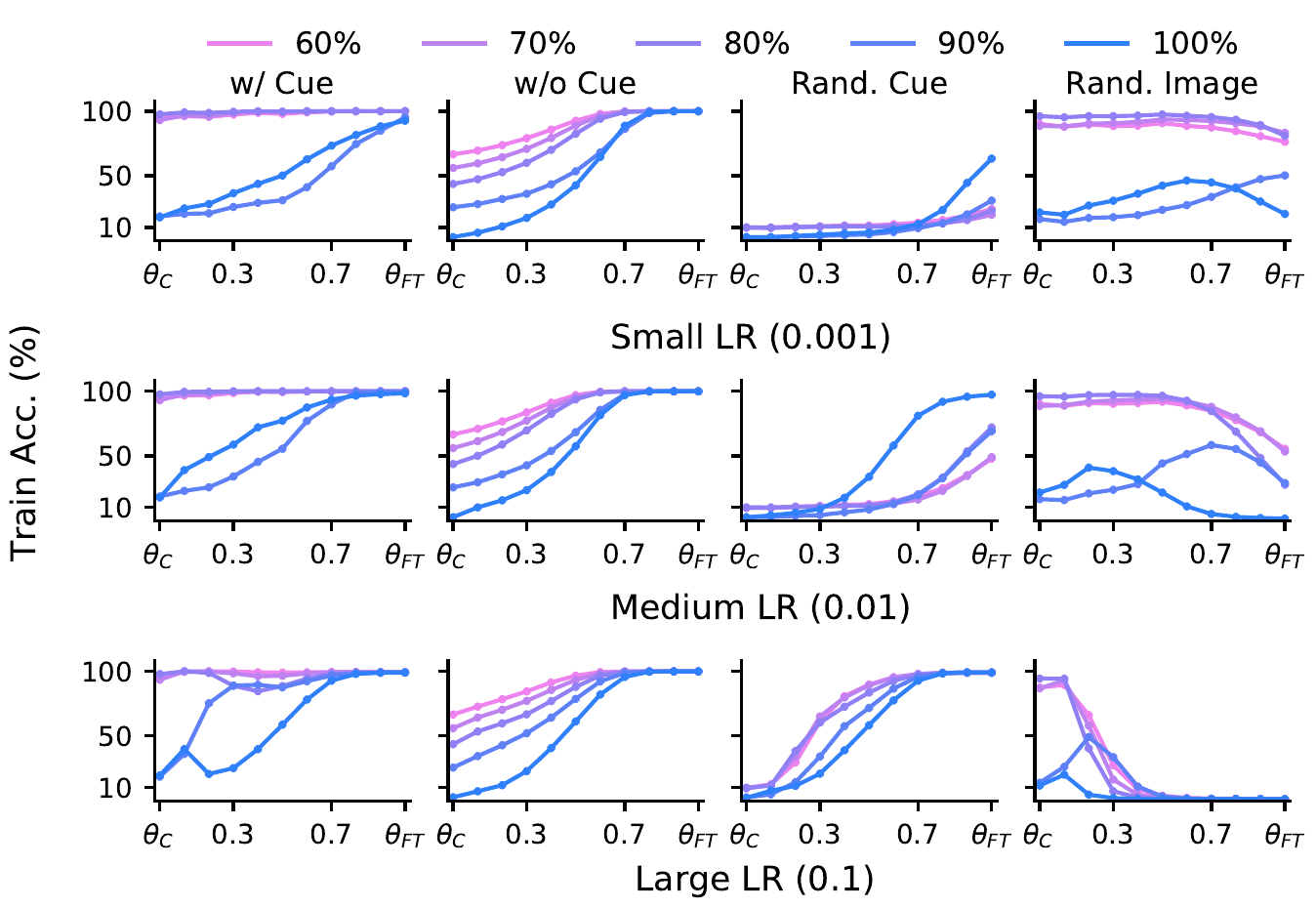}}
  \caption{ResNet18.}
\end{subfigure}
\vspace{-1mm}
\caption{\label{fig:train_ft_c100}\textbf{Fine-tuning of models trained on CIFAR-100 with Box/Color Cue}. We plot train accuracy curves along the linear path between $\theta_{\text{C}}$ and $\theta_{\text{FT}}$ and see thorough corroboration of our claims in the main text: Linearly connected minimizers exhibit mechanistic similarity, behaving identically on counterfactual datasets, indicating mechanistic connectivity.}
\vspace{5pt}
\end{figure}

\begin{figure}[H]
\centering
\begin{subfigure}[b]{0.49\textwidth}
  \centering
  \centerline{\includegraphics[width=\columnwidth]{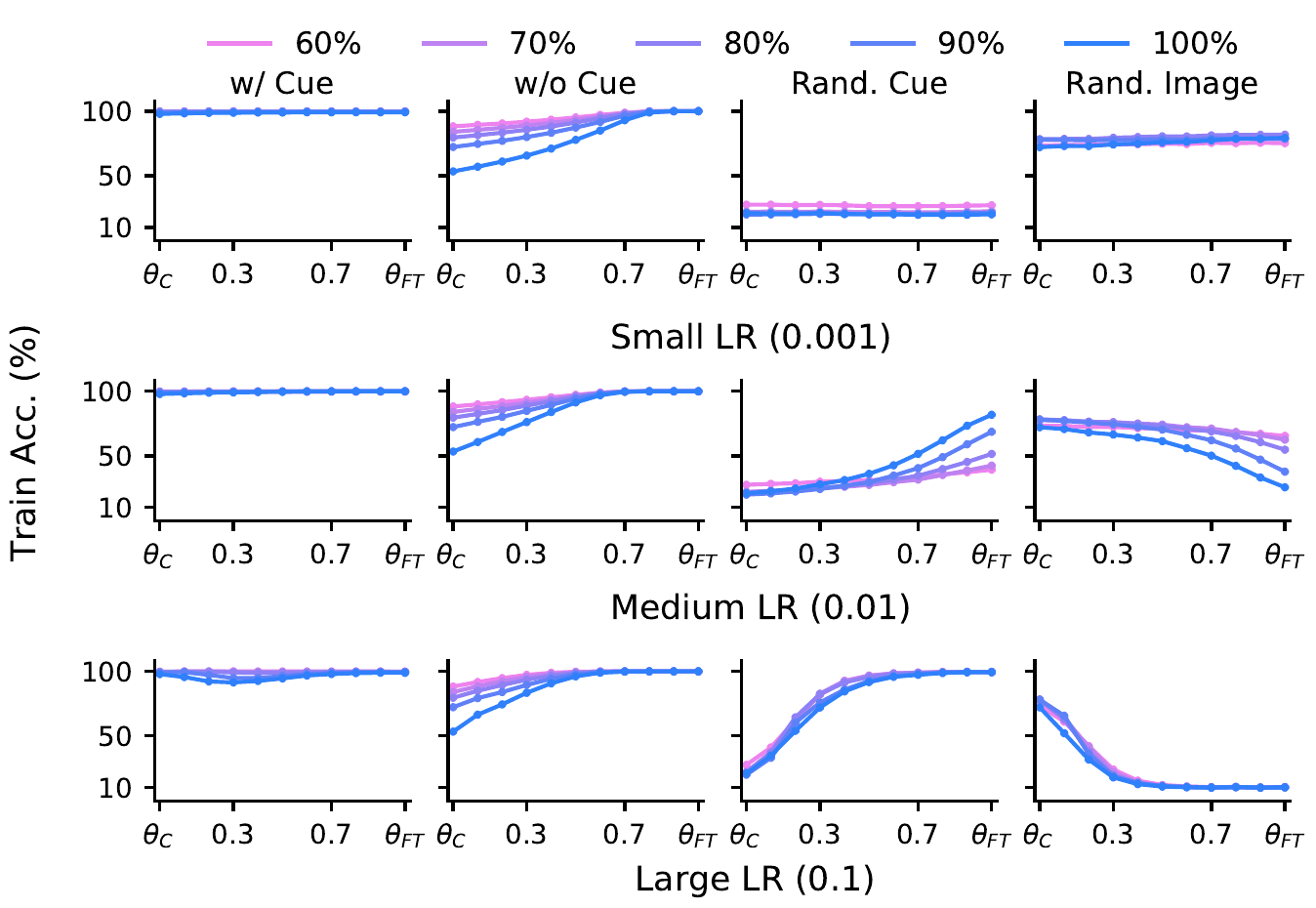}}
  \caption{VGG.}
\end{subfigure}%
\begin{subfigure}[b]{0.49\textwidth}
  \centering
  \centerline{\includegraphics[width=\columnwidth]{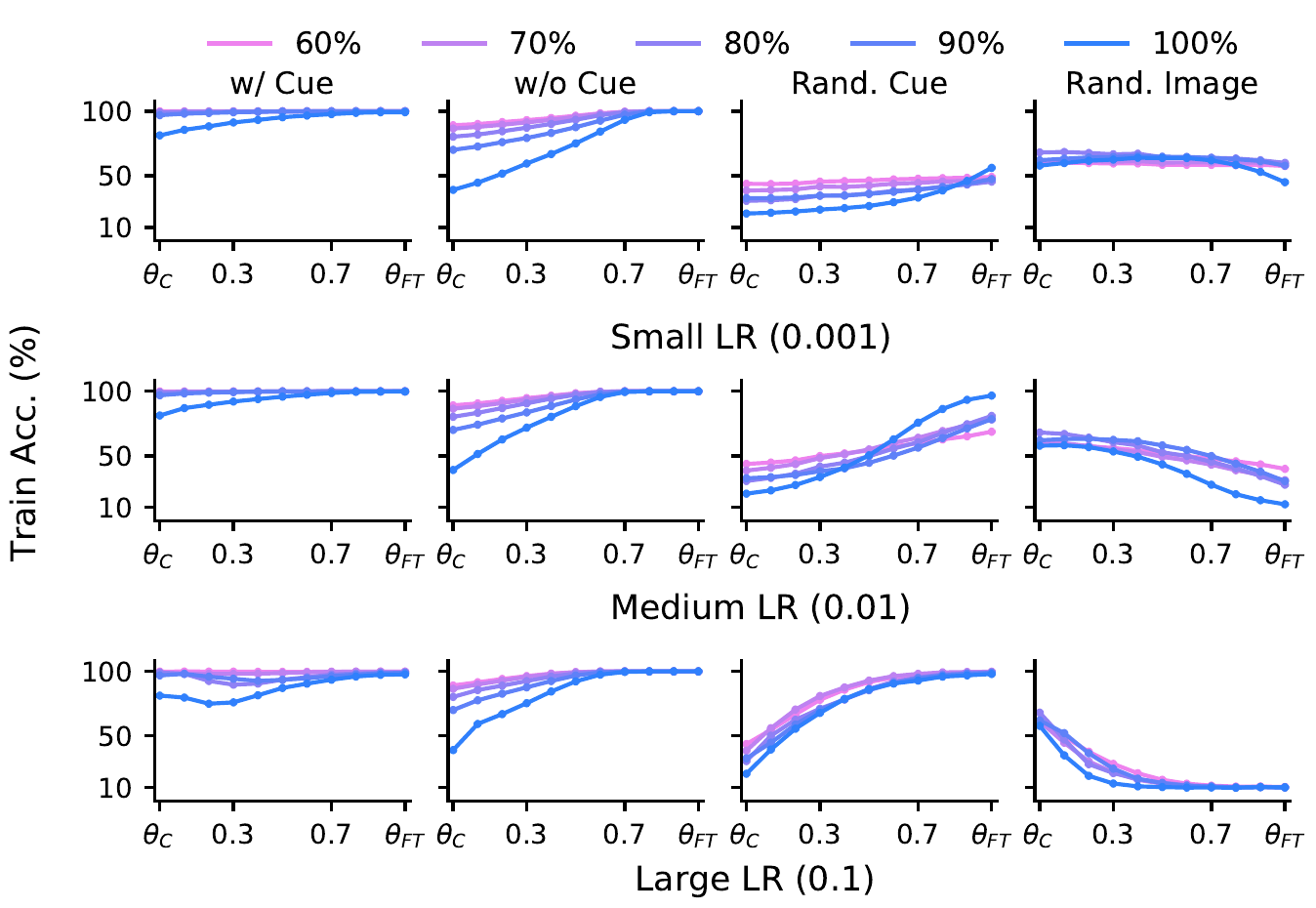}}
  \caption{ResNet18.}
\end{subfigure}
\vspace{-1mm}
\caption{\label{fig:train_ft_dominoes}\textbf{Fine-tuning of models trained on Dominoes}. We plot test accuracy curves along the linear path between $\theta_{\text{C}}$ and $\theta_{\text{FT}}$ and see thorough corroboration of our claims in the main text: Linearly connected minimizers exhibit mechanistic similarity, behaving identically on counterfactual datasets, indicating mechanistic connectivity.}
\vspace{5pt}
\end{figure}

\begin{figure}[H]
\centering
\begin{subfigure}[b]{0.49\textwidth}
  \centering
  \centerline{\includegraphics[width=\columnwidth]{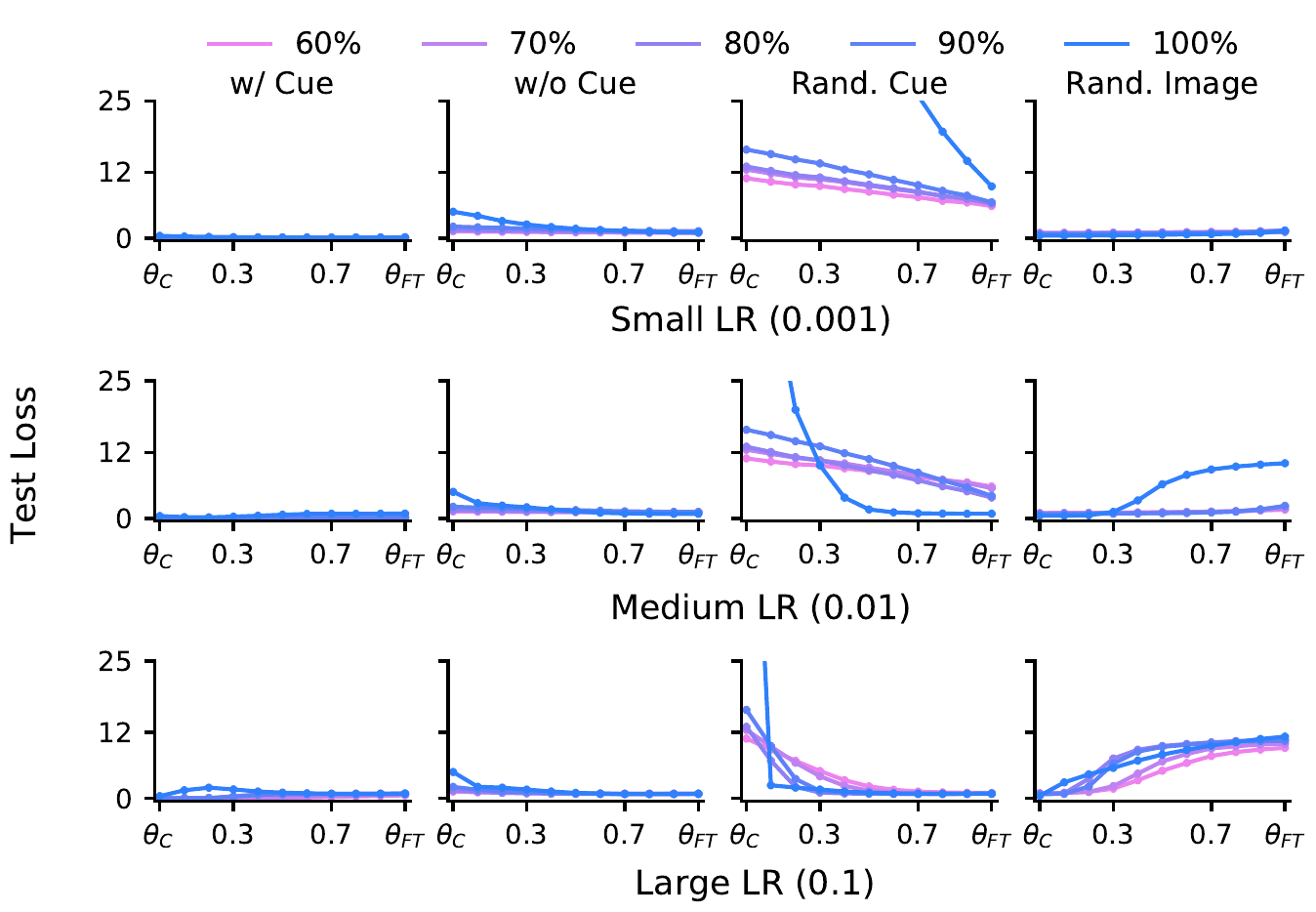}}
  \caption{VGG.}
\end{subfigure}%
\begin{subfigure}[b]{0.49\textwidth}
  \centering
  \centerline{\includegraphics[width=\columnwidth]{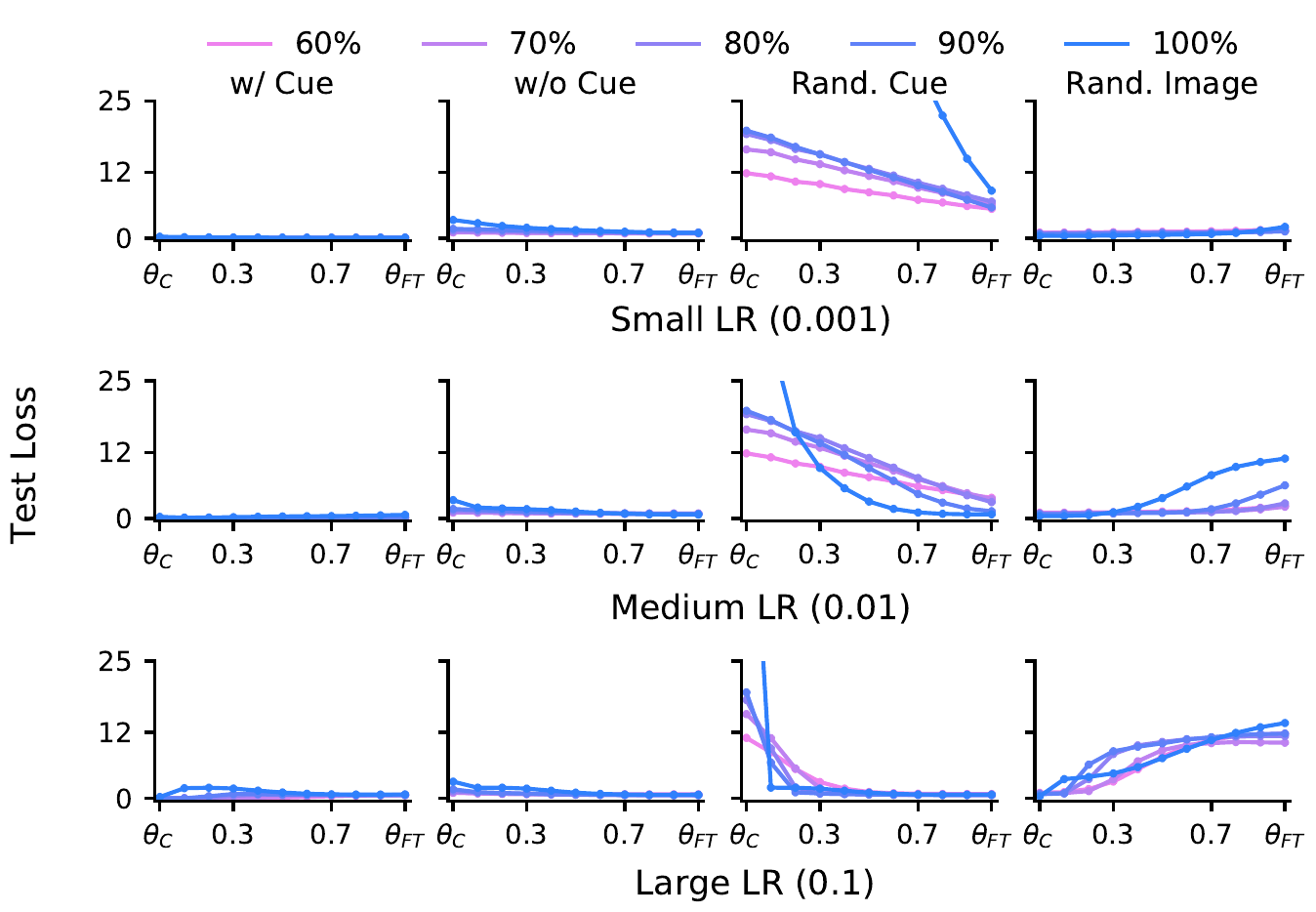}}
  \caption{ResNet18.}
\end{subfigure}
\vspace{-1mm}
\caption{\label{fig:testloss_ft_c10}\textbf{Fine-tuning of models trained on CIFAR-10 with Box Cue}. We plot test loss curves along the linear path between $\theta_{\text{C}}$ and $\theta_{\text{FT}}$ and see thorough corroboration of our claims in the main text: Linearly connected minimizers exhibit mechanistic similarity, behaving identically on counterfactual datasets, indicating mechanistic connectivity.}
\vspace{5pt}
\end{figure}

\begin{figure}[H]
\centering
\begin{subfigure}[b]{0.49\textwidth}
  \centering
  \centerline{\includegraphics[width=\columnwidth]{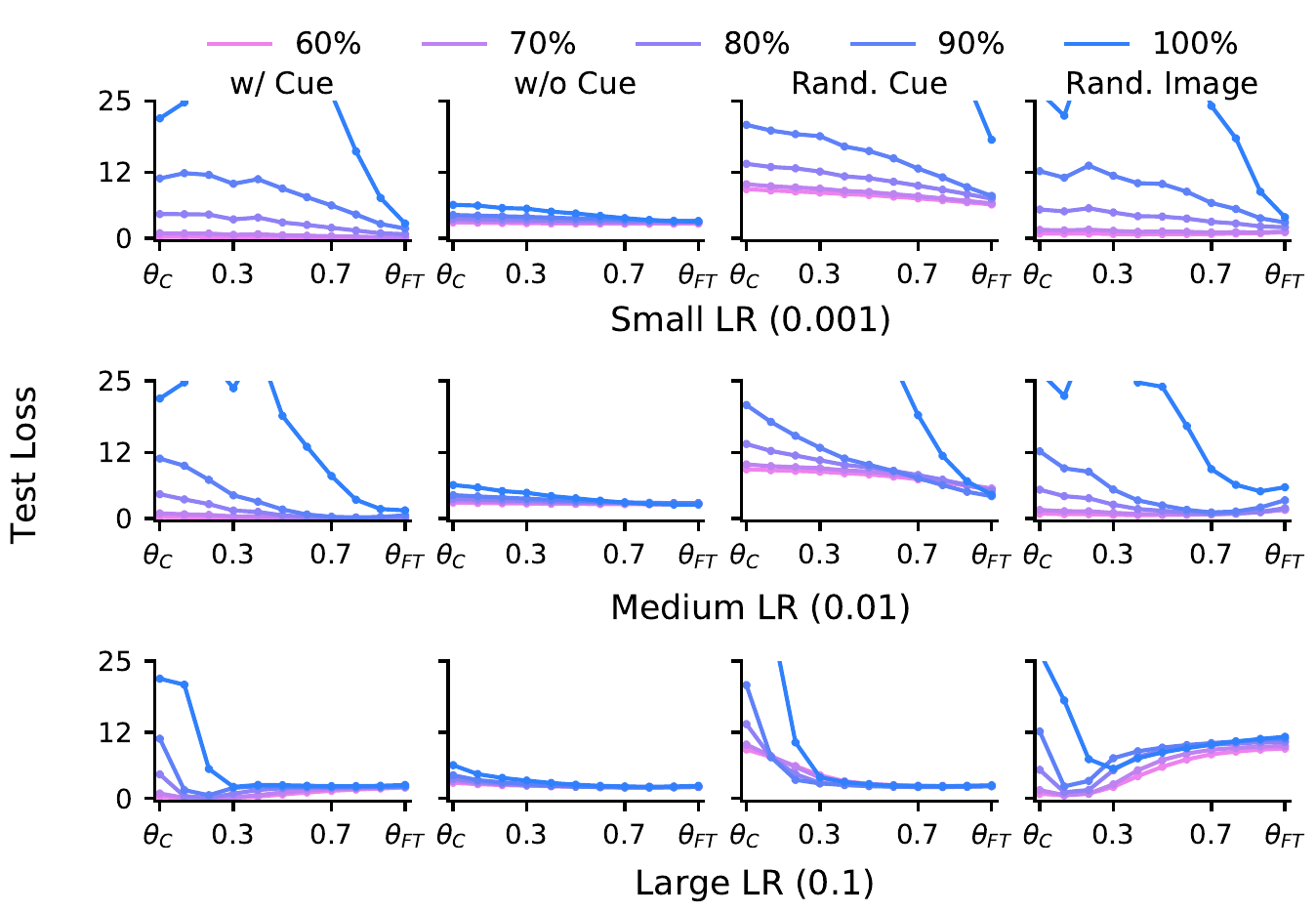}}
  \caption{VGG.}
\end{subfigure}%
\begin{subfigure}[b]{0.49\textwidth}
  \centering
  \centerline{\includegraphics[width=\columnwidth]{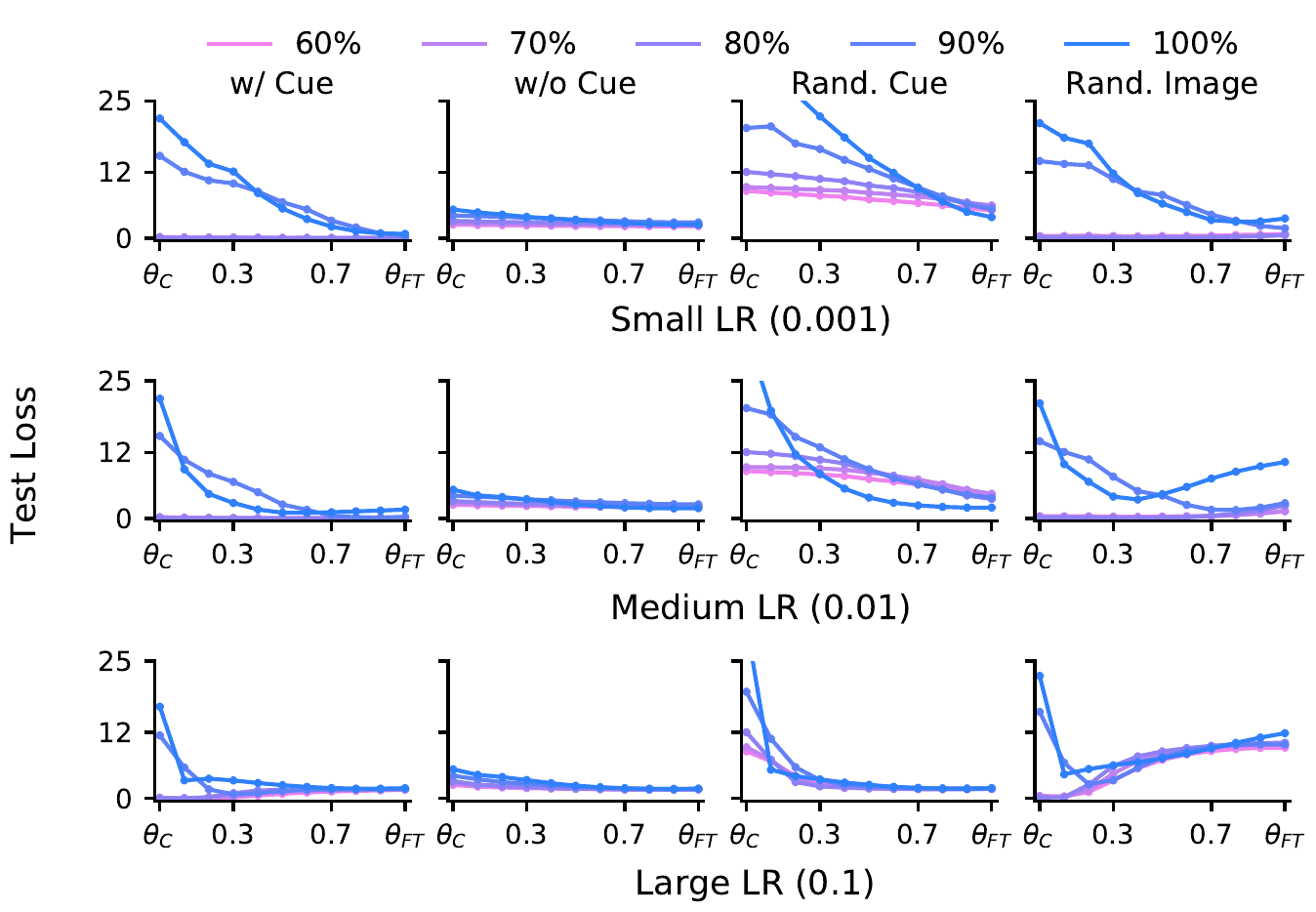}}
  \caption{ResNet18.}
\end{subfigure}
\vspace{-1mm}
\caption{\label{fig:testloss_ft_c100}\textbf{Fine-tuning of models trained on CIFAR-100 with Box / Color Cue}. We plot test loss curves along the linear path between $\theta_{\text{C}}$ and $\theta_{\text{FT}}$ and see thorough corroboration of our claims in the main text: Linearly connected minimizers exhibit mechanistic similarity, behaving identically on counterfactual datasets, indicating mechanistic connectivity.}
\vspace{5pt}
\end{figure}

\begin{figure}[H]
\centering
\begin{subfigure}[b]{0.49\textwidth}
  \centering
  \centerline{\includegraphics[width=\columnwidth]{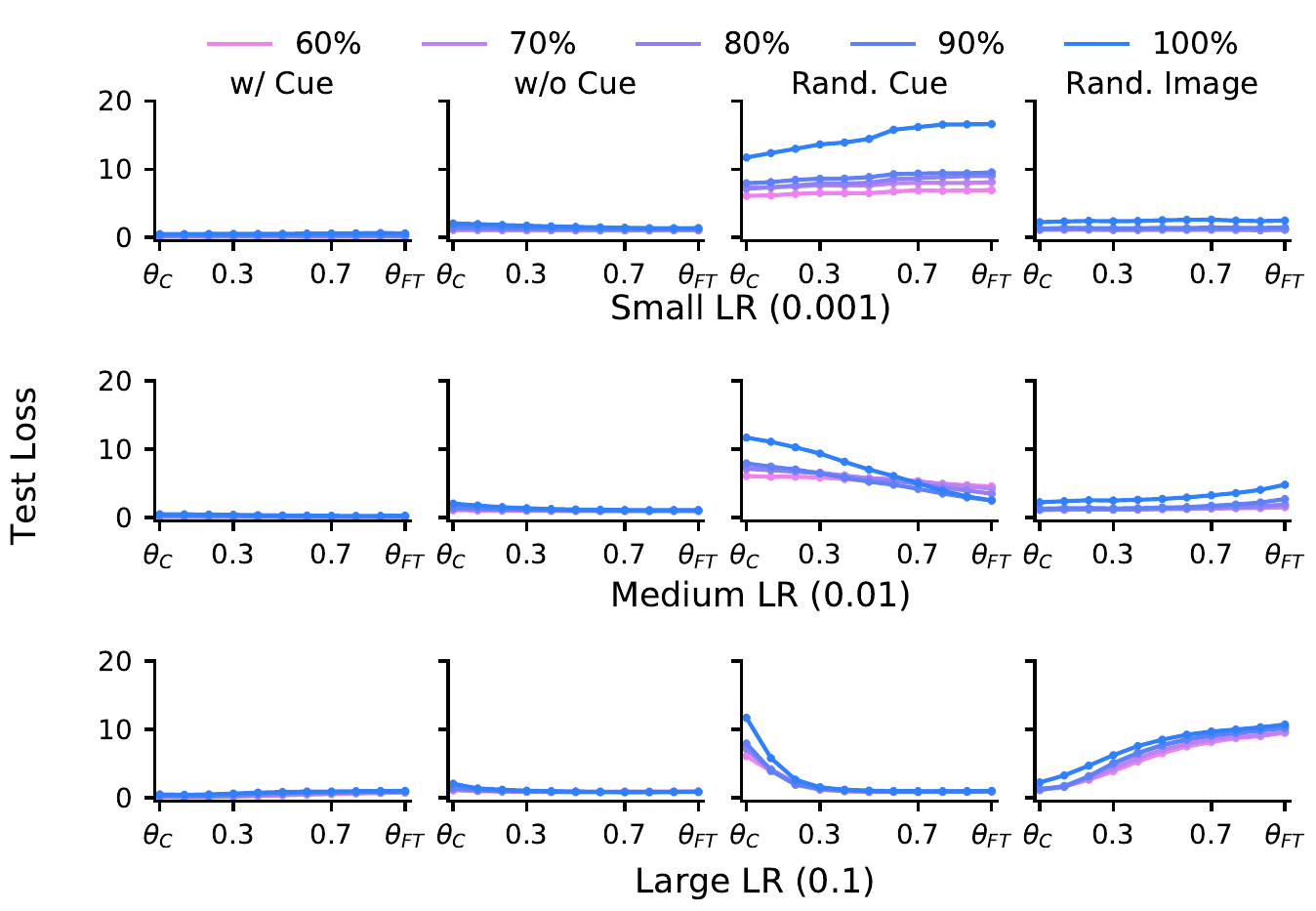}}
  \caption{VGG.}
\end{subfigure}%
\begin{subfigure}[b]{0.49\textwidth}
  \centering
  \centerline{\includegraphics[width=\columnwidth]{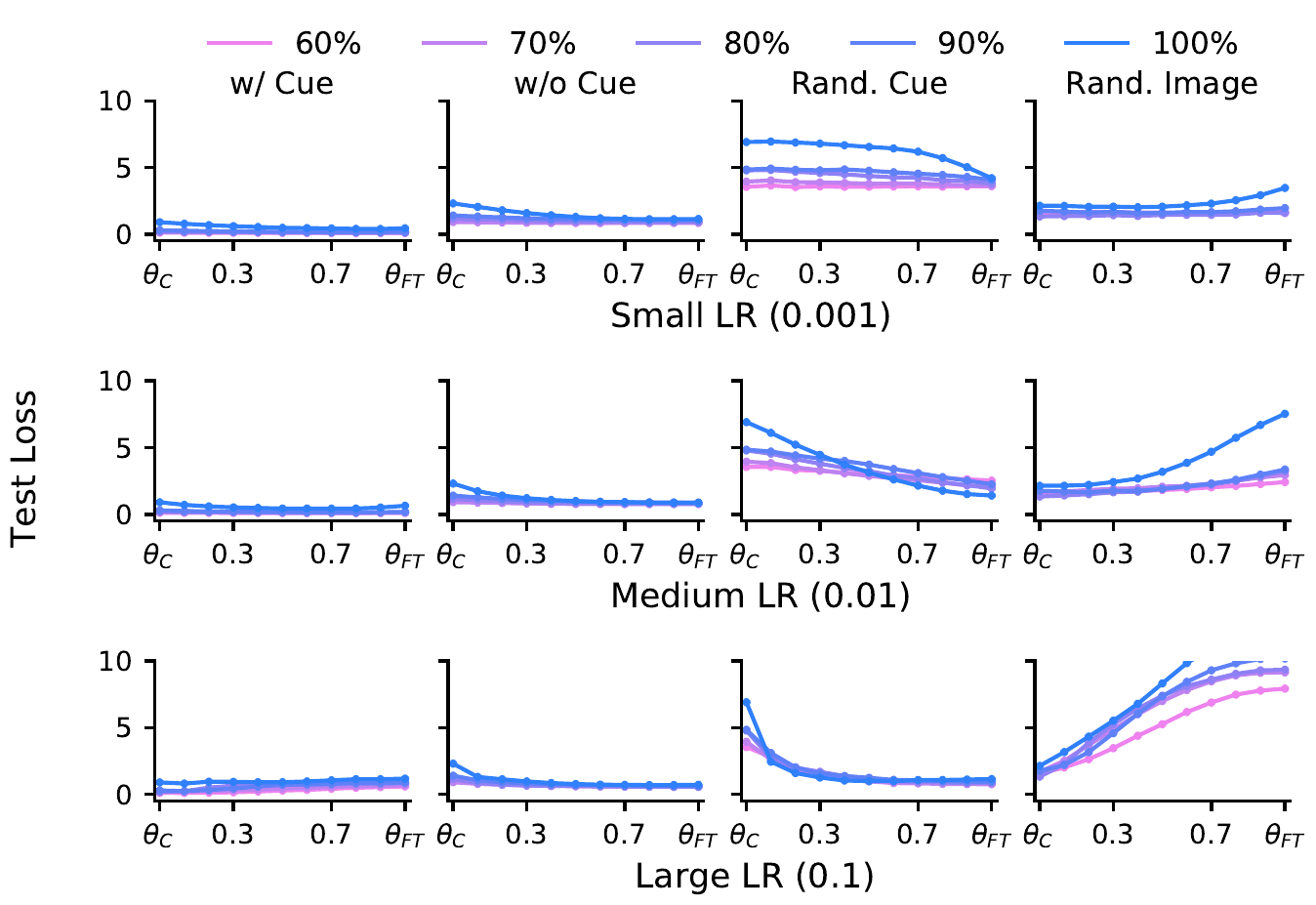}}
  \caption{ResNet18.}
\end{subfigure}
\vspace{-1mm}
\caption{\label{fig:testloss_ft_dominoes}\textbf{Fine-tuning of models trained on Dominoes}. We plot test loss curves along the linear path between $\theta_{\text{C}}$ and $\theta_{\text{FT}}$ and see thorough corroboration of our claims in the main text: Linearly connected minimizers exhibit mechanistic similarity, behaving identically on counterfactual datasets, indicating mechanistic connectivity.}
\vspace{5pt}
\end{figure}

\begin{figure}[H]
\centering
\begin{subfigure}[b]{0.49\textwidth}
  \centering
  \centerline{\includegraphics[width=\columnwidth]{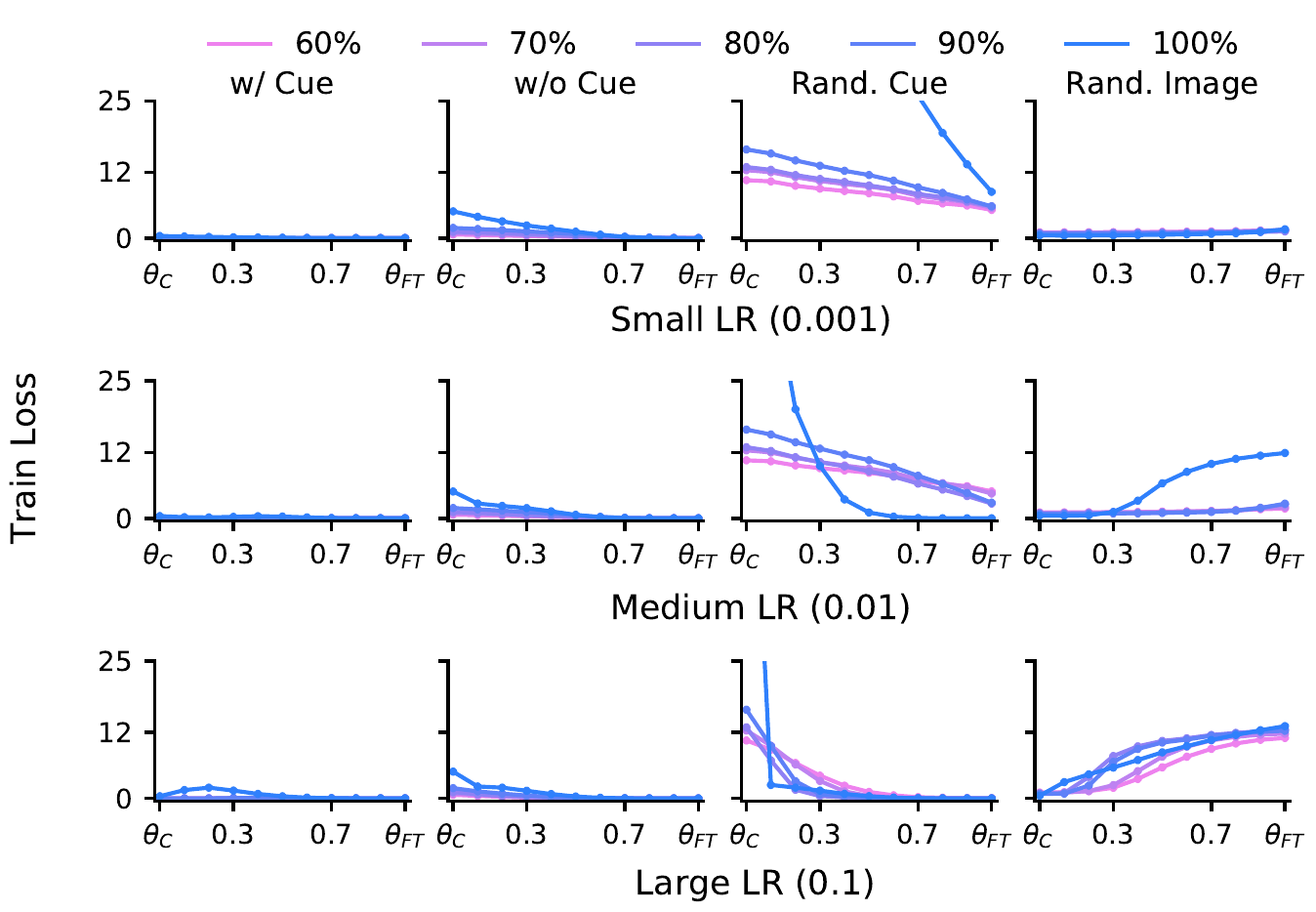}}
  \caption{VGG.}
\end{subfigure}%
\begin{subfigure}[b]{0.49\textwidth}
  \centering
  \centerline{\includegraphics[width=\columnwidth]{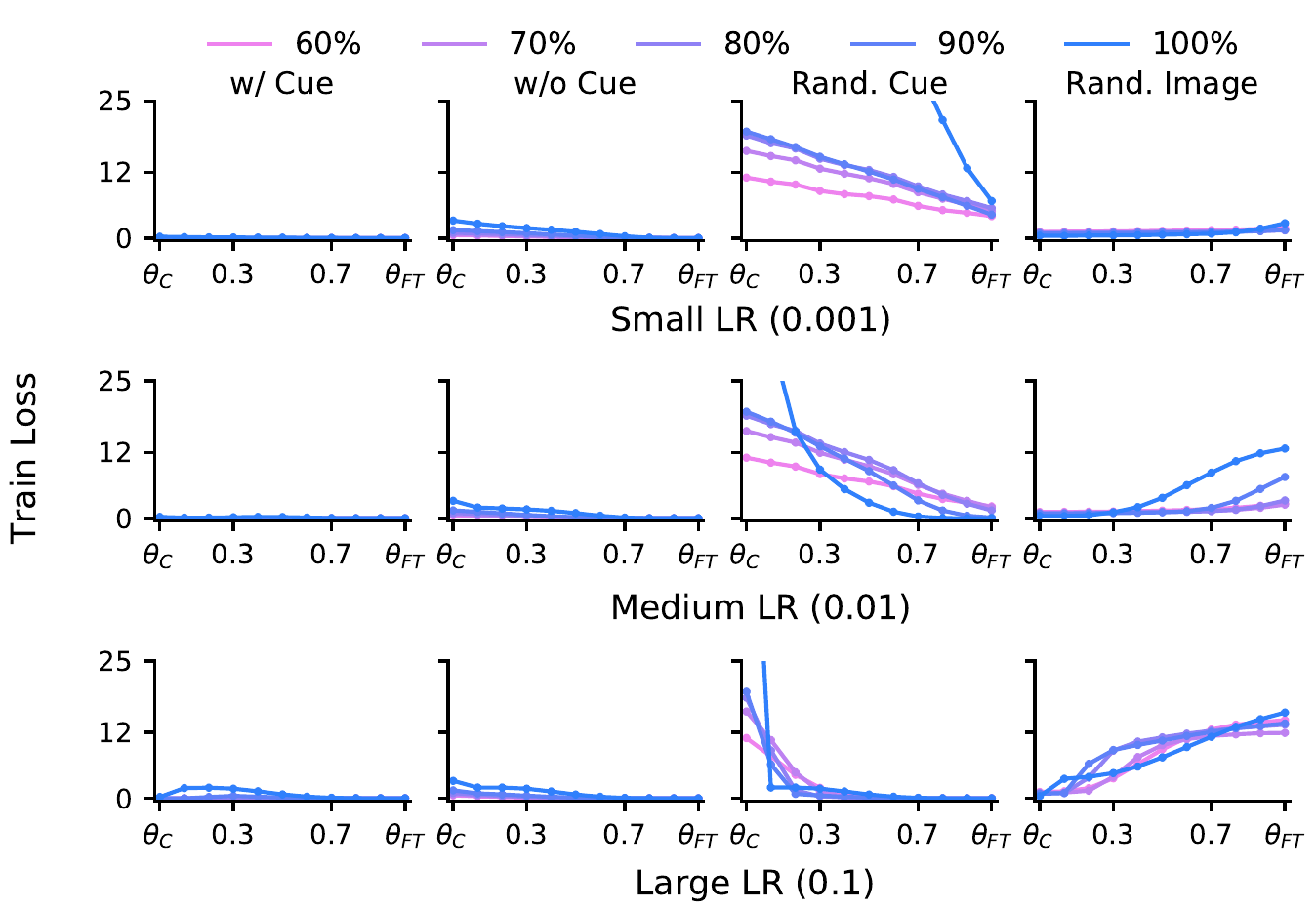}}
  \caption{ResNet18.}
\end{subfigure}
\vspace{-1mm}
\caption{\label{fig:trainloss_ft_c10}\textbf{Fine-tuning of models trained on CIFAR-10 with Box Cue}. We plot train loss curves along the linear path between $\theta_{\text{C}}$ and $\theta_{\text{FT}}$ and see thorough corroboration of our claims in the main text: Linearly connected minimizers exhibit mechanistic similarity, behaving identically on counterfactual datasets, indicating mechanistic connectivity.}
\vspace{5pt}
\end{figure}

\begin{figure}[H]
\centering
\begin{subfigure}[b]{0.49\textwidth}
  \centering
  \centerline{\includegraphics[width=\columnwidth]{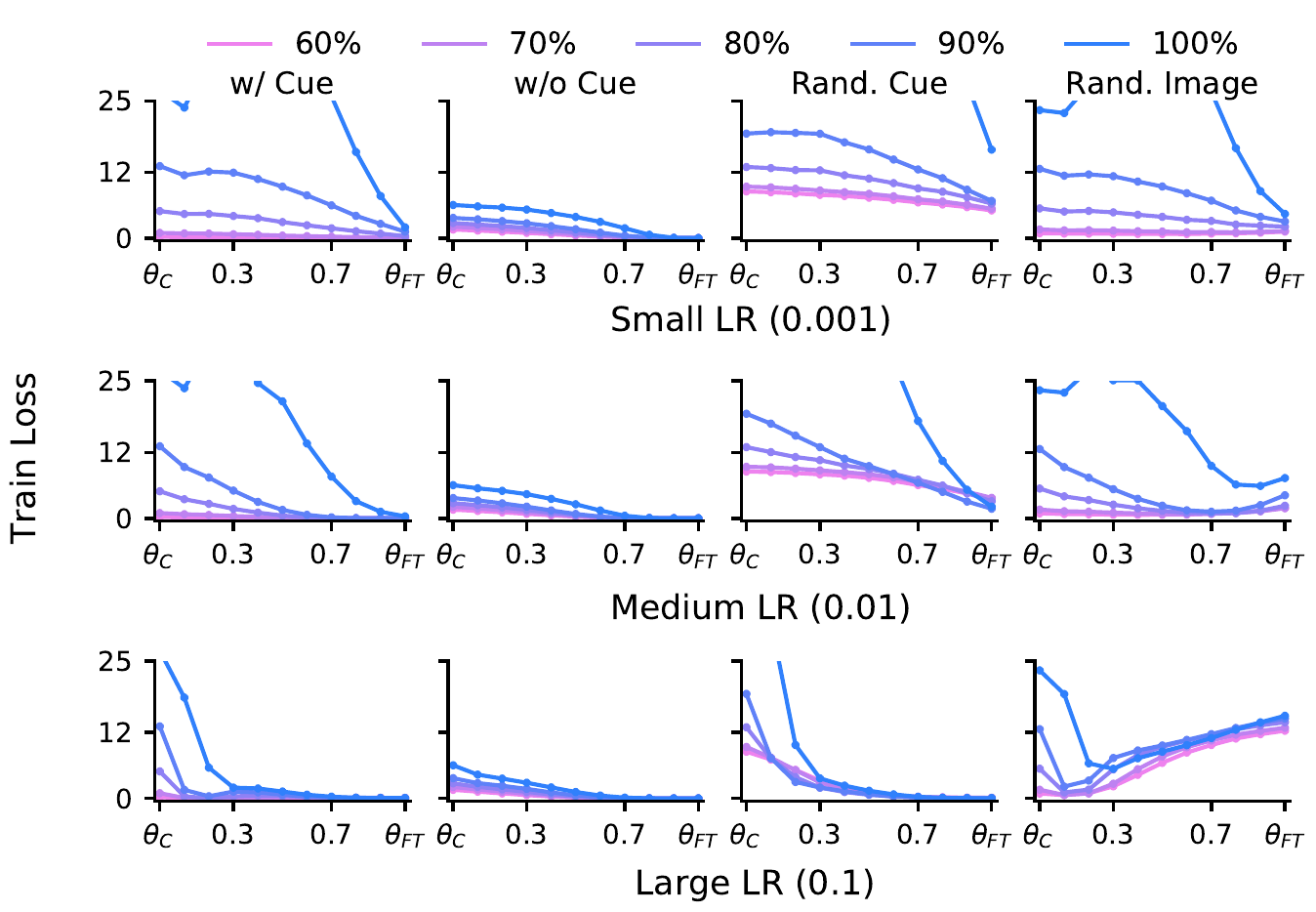}}
  \caption{VGG.}
\end{subfigure}%
\begin{subfigure}[b]{0.49\textwidth}
  \centering
  \centerline{\includegraphics[width=\columnwidth]{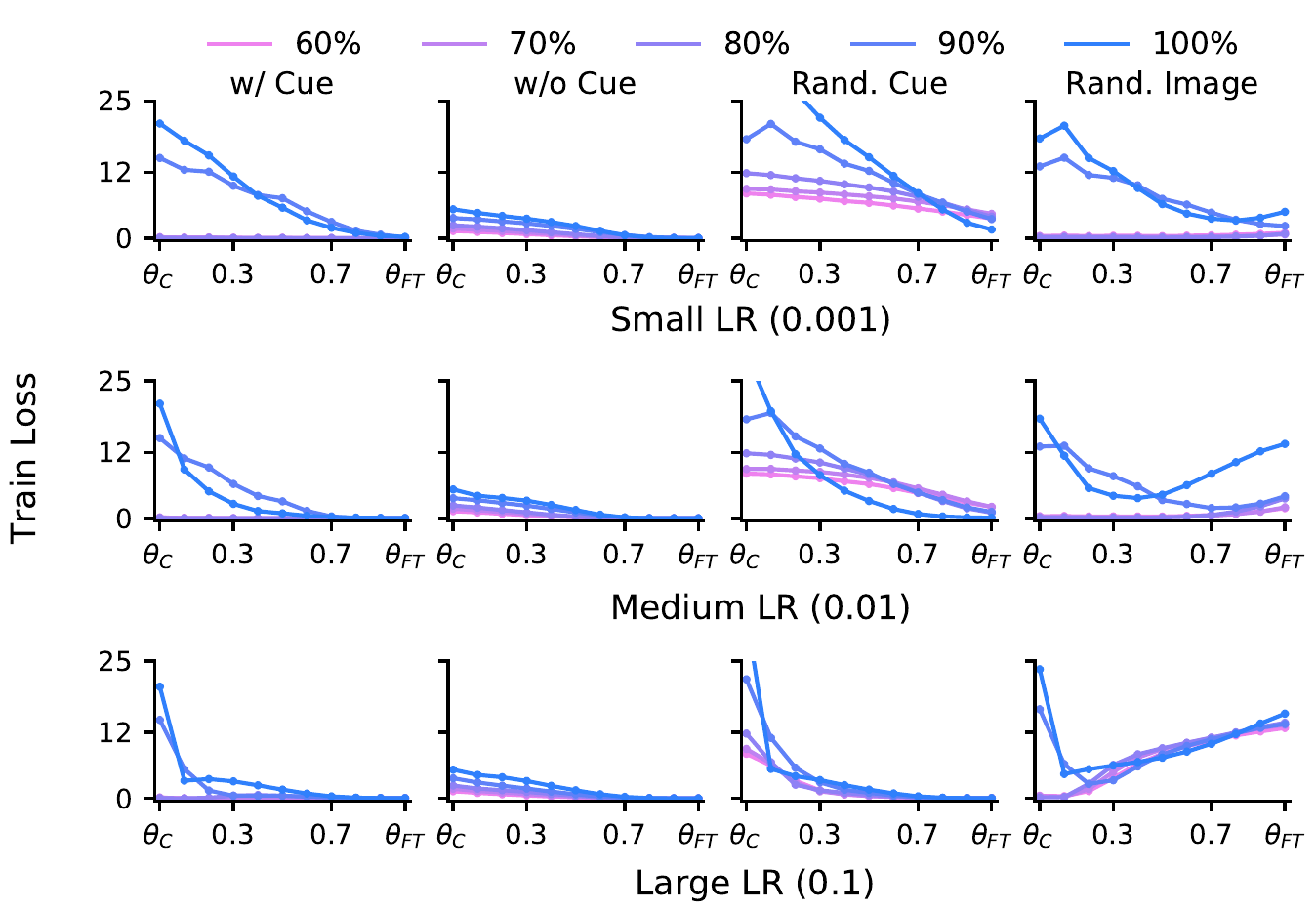}}
  \caption{ResNet18.}
\end{subfigure}
\vspace{-1mm}
\caption{\label{fig:trainloss_ft_c100}\textbf{Fine-tuning of models trained on CIFAR-100 with Box / Color Cue}. We plot train loss curves along the linear path between $\theta_{\text{C}}$ and $\theta_{\text{FT}}$ and see thorough corroboration of our claims in the main text: Linearly connected minimizers exhibit mechanistic similarity, behaving identically on counterfactual datasets, indicating mechanistic connectivity.}
\vspace{5pt}
\end{figure}

\begin{figure}[H]
\centering
\begin{subfigure}[b]{0.49\textwidth}
  \centering
  \centerline{\includegraphics[width=\columnwidth]{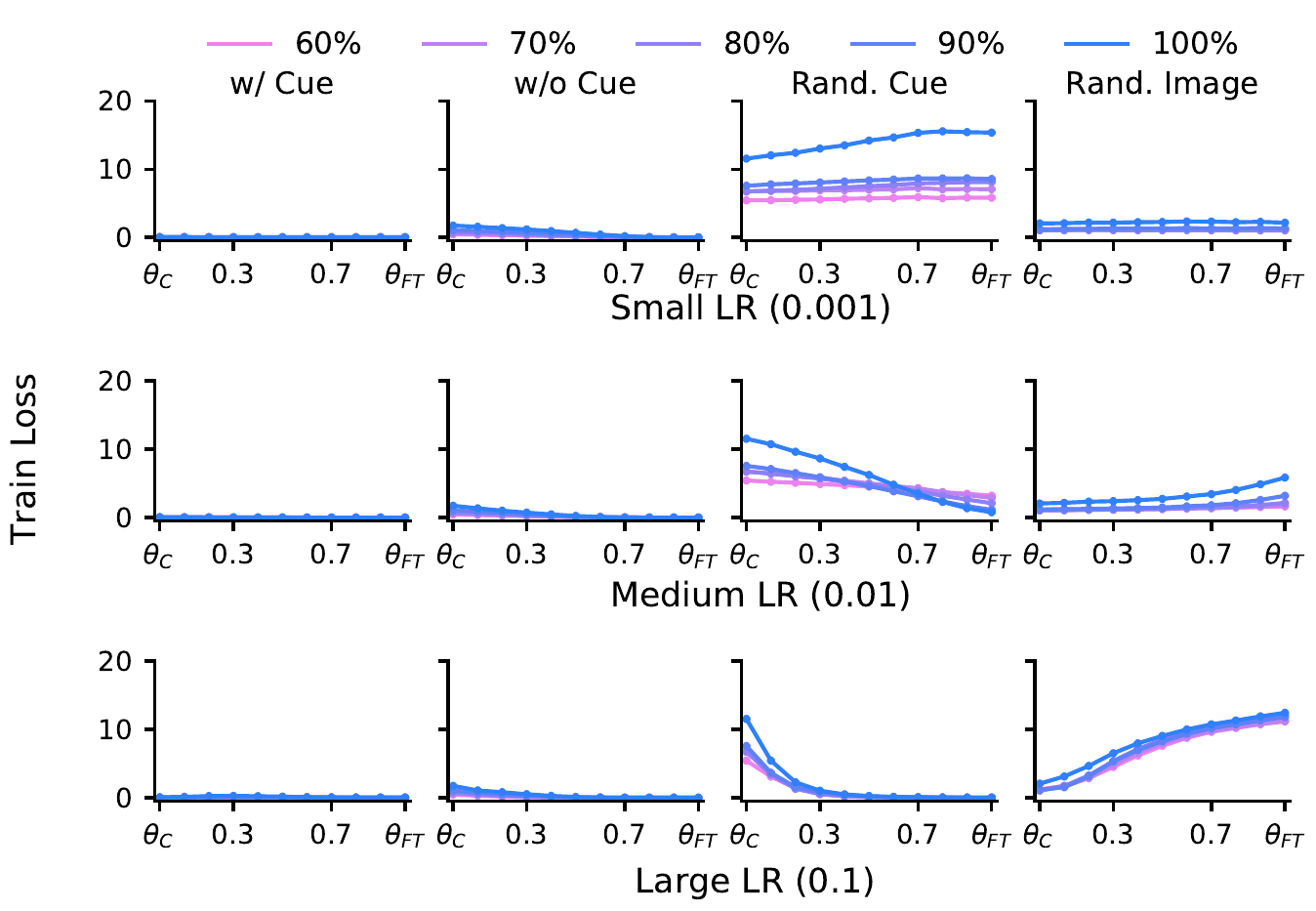}}
  \caption{VGG.}
\end{subfigure}%
\begin{subfigure}[b]{0.49\textwidth}
  \centering
  \centerline{\includegraphics[width=\columnwidth]{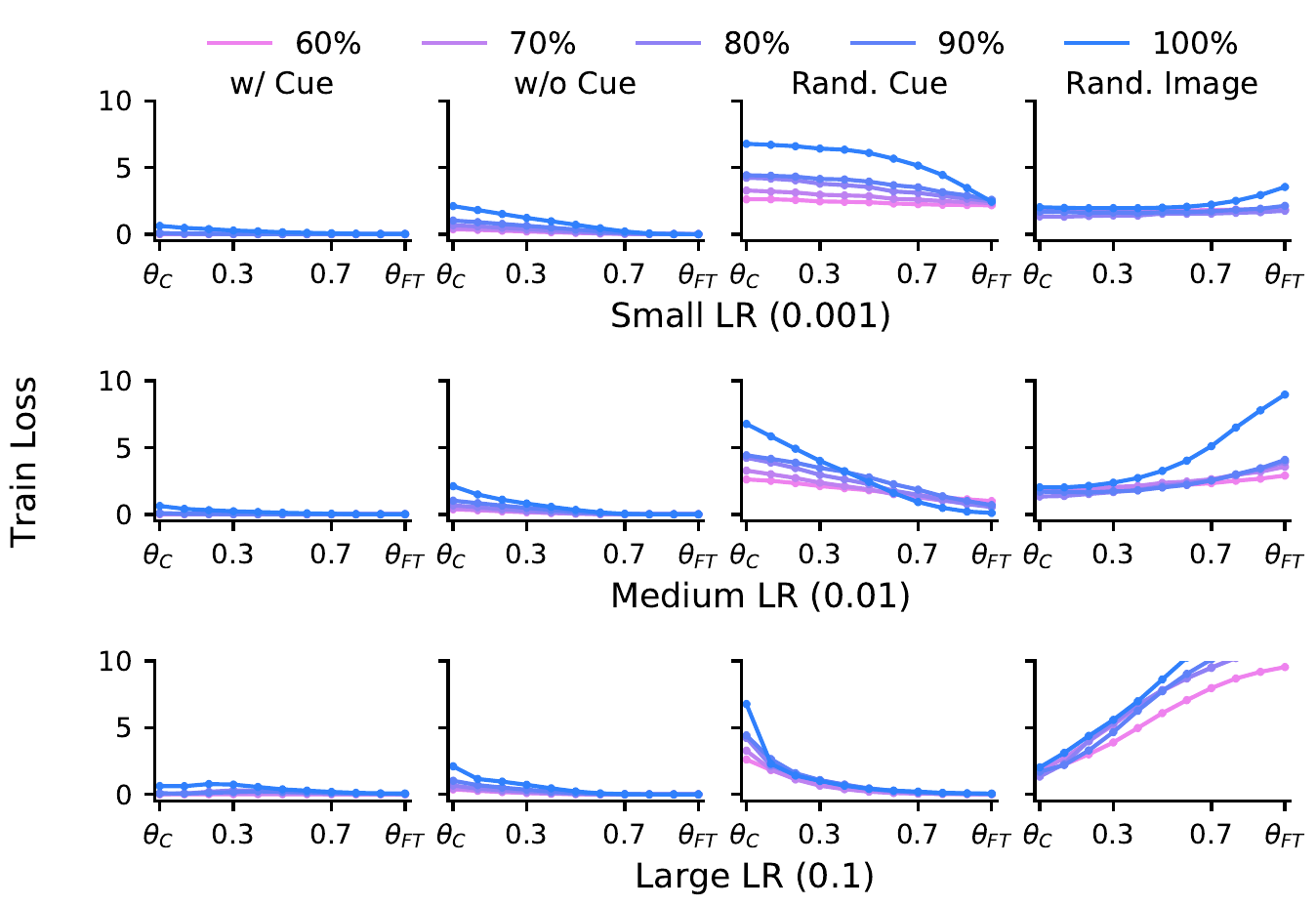}}
  \caption{ResNet18.}
\end{subfigure}
\vspace{-1mm}
\caption{\label{fig:trainloss_ft_dominoes}\textbf{Fine-tuning of models trained on Dominoes}. We plot train loss curves along the linear path between $\theta_{\text{C}}$ and $\theta_{\text{FT}}$ and see thorough corroboration of our claims in the main text: Linearly connected minimizers exhibit mechanistic similarity, behaving identically on counterfactual datasets, indicating mechanistic connectivity.}
\vspace{5pt}
\end{figure}

\end{document}

%% file: tables/table_ft.tex
\setlength{\tabcolsep}{4.5pt}
\begin{table*}
\vspace{-5pt}
\caption{\label{tab:CBFT}\textbf{Evaluating CBFT.} We train ResNet-18 models on our synthetic CIFAR-10, CIFAR-100, and Dominoes dataset with different proportions of samples with cue features and fine-tune them using 2500 ``clean'' samples from a dataset without any cues. Test accuracies (\%) on counterfactual test datasets with No Cue (NC), with Cue (C), Randomized Cue (RC), and Randomized Image (RI) are reported (mean of three seeds). We compare our method, Connectivity-Based Fine-Tuning (CBFT), with several baselines: Fine-tuning with a medium/small learning rate (FT$_{\text{M/S}}$), LLR~\citep{kirichenko2022last}, and LPFT~\citep{kumar2022fine}. $\sim$ denotes invariance is desirable, i.e., accuracy should be similar to that on NC; $\uparrow$/$\downarrow$ indicate higher/lower accuracy is desirable; best results are in bold. 
We generally see that all baselines yield large degradations in the absence of cues, and even achieve very high accuracy when the underlying image is randomized. Meanwhile, CBFT is able to break reliance on cues, inducing representations that are completely invariant to their presence.
}
\centering
\scriptsize
\begin{tabular}{@{}c|cccc|cccc|cccc|cccc@{}}
\toprule
& \multicolumn{4}{c|}{60\% Cue data}         & \multicolumn{4}{c|}{70\% Cue data}         & \multicolumn{4}{c|}{80\% Cue data}         & \multicolumn{4}{c}{90\% Cue data}         \\ \midrule
C-10  & NC$^{\uparrow}$    & C$^{\sim}$   & RC$^{\sim}$   & RI$^{\downarrow}$  & NC$^{\uparrow}$    & C$^{\sim}$   & RC$^{\sim}$   & RI$^{\downarrow}$  & NC$^{\uparrow}$    & C$^{\sim}$   & RC$^{\sim}$   & RI$^{\downarrow}$   & NC$^{\uparrow}$    & C$^{\sim}$   & RC$^{\sim}$   & RI$^{\downarrow}$   \\ \midrule
%
%
%
FT$_{\text{M}}$   & \textbf{75.7} & 98.4 & 23.6    & 83.4      & \textbf{75.8} & 98.6 & 27.7    & 78.6      & \textbf{71.3} & 97.7 & 37.6    & 63.6      & 67.2 & 95.4 & 49.6    & 46.6      \\
FT$_{\text{S}}$ & 75.8 & 98.7 & 17.5    & 90.1      & 74.9 & 98.8 & 16.3    & 91.1      & 69.9 & 98.4 & 15.7    & 90.9      & 64.7 & 97.9 & 15.3    & 90.7      \\
LLR        & 71.6 & 95.1 & 36.3    & 57.1      & 70.9 & 95.8 & 29.9    & 65.8      & 65.1 & 81.8 & 27.0    & 53.2      & 59.3 & 70.7 & 24.6    & 40.7      \\
LPFT    & 70.6 & 88.1 & 21.0    & 70.7      & 69.6 & 87.3 & 18.7    & 72.5      & 64.4 & 63.8 & 18.8    & 48.0      & 59.7 & 56.6 & 19.8    & 37.8       \\
CBFT   & 74.1 & \textbf{71.5} & \textbf{73.4}    & \textbf{8.75}      & 73.2 & \textbf{69.2} & \textbf{72.3}    & \textbf{8.60}      & 70.0 & \textbf{70.0} & \textbf{69.5}    & \textbf{9.68}      & \textbf{67.9} & \textbf{72.5} & \textbf{68.1}    & \textbf{13.1}      \\ \midrule
C-100  & NC$^{\uparrow}$    & C$^{\sim}$   & RC$^{\sim}$   & RI$^{\downarrow}$  & NC$^{\uparrow}$    & C$^{\sim}$   & RC$^{\sim}$   & RI$^{\downarrow}$  & NC$^{\uparrow}$    & C$^{\sim}$   & RC$^{\sim}$   & RI$^{\downarrow}$   & NC$^{\uparrow}$    & C$^{\sim}$   & RC$^{\sim}$   & RI$^{\downarrow}$   \\ \midrule
%
%
%
%
FT$_{\text{m}}$   & \textbf{44.4} & 99.2 & 12.8    & 85.3      & \textbf{40.3} & 99.6 & 12.3    & 89.8      & 33.6 & 99.0 & 11.4    & 90.5      & 25.2 & 79.2 & 9.79    & 57.9      \\
FT$_{\text{s}}$ & 43.1 & 99.6 & 10.3    & 93.6      & 38.2 & 99.7 & 10.5    & 95.7      & 32.5 & 99.6 & 10.4    & 97.0      & 24.5 & 39.4 & 4.87    & 30.9      \\
LLR        & 35.5 & 99.2 & 12.1    & 89.0      & 31.5 & 98.6 & 11.3    & 89.6      & 25.3 & 96.7 & 10.6    & 89.4      & 18.9 & 75.1 & 9.1     & 58.7 \\
LPFT    & 35.1 & 93.2 & 10.3    & 82.3      & 31.1 & 90.2 & 9.89    & 78.5      & 25.6 & 89.6 & 9.70    & 80.8      & 18.7 & \textbf{28.6} & \textbf{4.42}    & \textbf{19.6}      \\
CBFT   & 42.7 & \textbf{65.0} & \textbf{36.4}    & \textbf{14.6}      & 38.5 & \textbf{66.7} & \textbf{34.7}    & \textbf{21.2}      & \textbf{34.6} & \textbf{69.3} & \textbf{23.0}    & \textbf{27.9}      & \textbf{28.5} & 72.9 & 23.2    & 46.0      \\ \midrule
Dom.  & NC$^{\uparrow}$    & C$^{\sim}$   & RC$^{\sim}$   & RI$^{\downarrow}$  & NC$^{\uparrow}$    & C$^{\sim}$   & RC$^{\sim}$   & RI$^{\downarrow}$  & NC$^{\uparrow}$    & C$^{\sim}$   & RC$^{\sim}$   & RI$^{\downarrow}$   & NC$^{\uparrow}$    & C$^{\sim}$   & RC$^{\sim}$   & RI$^{\downarrow}$   \\ \midrule
%
%
%
%
FT$_{\text{m}}$   & \textbf{77.4} & 96.8 & 43.8    & 56.1      & \textbf{76.6} & 96.6 & 42.7    & 58.7      & \textbf{74.1} & 95.7 & 41.7    & 61.3      & \textbf{68.8} & 95.1 & 40.0    & 57.5      \\
FT$_{\text{s}}$ & 76.4 & 96.9 & 37.5    & 62.4      & 76.8 & 96.6 & 32.5    & 66.5      & 73.2 & 96.4 & 30.8    & 67.7      & 67.3 & 95.2 & 31.2    & 65.6      \\
LLR        & 74.6 & 94.4 & 39.8    & 53.0      & 73.9 & 93.2 & 36.3    & 54.7      & 70.8 & 84.8 & 33.1    & 46.6      & 63.3 & 77.0 & 31.2    & 39.0      \\
LPFT    & 73.2 & 92.5 & 38.0    & 51.8      & 72.7 & 88.0 & 34.8    & 50.9      & 69.4 & 34.8 & 33.1    & 39.1      & 61.2 & 60.8 & 31.2    & 26.6      \\
CBFT   & 72.0 & \textbf{64.9} & \textbf{67.5}    & \textbf{9.9}       & 71.5 & \textbf{70.0} & \textbf{59.2}    & \textbf{12.1}      & 70.8 & \textbf{69.7} & \textbf{65.9}    & \textbf{11.9}      & 67.2 & \textbf{68.7} & \textbf{61.5}    & \textbf{14.9}  \\ \bottomrule
\end{tabular}
\vspace{-4pt}
\end{table*}

%% file: tables/table_ablations.tex
\setlength{\tabcolsep}{4.5pt}
\begin{table}
\caption{\label{tab:ablations}\textbf{Ablating CBFT.} We train ResNet-18 models on our synthetic CIFAR-10, CIFAR-100, and Dominoes dataset with different proportions of samples with cue features and fine-tune them using 2500 ``clean'' samples from a dataset without any cues. Test accuracies (\%) on counterfactual test datasets with No Cue (NC), with Cue (C), Randomized Cue (RC), and Randomized Image (RI) are reported. We compare Connectivity-Based Fine-Tuning (CBFT) with two of its ablations (see App.~\ref{app:ablations}): (i) $-\mathcal{L}_{\text{barrier}}$, for which the barrier inducing loss is removed from the training process and (ii) $-\mathcal{L}_{\text{Inv.}}$, for which the invariance loss is removed. $\sim$ denotes invariance is desirable, i.e., accuracy should be similar to that on NC; $\uparrow$/$\downarrow$ indicate higher/lower accuracy is desirable; best results are in bold. For discussion, please see App.~\ref{app:ablations}.
\vspace{-8pt}
}
\centering
\scriptsize
\begin{tabular}{@{}c|cccc|cccc|cccc|cccc@{}}
\toprule
& \multicolumn{4}{c|}{60\% Cue data}         & \multicolumn{4}{c|}{70\% Cue data}         & \multicolumn{4}{c|}{80\% Cue data}         & \multicolumn{4}{c}{90\% Cue data}         \\ \midrule
C-10  & NC$^{\uparrow}$    & C$^{\sim}$   & RC$^{\sim}$   & RI$^{\downarrow}$  & NC$^{\uparrow}$    & C$^{\sim}$   & RC$^{\sim}$   & RI$^{\downarrow}$  & NC$^{\uparrow}$    & C$^{\sim}$   & RC$^{\sim}$   & RI$^{\downarrow}$   & NC$^{\uparrow}$    & C$^{\sim}$   & RC$^{\sim}$   & RI$^{\downarrow}$   \\ \midrule
CBFT   & 74.1 & \textbf{71.5} & \textbf{73.4}    & \textbf{8.75}      & 73.2 & \textbf{69.2} & \textbf{72.3}    & 8.60      & 70.0 & \textbf{70.0} & \textbf{69.5}    & \textbf{9.68}      & \textbf{67.9} & 72.5 & 68.1    & 13.1      \\
$-\mathcal{L}_{\text{barrier}}$  & \textbf{75.8} & 93   & 69.3 & 24.4 & \textbf{75.9} & 90   & 72.1 & 18.6 & \textbf{71.6} & 89.9 & 66.3 & 23.5 & 67.8 & 89.6 & 65.1 & 20.5 \\
$-\mathcal{L}_{\text{Inv.}}$  & 73.4 & 69.4 & 68.8 & 14.2 & 72.9 & 65.2 & 71.3 & {8.26} & 69.3 & 64.8 & 68.1 & 9.72 & 65.8 & \textbf{64.8} & \textbf{65}   & \textbf{10.3} \\ \midrule
C-100  & NC$^{\uparrow}$    & C$^{\sim}$   & RC$^{\sim}$   & RI$^{\downarrow}$  & NC$^{\uparrow}$    & C$^{\sim}$   & RC$^{\sim}$   & RI$^{\downarrow}$  & NC$^{\uparrow}$    & C$^{\sim}$   & RC$^{\sim}$   & RI$^{\downarrow}$   & NC$^{\uparrow}$    & C$^{\sim}$   & RC$^{\sim}$   & RI$^{\downarrow}$   \\ \midrule
CBFT   & 42.7 & 65.0 & \textbf{36.4}    & 14.6      & 38.5 & \textbf{66.7} & \textbf{34.7}    & \textbf{21.2}      & \textbf{34.6} & \textbf{69.3} & {23.0}    & \textbf{27.9}      & \textbf{28.5} & \textbf{72.9} & \textbf{23.2}    & 46.0      \\
$-\mathcal{L}_{\text{barrier}}$  & \textbf{44.7} & 99.8 & 17.5 & 81.6 & \textbf{40.2} & 99.9 & 13.7 & 88.9 & 34.6 & 99.9 & 11.3 & 95.1 & 26.5 & 99.1 & 13.5 & 82.2 \\
$-\mathcal{L}_{\text{Inv.}}$ & 43.2 & \textbf{59.4} & 36.5 & \textbf{12.5} & 35.7 & 64.2 & 26   & 25.5 & 34.1 & 70.2 & \textbf{23.5} & 36.7 & 24.7 & 69.2 & 15.9 & \textbf{45.6} \\ \midrule
Dom.  & NC$^{\uparrow}$    & C$^{\sim}$   & RC$^{\sim}$   & RI$^{\downarrow}$  & NC$^{\uparrow}$    & C$^{\sim}$   & RC$^{\sim}$   & RI$^{\downarrow}$  & NC$^{\uparrow}$    & C$^{\sim}$   & RC$^{\sim}$   & RI$^{\downarrow}$   & NC$^{\uparrow}$    & C$^{\sim}$   & RC$^{\sim}$   & RI$^{\downarrow}$   \\ \midrule
CBFT   & 72.0 & \textbf{64.9} & \textbf{67.5}    & {9.9}       & 71.5 & \textbf{70.0} & \textbf{59.2}    & {12.1}      & 70.8 & \textbf{69.7} & \textbf{65.9}    & {11.9}      & \textbf{67.2} & \textbf{68.7} & \textbf{61.5}    & {14.9}  \\ %
$-\mathcal{L}_{\text{barrier}}$  & \textbf{77.1} & 94.9 & 63.2 & 32.7 & \textbf{77.4} & 94.2 & 65.8 & 29.2 & \textbf{74.5} & 93.3 & 63.5 & 30.1 & 67.1 & 91.9 & 55.5 & 32.9 \\
$-\mathcal{L}_{\text{Inv.}}$  & 74.2 & 40.4 & 41.8 & \textbf{6.93} & 74.6 & 28.2 & 24.9 & \textbf{10.6} & 71.3 & 20.1 & 22.2 & \textbf{6.92} & 66   & 21.2 & 20.9 & \textbf{6.26} \\ 
\bottomrule
\end{tabular}
\vspace{-12pt}
\end{table}

%% file: tables/tab_scratch.tex
\begin{table}
\caption{\label{tab:scratch}\textbf{Training from scratch on minimal clean data.} We train ResNet-18 models on the 2500 ``clean'' samples used in Tab.~\ref{tab:CBFT} from the original CIFAR-10, CIFAR-100, and Dominoes datasets. Test accuracies (\%) on counterfactual test datasets with No Cue (NC), with Cue (C), Randomized Cue (RC), and Randomized Image (RI) are reported. $\sim$ denotes invariance is desirable, i.e., accuracy should be similar to that on NC; $\uparrow$/$\downarrow$ indicate higher/lower accuracy is desirable. 
}
\centering
\footnotesize
\begin{tabular}{@{}c|cccc@{}}
\toprule
      & NC$^{\uparrow}$ & C$^{\sim}$ & RC$^{\sim}$ & RI$^{\downarrow}$ \\ \midrule
C-10  & 47.5            & 47.4       & 47.5        & 9.69              \\
C-100 & 16.5            & 16.4       & 16.4        & 1.19              \\
Dom.  & 48.5            & 31         & 31          & 10.8              \\ \bottomrule
\end{tabular}
\vspace{-10pt}
\end{table}